\documentclass[twoside,11pt]{article}

\usepackage[preprint]{jmlr2e}

\usepackage[utf8]{inputenc} 
\usepackage[T1]{fontenc}    
\usepackage{hyperref}       
\usepackage{url}            
\usepackage{booktabs}       
\usepackage{amsfonts}       
\usepackage{nicefrac}       
\usepackage{microtype}      
\usepackage{xcolor}         

\usepackage{amsmath}
\usepackage{algorithmic}
\usepackage{algorithm}
\usepackage[caption=false,font=normalsize,labelfont=sf,textfont=sf]{subfig}
\usepackage{graphicx}
\usepackage{comment}

\usepackage{xr}
\makeatletter
\newcommand*{\addFileDependency}[1]{
\typeout{(#1)}
\@addtofilelist{#1}
\IfFileExists{#1}{}{\typeout{No file #1.}}
}\makeatother

\newcommand{\dataset}{{\cal D}}
\newcommand{\fracpartial}[2]{\frac{\partial #1}{\partial  #2}}

\usepackage{lastpage}
\jmlrheading{26}{2025}{1-\pageref{LastPage}}{5/25}{}{}{Zhijie Xie and Shenghui Song}

\ShortHeadings{Client Selection for Federated Policy Optimization}{Xie and Song}
\firstpageno{1}

\begin{document}

\title{Client Selection for Federated Policy Optimization with Environment Heterogeneity}

\author{\name Zhijie Xie \email zhijie.xie@connect.ust.hk \\
        \name Shenghui Song \email eeshsong@ust.hk \\
        \addr Department of Electronic and Computer Engineering\\
  The Hong Kong University of Science and Technology\\
  Clear Water Bay, Kowloon, Hong Kong}

\editor{Waiting for assignment}

\maketitle

\begin{abstract}
The development of Policy Iteration (PI) has inspired many recent algorithms for Reinforcement Learning (RL), including several policy gradient methods that gained both theoretical soundness and empirical success on a variety of tasks. The theory of PI is rich in the context of centralized learning, but its study under the federated setting is still in the infant stage. This paper investigates the federated version of Approximate PI (API) and derives its error bound, taking into account the approximation error introduced by environment heterogeneity. We theoretically prove that a proper client selection scheme can reduce this error bound. Based on the theoretical result, we propose a client selection algorithm to alleviate the additional approximation error caused by environment heterogeneity. Experiment results show that the proposed algorithm outperforms other biased and unbiased client selection methods on the federated mountain car problem, the Mujoco Hopper problem, and the SUMO-based autonomous vehicle training problem by effectively selecting clients with a lower level of heterogeneity from the population distribution.
\end{abstract}


\section{Introduction\label{sec:Introduction}}

Reinforcement Learning (RL) has been applied to many real-world applications ranging from gaming and robotics to recommender systems \citep{Silver2016,10.1145/3289600.3290999}. However, single-agent RL often suffers from poor sample efficiency, resulting in slow convergence and a high cost of sample collection \citep{JMLR:v21:18-012,fan2021fault,pmlr-v80-papini18a}. Therefore, it is desirable to deploy RL algorithms to large-scale and distributed systems where multiple agents can contribute to the learning collaboratively. However, Multi-Agent RL (MARL) \citep{DBLP:journals/corr/abs-1911-10635} and parallel RL \citep{DBLP:journals/corr/NairSBAFMPSBPLM15,pmlr-v48-mniha16} require intensive communication among agents or data sharing, which may not be practical due to both the communication bottleneck and privacy concerns of many real-world applications. For example, privacy is a major concern in autonomous driving \citep{DBLP:journals/corr/abs-1910-06001,9457207}, and sharing data among vehicles is not allowed. To this end, Federated Learning (FL) \citep{XIANJIA2021135,10025836,LimHK2020}, which enables multiple clients to jointly train a global model without violating user privacy, is an appealing solution for addressing the sample inefficiency and privacy issue of RL in innovative applications such as autonomous driving, IoT network, and healthcare \citep{https://doi.org/10.48550/arxiv.2206.05581}. As a result, Federated Reinforcement Learning (FRL) has attracted much research attention \citep{DBLP:journals/corr/abs-2108-11887}.

Despite the significant progress of empirical works on FRL \citep{DBLP:journals/corr/abs-2108-11887}, the community's understanding of FRL is still in its infancy, especially from the theoretical perspective. For example, the sample efficiency of Policy Gradient (PG) methods is typically low due to the large variance in gradient estimation. This issue could be exacerbated in the context of FL, where clients with heterogeneous environments can generate a diverse range of trajectories. To address this problem, a variance-reduced policy gradient method, namely Federated Policy Gradient with Byzantine Resilience (FedPG-BR), was proposed together with an analysis of the sample efficiency and convergence guarantee \citep{fan2021fault}. While clients are assumed to be homogeneous in FedPG-BR, another line of work, termed FedKL \citep{10038492}, noticed that the environment heterogeneity imposes an extra layer of difficulty in learning and proved that a Kullback-Leibler (KL) penalized local objective can generate a monotonically improving sequence of policies to accelerate convergence. The authors of QAvg \& PAvg \citep{pmlr-v151-jin22a} provided a convergence proof for the federated Q-Learning and federated PG. QAvg offered important insights regarding how the Bellman operators can be generalized to the federated setting and proposed a useful tool, i.e., the imaginary environment (the average of all clients' environments), for analyzing FRL. More recently, FedSARSA \citep{zhang2024finitetime} studied the integration of FRL and SARSA, where SARSA is an on-policy Temporal Difference (TD) algorithm. However, there has not been any convergence analysis regarding Policy Iteration (PI) in FRL in the literature. Given PI's application and theoretical importance, it is desirable to fill this knowledge gap and derive efficient FRL algorithms accordingly.

Among existing RL methods, PI is one of the most popular ones and serves as the foundation of many policy optimization methods, e.g., Safe Policy Iteration (SPI) \citep{pmlr-v28-pirotta13}, Trust Region Policy Optimisation (TRPO) \citep{DBLP:journals/corr/SchulmanLMJA15}, and Deep Conservative Policy Iteration (DCPI) \citep{Vieillard_Pietquin_Geist_2020}. With exact PI, convergence to the optimal policy is guaranteed under mild conditions. However, exact policy evaluation and policy improvement are normally impractical. With Approximate Policy Iteration (API) \citep{bertsekas2022abstract,bertsekas1996neuro}, it is assumed that the approximation error is inevitable, and only estimates of the value function and improved policy with bounded errors are available. In the presence of these approximation errors, convergence is not ensured, but the difference in value functions between the generated policy and the optimal policy is bounded \citep{bertsekas2022abstract}. In some cases, the algorithm ends up generating a cycle of policies, which is called the policy oscillation/chattering phenomenon \citep{Bertsekas2011,NIPS2011_d9731321}. Unfortunately, FRL with heterogeneous environments will introduce extra approximation errors into the policy iteration process, making the associated analysis more challenging. As will be shown in the following sections, this error is proportional to the level of heterogeneity of the system, and client selection is an effective way to alleviate this problem.

There exist various client selection schemes for Federated Supervised Learning (FSL) and most of them can be classified into two categories: (1) unbiased client selection; and (2) biased client selection. Convergence guarantee for both schemes has been studied and generalized to tackle the heterogeneity issue of FSL \citep{li2020on,DBLP:conf/iclr/LiSBS20,pmlr-v151-jee-cho22a}. However, to the best of the authors' knowledge, there is no known client selection scheme specifically designed to tackle the heterogeneity issue of FRL.

\textbf{Contributions.} In this paper, we derive the error bound of Federated Approximate Policy Iteration (FAPI) under heterogeneous environments, which is not yet available in the literature. The derived error bound takes the heterogeneity level of clients into consideration and explicitly reveals its impact. Based on the error bound, we propose a client selection algorithm to improve the convergence speed of federated policy optimization. The efficacy of the proposed algorithm is validated on the federated mountain car problem, the Mujoco Hopper problems, and the SUMO-based autonomous vehicle training problem.

\section{Background\label{sec:Background}}

In Section \ref{subsec:FRL}, we introduce the optimization problem of FRL. In Section \ref{subsec:API}, we review some known results on API. An imaginary environment is introduced in Section \ref{subsec:IE} to assist the following analysis.

\subsection{Federated Reinforcement Learning\label{subsec:FRL}}

The system setup of FRL in this paper is similar to that of FL \citep{pmlr-v54-mcmahan17a}, i.e., a federated system consisting of one central server
and $N$ distributed clients. In the $t$-th training round, the central server broadcasts the current global policy $\pi^{t}$
to $K$ selected clients which will perform $I$ iterations of local
training. In each iteration, the $n$-th client interacts with its environment
to collect $E$ trajectories and utilize them to update its local policy to $\pi^{t+1}_{n}$. At the end of each round, the training results will be uploaded to the central server for aggregation to obtain the new global policy $\pi^{t+1}$.

We model the local learning problem of each client as a finite-state infinite-horizon discounted Markov Decision Process (MDP). Accordingly, the FRL system consists of $N$ finite MDPs $\{ (\mathcal{S},\mu,\mathcal{A},P_{n},\mathcal{R},\gamma) \vert n \in\{1,...,N\}\} $,
where $\mathcal{S}$ denotes a finite set of states, $\mu$ represents the initial state distribution, $\mathcal{A}$ is a finite set of actions, and $\gamma\in(0,1)$ is the discount factor. The transition function $P_{n}(s^{\prime}\vert s,a):\mathcal{S}\times\mathcal{S}\times\mathcal{A}\to[0, 1]$
represents the probability that the n-th MDP transits from state $s$ to $s^{\prime}$
after taking action $a$ \citep{Sutton1998}. The reward function $\mathcal{R}(s,a):\mathcal{S}\times\mathcal{A}\to[0,R_{\max}]$ gives the expected reward for taking action $a$ in state $s$, and we assume rewards are bounded and non-negative. As a result, the $n$-th MDP $\mathcal{M}_n$ can be represented by a 6-tuple $(\mathcal{S},\mu,\mathcal{A},P_{n},\mathcal{R},\gamma)$
sharing the same state space, action space, initial state distribution, and reward function with other clients, but with possibly different transition probabilities. A client's behavior is controlled by a stochastic policy $\pi_{n}:\mathcal{S}\times\mathcal{A}\rightarrow[0,1]$
which outputs the probability of taking an action $a$ in a given state
$s$. Throughout this work, we consider parameterized policies $\pi^{\theta}$ where $\pi^{\theta}(a \vert s)$ is a differentiable function of the parameter vector $\theta$. For parameterized policies with t-indexed notation, we omit the parameter vector for notation simplicity and write $\pi^{t}$ and $\pi^{t}_{n}$ for the global policy $\pi^{\theta^{t}}$ and the $n$-th local policy $\pi^{\theta^{t}_{n}}$, respectively. Furthermore, we define the state-value function
\begin{alignat}{1}
V^{\pi}_{n}(s) & =\mathbb{E}_{\pi,P_{n}}\left[\sum_{l=0}^{\infty} \gamma^{l}\mathcal{R}(s_{t+l},a_{t+l}) \vert s_{t}=s\right], \nonumber
\end{alignat}
where the expectation is performed over actions sampled from policy $\pi$ and states sampled from the transition probability $P_{n}$. It gives the expected return when the client starts from state $s$ and follows policy $\pi$ thereafter in the $n$-th MDP. For parameterized value functions with t-indexed notation, we omit the parameter vector $w$ for notation simplicity and write $V^{t}$ and $V^{t}_{n}$ for $V^{w^{t}}$ and $V^{w^{t}_{n}}$, respectively. In each round, every client aims to train a local policy to maximize its expected discounted reward
\begin{alignat}{1}
\eta_{n}(\pi) & = \mathbb{E}_{s_{0} \thicksim \mu,a_{t} \thicksim \pi,s_{t+1} \thicksim P_{n}} \left[\sum_{t=0}^{\infty}\gamma^{t}\mathcal{R}(s_{t},a_{t})\right], \label{eq:1}
\end{alignat}
or equivalently, $\eta_{n}(\pi) = \mathbb{E}_{s_{0} \thicksim \mu} \left[ V^{\pi}_{n}(s_{0}) \right]$. The notation $\mathbb{E}_{s_{0} \thicksim \mu_{n},a_{t} \thicksim \pi,s_{t+1} \thicksim P_{n}}$ indicates that the reward is averaged over all states and actions according to the initial state distribution, the transition probability, and the policy. Accordingly, the optimization problem for FRL can be formulated as 
\begin{alignat}{1}
\max_{\pi}\eta(\pi) \quad \text{where} \quad \eta(\pi) &= \sum_{n=1}^{N}q_{n}\eta_{n}(\pi), \label{eq:3}
\end{alignat}
where $q_{n}$ is the weight of the $n$-th client. Denote the average value function of policy $\pi$ as
\begin{alignat}{1}
\bar{V}^\pi(s) = \sum_{n=1}^{N} q_{n} V_{n}^{\pi}(s), \; \forall s \in \mathcal{S}, \nonumber
\end{alignat}
then we can rewrite (\ref{eq:3}) as
\begin{alignat}{1}
\max_{\pi}\eta(\pi) \quad \text{where} \quad \eta(\pi) &= \mathbb{E}_{s_{0} \thicksim \mu} \left[ \bar{V}^{\pi}(s_{0}) \right]. \label{eq:4}
\end{alignat}

The above formulation covers both heterogeneous and homogeneous cases. In particular, the different MDPs, i.e., different transition probabilities, represent the heterogeneous environments experienced by clients. All MDPs will be identical for the homogeneous case \citep{fan2021fault}. It is worth noting that the optimization problem in (\ref{eq:4}) is often referred to as the Weighted Value Problem (WVP) in the latent MDP literature. Finding the optimal solution of WVP is NP-hard \citep{steimle2021multi}. In contrast, an error bound showing the distance between the obtained policy and the optimal policy is feasible as demonstrated in Section \ref{sec:ebFAPI}.

\subsection{Approximate Policy Iteration\label{subsec:API}}

Given any MDP $\mathcal{M}_n$ defined in Section \ref{subsec:FRL}, it is well known \citep{Sutton1998} that the value function $V^{\pi}_{n}$ is the unique fixed point of the Bellman operator $T^{\pi}_{n}: \mathbb{R}^{\left\vert \mathcal{S} \right\vert} \to \mathbb{R}^{\left\vert \mathcal{S} \right\vert}$, i.e., $\forall s \in \mathcal{S}, V \in \mathbb{R}^{\left\vert \mathcal{S} \right\vert}$
\begin{alignat}{1}
T^{\pi}_{n} V(s) & = \sum_{a} \pi(a \vert s) \left( \mathcal{R}(s,a) + \gamma \sum_{s^\prime} P_{n}(s^{\prime} \vert s, a) V(s^\prime) \right), \quad V^{\pi}_{n}(s) = T^{\pi}_{n} V^{\pi}_{n}(s), \nonumber
\end{alignat}
where $\left\vert \mathcal{S} \right\vert$ denotes the cardinality of $\mathcal{S}$. Similarly, the optimal value function $V^{\ast}_{n}$ is the unique fixed point of the Bellman operator $T_{n}: \mathbb{R}^{\left\vert \mathcal{S} \right\vert} \to \mathbb{R}^{\left\vert \mathcal{S} \right\vert}$, i.e., $\forall s \in \mathcal{S}, V \in \mathbb{R}^{\left\vert \mathcal{S} \right\vert}$
\begin{alignat}{1}
T_{n} V(s) & = \arg\max_{a} \left( \mathcal{R}(s,a) + \gamma \sum_{s^\prime} P_{n}(s^{\prime} \vert s, a) V(s^\prime) \right), \quad V^{\ast}_{n}(s) = T_{n} V^{\ast}_{n}(s). \nonumber
\end{alignat}
Note that the subscript $n$ of the Bellman operators denotes the index of transition probability with which the operator is applied. Both operators are monotonic and sup-norm contractive \citep{bertsekas1996neuro}.

Now we describe the classic API, which is an iterative algorithm that generates a sequence of policies and the associated value functions. Let $\left\Vert \cdot \right\Vert$ denote the sup-norm, i.e. $\left\Vert V \right\Vert = \sup_{s \in \mathcal{S}} \left\vert V(s) \right\vert,\forall V \in \mathbb{R}^{\left\vert \mathcal{S} \right\vert}$, $\left\Vert \cdot \right\Vert_{2}$ denote the l2-norm, i.e. $\left\Vert V \right\Vert_{2} = \sqrt{\sum_{s \in \mathcal{S}} V(s)^{2}},\forall V \in \mathbb{R}^{\left\vert \mathcal{S} \right\vert}$, and $V^{\ast}$ denote the value function of the optimal policy $\pi^{\ast}$. Given an initial policy $\pi^{t}$, each iteration consists of two phases, where $\delta$ and $\epsilon$ are some scalars:\\
\textbf{Policy Evaluation.} The value function $V^{\pi^{t}}$ of the current policy is approximated by $V^{t}$ satisfying
\begin{alignat}{1}
\left\Vert V^{t} - V^{\pi^{t}} \right\Vert \le \delta, \; t=0,1,\cdots. \label{eq:apipe}
\end{alignat}
\textbf{Policy Improvement.} A greedy improvement is made to the policy with an approximation error
\begin{alignat}{1}
\left\Vert T^{\pi^{t+1}} V^{t} - T V^{t} \right\Vert \le \epsilon, \; t=0,1,\cdots. \label{eq:apipi}
\end{alignat}
The following proposition gives the error bound of API.
\begin{proposition}
\label{proposition:BertsekasProp2.4.3}
The sequence $\left(\pi^{t}\right)_{t=0}^{\infty}$ generated by the API algorithm described by (\ref{eq:apipe}), (\ref{eq:apipi}) satisfies
\begin{alignat}{1}
\limsup_{t\to\infty} \left\Vert V^{\pi^{t}} - V^{\ast} \right\Vert & \le \frac{\epsilon + 2 \gamma \delta}{(1 - \gamma)^{2}}, \label{eq:proposition:BertsekasProp2.4.3}
\end{alignat}
\end{proposition}
The detailed proof of Proposition \ref{proposition:BertsekasProp2.4.3} can be found in Proposition 2.4.3 of \cite{bertsekas2022abstract}.

\subsection{Imaginary MDP\label{subsec:IE}}

We define the imaginary MDP as in QAvg \citep{pmlr-v151-jin22a}. Specifically, it is a MDP represented by the 6-tuple $(\mathcal{S},\mu,\mathcal{A},\bar{P},\mathcal{R},\gamma)$ where
\begin{alignat}{1}
\bar{P}(s^\prime \vert s,a) = \sum_{n=1}^{N} q_{n} P_{n} (s^\prime \vert s,a), \forall s^\prime,s \in \mathcal{S}, a \in \mathcal{A}, \nonumber
\end{alignat}
denotes the average transition probability. Accordingly, we denote the Bellman operators in the imaginary MDP as $T^{\pi}_{I}$ and $T_{I}$, where the subscript $I$ indicates that the transition probability is $\bar{P}(s^{\prime} \vert s, a), \forall s,a,s^\prime$.
Moreover, denoting $\pi_{I}^{\ast}$ the optimal policy in the imaginary MDP $\mathcal{M}_{I}$, we have the value function of $\pi$ and the optimal value function in the imaginary MDP
\begin{alignat}{1}
V^{\pi}_{I}(s) & = T^{\pi}_{I} V^{\pi}_{I}(s), \quad V^{\pi^{\ast}_{I}}_{I}(s) = V^{\ast}_{I}(s) = T_{I} V^{\ast}_{I}(s), \nonumber
\end{alignat}
respectively. The imaginary MDP is a handy tool to analyze the behavior of FAPI since it provides a unified view of all clients in the context of MDP where the theory of API is richly supplied with operator theory and the fixed point theorem.

\section{Error Bound of FAPI\label{sec:ebFAPI}}

In this section, we establish the error bound of FAPI under the framework of API and the imaginary MDP. The analysis shows that the FAPI process can be considered as learning in the imaginary MDP which eases the analysis. We focus our discussion on policy with function approximation where the error term introduced by the aggregation step is nontrivial.

To establish FAPI's error bound, we first examine the approximation error of policy evaluation and improvement within the FRL framework. Then, we can derive the distance between the optimal value function and the value function produced by FAPI.

Throughout this work, we assume the policy iteration performed by each client is approximate. In particular, let $V^{t}_{n}$ denote the evaluated value function of the $n$-th client in round $t$, we have
\begin{alignat}{1}
\left\Vert V^{t}_{n} - V^{\pi^{t}}_{n} \right\Vert \le \delta_{n}, \label{eq:deltan} \quad \max_{V\in\mathbb{R}^{\left\vert \mathcal{S} \right\vert}} \left\Vert T^{\pi^{t+1}_{n}}_{n} V - T_{n} V \right\Vert \le \epsilon_{n}, \quad t=0,1,\cdots. 
\end{alignat}

In the following, we define a metric to quantify the level of heterogeneity of the FRL system. Specifically, we provide two metrics, which will be used for bounding the error of policy evaluation and policy improvement, respectively.
\begin{definition}
\label{definition:impacthete}
(Level of heterogeneity). We define two parameters to measure the level of
heterogeneity as
\begin{alignat}{1}
\kappa_{1} = \sum_{n=1}^{N} q_{n} \kappa_{n,I}, \quad \kappa_{2} = \sum_{i,j} q_{i} q_{j} \kappa_{i,j}, \nonumber
\end{alignat}
where $P^{\pi}_{n}(s^\prime \vert s) = \sum_{a} \pi(a \vert s) P_{n}(s^\prime \vert s,a),\forall s,s^{\prime} \in \mathcal{S}$, $\kappa_{i,j} = \max_{\pi,s} \sum_{s^{\prime}} \left\vert P_{i}^{\pi}(s^{\prime} | s) - P_{j}^{\pi}(s^{\prime} | s) \right\vert$ and $\kappa_{n,I} = \max_{\pi,s} \sum_{s^{\prime}} \left\vert P_{n}^{\pi}(s^{\prime} | s) - \sum_{j=1}^{N} q_{j} P_{j}^{\pi}(s^{\prime} | s) \right\vert$.
\end{definition}
Here, $\kappa_{1}$ measures the average deviation of clients' MDP from the imaginary MDP, and $\kappa_{2}$ measures the average distance between each pair of clients in terms of the transition probability. With homogeneous environments, both $\kappa_{1}$ and $\kappa_{2}$ will be equal to $0$. As the transition probability of each MDP gets farther away from each other, $\kappa_{1}$ and $\kappa_{2}$ tend to be larger and indicate a heterogeneous network. It is trivial to show that $\kappa_{1} \le \kappa_{2}$.

\subsection{Federated Policy Evaluation}

The following lemmas describe the relation between the averaged value function $\bar{V}^{\pi}(s) = \sum_{n=1}^{N} q_{n} V_{n}^{\pi}(s), \forall s \in \mathcal{S}$ of policy $\pi$ and the value function $V_{I}^{\pi}$ of policy $\pi$ in the imaginary MDP.
\begin{lemma}
\label{lemma:barvgevi}
For all states $s$ and policies $\pi$, we have $\bar{V}^{\pi}(s) \ge V^{\pi}_{I}(s)$.
\end{lemma}
\begin{lemma}
\label{lemma:normbarvdiffvi}
For all policies $\pi$, we have $\left\Vert \bar{V}^{\pi} - V_{I}^{\pi} \right\Vert \le \frac{\gamma R_{\max} \kappa_{1}}{(1 - \gamma)^{2}}$.
\end{lemma}
Readers are referred to \cite{pmlr-v151-jin22a} for detailed proofs and discussions of Lemmas \ref{lemma:barvgevi} and \ref{lemma:normbarvdiffvi}. Briefly, $V_{I}^{\pi}$ serves as a lower bound of $\bar{V}^{\pi}$. Among all policies, the optimal policy $\pi_{I}^{\ast}$ in the imaginary MDP is of particular interest since its value function $V_{I}^{\pi_{I}^{\ast}} = V_{I}^{\ast}$ is the largest lower bound of the average value function $\bar{V}^{\pi}$.
Since there are approximation errors (\ref{eq:deltan}) in each client and aggregation error in the central server, we can only obtain an approximation $V^{t}$ of the averaged value function $\bar{V}^{\pi^{t}}$ with
\begin{alignat}{1}
\left\Vert V^{t} - \bar{V}^{\pi^{t}} \right\Vert \le \left\Vert V^{t} - \bar{V}^{t} \right\Vert + \left\Vert \bar{V}^{t} - \bar{V}^{\pi^{t}} \right\Vert & \le \bar{\varepsilon}_{w} + \sum_{n=1}^{N} q_{n} \delta_{n} = \bar{\delta}, \nonumber
\end{alignat}
where $\bar{V}^{t}(s) = \sum_{n}^{N} q_{n} V_{n}^{t}(s), \forall s \in \mathcal{S}$ is the average of the local approximation and $\bar{\varepsilon}_{w} = \max_{t} \left\Vert V^{t} - \bar{V}^{t} \right\Vert$ is the error induced by value function aggregation. Note that we postpone the discussion of the error $\varepsilon_{w}$ induced by value function aggregation to the end of this section. Therefore, by the triangle inequality, we have
\begin{alignat}{1}
\left\Vert V^{t} - V_{I}^{\pi^{t}} \right\Vert & \le \frac{\gamma R_{\max} \kappa_{1}}{(1 - \gamma)^{2}} + \bar{\delta} = \dot{\delta}. \label{eq:delta}
\end{alignat}

\subsection{Federated Policy Improvement}

Note that the approximation error in (\ref{eq:delta}) matches that of the policy evaluation of API (\ref{eq:apipe}). This motivates us to further obtain an API-style approximation error for the policy improvement phase of FAPI. We consider two variants of FAPI: FAPI with federated policy evaluation (Algorithm \ref{alg:FAPI1}) and FAPI without federated policy evaluation (Algorithm \ref{alg:FAPI2}). 
Algorithm \ref{alg:FAPI1} is communication inefficient in practice since it introduces an extra round of communication for policy evaluation (lines 3 - 8 in Algorithm \ref{alg:FAPI1}). The two algorithms lead to different approximation errors of the policy improvement phase as will be shown by Lemma \ref{lemma:wfpepieb} and Lemma \ref{lemma:wofpepieb}, respectively.

\begin{algorithm}
\caption{\label{alg:FAPI1}FAPI with federated policy evaluation}
\begin{algorithmic}[1]
\STATE{\textbf{Input:} N, T.}
	\FOR{t = 0,1,...,T}
            \STATE{Synchronize the global policy $\pi^{t}$ to every client.}
		\FOR{n = 0,1,...,N}
			\STATE{Approximate the value function $V^{\pi^{t}}_{n}$ with $V^{t}_{n}$.}
                \STATE{Upload $V^{t}_{n}$ to the central server.}
		\ENDFOR
		\STATE{The central server aggregates $V^{t}_{n}$ to obtain $V^{t}$: $w^{t+1} \leftarrow \sum_{n=1}^{N}q_{n}w_{n}^{t+1}$.}
            \STATE{Synchronize the global policy $\pi^{t}$ and value function $V^{t}$ to every client.}
		\FOR{n = 0,1,...,N}
			\STATE{Local update of client policy: $\left\Vert T^{\pi^{t+1}_{n}}_{n} V^{t} - T_{n} V^{t} \right\Vert \le \epsilon_{n}$.}
                \STATE{Upload $\pi^{t+1}_{n}$ to the central server.}
		\ENDFOR
		\STATE{The central server aggregates $\pi^{t+1}_{n}$ to obtain $\pi^{t+1}$: $\theta^{t+1} \leftarrow \sum_{n=1}^{N}q_{n}\theta_{n}^{t+1}$.}
	\ENDFOR
\end{algorithmic}
\end{algorithm}

\begin{algorithm}
\caption{\label{alg:FAPI2}FAPI without federated policy evaluation}
\begin{algorithmic}[1]
\STATE{\textbf{Input:} N, T.}
	\FOR{t = 0,1,...,T}
            \STATE{Synchronize the global policy $\pi^{t}$ to every client.}
		\FOR{n = 0,1,...,N}
                \STATE{Approximate the value function $V^{\pi^{t}}_{n}$ with $V^{t}_{n}$.}
			\STATE{Local update of client policy: $\left\Vert T^{\pi^{t+1}_{n}}_{n} V^{t}_{n} - T_{n} V^{t}_{n} \right\Vert \le \epsilon_{n}$.}
                \STATE{Upload $\pi^{t+1}_{n}$ to the central server.}
		\ENDFOR
		\STATE{The central server aggregates $\pi^{t+1}_{n}$ to obtain $\pi^{t+1}$: $\theta^{t+1} \leftarrow \sum_{n=1}^{N}q_{n}\theta_{n}^{t+1}$.}
	\ENDFOR
\end{algorithmic}
\end{algorithm}
\begin{lemma}
\label{lemma:wfpepieb}
With federated policy evaluation, the sequence $\left(\pi^{t}\right)_{t=0}^{\infty}$ generated by Algorithm \ref{alg:FAPI1} satisfies
\begin{alignat}{1}
\left\Vert T_{I}^{\pi^{t+1}} V^{t} - T_{I} V^{t} \right\Vert \le \frac{\bar{\varepsilon}_{\theta} \sqrt{\left\vert \mathcal{A} \right\vert} R_{\max}}{1 - \gamma} + \frac{2 \gamma R_{\max} \kappa_{1}}{1 - \gamma} + \bar{\epsilon} = \epsilon^{\prime}, \nonumber
\end{alignat}
where $\bar{\epsilon} = \sum_{n=1}^{N} q_{n} \epsilon_{n}$ and $\bar{\varepsilon}_{\theta} = \max_{t>0,s \in \mathcal{S}} \left\Vert \pi^{t+1}(\cdot \vert s) - \sum_{n=1}^{N} q_{n} \pi^{t+1}_{n}(\cdot \vert s) \right\Vert_{2}$.
\end{lemma}
See Appendix \ref{proof:lemma:wfpepieb} for the detailed proof. Note that we postpone the discussion of the error $\varepsilon_{\theta}$ induced by policy aggregation to the end of this section.
\begin{lemma}
\label{lemma:wofpepieb}
Without federated policy evaluation, the sequence $\left(\pi^{t}\right)_{t=0}^{\infty}$ generated by Algorithm \ref{alg:FAPI2} satisfies
\begin{alignat}{1}
\left\Vert T_{I}^{\pi^{t+1}} V^{t} - T_{I} V^{t} \right\Vert & \le \frac{\bar{\varepsilon}_{\theta} \sqrt{\left\vert \mathcal{A} \right\vert} R_{\max}}{1 - \gamma} + 2 \gamma \bar{\varepsilon}_{w} + \frac{2 \gamma^{2} R_{\max} \kappa_{2}}{(1 - \gamma)^{2}} + \frac{\gamma R_{\max} \kappa_{1}}{1 - \gamma} + 4 \gamma \bar{\delta} + \bar{\epsilon} = \dot{\epsilon}. \label{eq:lemma:wofpepieb1}
\end{alignat}
\end{lemma}
See Appendix \ref{proof:lemma:wofpepieb} for the detailed proof. As shown by Lemma \ref{lemma:wfpepieb} and Lemma \ref{lemma:wofpepieb}, FAPI with federated policy evaluation provides a tighter bound. Intuitively, forcing clients to optimize their local policy from the same starting point (value function $\bar{V}^{\pi^{t}}$) helps them mitigate the negative impact of heterogeneity. While Algorithm \ref{alg:FAPI2} is what FRL applications typically employ in practice, it is also more vulnerable to environment heterogeneity as its error bound is inferior to the one of Algorithm \ref{alg:FAPI1} roughly by a factor of $1 - \gamma$.

By far, we have assumed full client participation, i.e., all clients participate in every round of training. However, partial client participation is more favorable in practice. Proposition \ref{proposition:errorboundpartial} accounts for this scenario.
\begin{proposition}
\label{proposition:errorboundpartial}
Let $\mathcal{C}$ denote the set of selected clients and $q^{\prime}_{m} = \frac{q_{m}}{\sum_{m \in \mathcal{C}} q_{m}}$. With partial client participation and without federated policy evaluation, the sequence $\left(\pi^{t}\right)_{t=0}^{\infty}$ generated by Algorithm \ref{alg:FAPI2} satisfies
\begin{alignat}{1}
\left\Vert T_{I}^{\pi^{t+1}}V^{t} - T_{I}V^{t} \right\Vert & \le \frac{\tilde{\varepsilon}_{\theta} \sqrt{\left\vert \mathcal{A} \right\vert} R_{\max}}{1 - \gamma} + 2 \gamma \bar{\varepsilon}_{w} + \sum_{m\in\mathcal{C}} \sum_{n=1}^{N} q^{\prime}_{m} q_{n} \frac{\left( \gamma + \gamma^{2} \right) R_{\max} \kappa_{m,n}}{(1 - \gamma)^{2}} \nonumber \\
& \quad + \sum_{m\in\mathcal{C}} q^{\prime}_{m} \frac{\gamma R_{\max} \kappa_{m,I}}{1 - \gamma} + 2 \gamma \sum_{m \in \mathcal{C}} q^{\prime}_{m} \delta_{m} + 2 \gamma \bar{\delta} + \sum_{m \in \mathcal{C}} q^{\prime}_{m} \epsilon_{m} = \hat{\epsilon}, \label{eq:proposition:errorboundpartial1}
\end{alignat}
where $\tilde{\varepsilon}_{\theta} = \max_{t>0,s \in \mathcal{S}} \left\Vert \pi^{t}(\cdot \vert s) - \sum_{m\in\mathcal{C}} q^{\prime}_{m} \pi^{t}_{m}(\cdot \vert s) \right\Vert_{2}$.
\end{proposition}
See Appendix \ref{proof:proposition:errorboundpartial} for the detailed proof. To better understand how to minimize the right-hand side of (\ref{eq:proposition:errorboundpartial1}), one can consider the optimization problem, $\min_{x} f(x) = \frac{1}{N} \sum_{n=1}^{N} \left\vert x - a_{j} \right\vert$, which corresponds to the case $q_{n}=\frac{1}{N},n=1,\cdots,N$ and $\left\vert \mathcal{C} \right\vert=1$ in (\ref{eq:proposition:errorboundpartial1}). It can be easily shown that $f(x)$ is minimal when $x$ is the median of $\{ a_{1},\cdots,a_{N} \}$.
\begin{remark}
\label{remark:errorboundpartial}
Proposition \ref{proposition:errorboundpartial} reveals a remarkable fact that a proper client selection method can effectively reduce the error bound of the policy improvement phase. In particular, to have the right-hand side of (\ref{eq:proposition:errorboundpartial1}) smaller than the right-hand side of (\ref{eq:lemma:wofpepieb1}), the selected clients shall have an average $\kappa_{m,n}$ that is smaller than $\frac{2 \gamma}{1 + \gamma} \kappa_{2}$. This encourages FAPI to select clients that are closer to the imaginary MDP, which is a reasonable approximation to the median of all transition probabilities.
\end{remark}

For completeness, we provide the error bound of the policy improvement phase with partial client participation and federated policy evaluation by the following proposition.

\begin{proposition}
\label{proposition:wfpeerrorboundpartial}
Let $\mathcal{C}$ denote the set of selected clients and $q^{\prime}_{m} = \frac{q_{m}}{\sum_{m \in \mathcal{C}} q_{m}}$. With partial client participation and federated policy evaluation, the sequence $\left(\pi^{t}\right)_{t=0}^{\infty}$ generated by Algorithm \ref{alg:FAPI1} satisfies
\begin{alignat}{1}
\left\Vert T_{I}^{\pi^{t+1}}V^{t} - T_{I}V^{t} \right\Vert & \le \frac{\tilde{\varepsilon}_{\theta} \sqrt{\left\vert \mathcal{A} \right\vert} R_{\max}}{1 - \gamma} + \sum_{m\in\mathcal{C}} \sum_{n=1}^{N} q^{\prime}_{m} q_{n} \frac{ \gamma R_{\max} \kappa_{m,n}}{1 - \gamma} \nonumber \\
& \quad + \sum_{m\in\mathcal{C}} q^{\prime}_{m} \frac{\gamma R_{\max} \kappa_{m,I}}{1 - \gamma} + \sum_{m \in \mathcal{C}} q^{\prime}_{m} \epsilon_{m} = \acute{\epsilon}. \label{eq:proposition:wfpeerrorboundpartial1}
\end{alignat}
\end{proposition}
See Appendix \ref{proof:proposition:wfpeerrorboundpartial} for the detailed proof.

\subsection{Bounding the distance between value functions}

The following theorem provides the error bound of FAPI in terms of the distance between value functions.

\begin{theorem}
\label{theorem:errorboundfinal}
Let $\pi^{\ast} = \arg\max_{\pi} \eta(\pi)$. The sequence $\left(\pi^{t}\right)_{t=0}^{\infty}$ generated by FAPI satisfies
\begin{alignat}{1}
\limsup_{t\to\infty} \left\vert \bar{V}^{\pi^{t}}(s) - \bar{V}^{\max}_{s} \right\vert & \le \frac{\tilde{\epsilon} + 2 \gamma \dot{\delta}}{(1 - \gamma)^{2}} + 2 \frac{\gamma R_{\max} \kappa_{1}}{(1 - \gamma)^{2}}, \label{eq:theorem:errorboundfinal1}
\end{alignat}
where $\bar{V}^{\max}_{s} = \max \left\{ \bar{V}^{\pi^{\ast}}(s), \bar{V}^{\pi^{\ast}_{I}}(s) \right\},\forall s \in \mathcal{S}$, and $\tilde{\epsilon}$ may be one of $\dot{\epsilon}$, $\epsilon^{\prime}$, $\hat{\epsilon}$ or $\acute{\epsilon}$. More specifically, the error bound for Algorithm \ref{alg:FAPI2} with partial client participation is
\begin{alignat}{1}
\limsup_{t\to\infty} \left\vert \bar{V}^{\pi^{t}}(s) - \bar{V}^{\max}_{s} \right\vert & \le C_{1} \kappa_{1} + C_{2} \sum_{m\in\mathcal{C}} q^{\prime}_{m} \kappa_{m,I} + C_{3} \sum_{m\in\mathcal{C}} \sum_{n=1}^{N} q^{\prime}_{m} q_{n} \kappa_{m,n} \nonumber \\
& \quad + \tilde{\mathcal{O}}\left( \tilde{\varepsilon}_{\theta} \sqrt{\left\vert \mathcal{A} \right\vert} R_{\max} + \bar{\varepsilon}_{w} + \bar{\delta} + \sum_{m \in \mathcal{C}} q^{\prime}_{m} \delta_{m} + \sum_{m \in \mathcal{C}} q^{\prime}_{m} \epsilon_{m} \right), \label{eq:theorem:errorboundfinal2}
\end{alignat}
where $\tilde{\mathcal{O}}$ omits some constants related to $\gamma$, $C_{1} = \frac{2 \gamma \left(\gamma^{2} - \gamma + 1 \right)}{(1 - \gamma)^{4}} R_{\max}$, $C_{2} = \frac{\gamma}{(1 - \gamma)^{3}} R_{\max}$, and $C_{3} = \frac{\gamma + \gamma^{2}}{(1 - \gamma)^{4}} R_{\max}$.
\end{theorem}
See Appendix \ref{proof:theorem:errorboundfinal} for the detailed proof.
\begin{remark}
\label{remark:errorboundfinal1}
The bound given by Theorem \ref{theorem:errorboundfinal} is similar to the one of centralized API as in Proposition \ref{proposition:BertsekasProp2.4.3} and inversely proportional to the heterogeneity level $\kappa_{1}$ and $\kappa_{2}$, which explicitly unveils the impact of the data heterogeneity. The second term in (\ref{eq:theorem:errorboundfinal1}) stems from the difference between $\left\Vert \bar{V}^{\pi^{t}} - \bar{V}^{\pi^{\ast}} \right\Vert$ and $\left\Vert V^{\pi^{t}}_{I} - V_{I}^{\ast} \right\Vert$. Although the imaginary MDP enables us to analyze the theoretical properties of FAPI and shows the error bound in terms of $V_{I}^{\ast}$, $\bar{V}^{\pi^{\ast}}$ is the actual target (refer to (\ref{eq:4})) that FAPI wants to achieve. Fortunately, this bound is still 
useful, since (\ref{eq:theorem:errorboundfinal1}) is dominated by its first term. To this end, the optimal policy in the imaginary MDP is a good estimation of the optimal policy for (\ref{eq:4}) unless there is a bound sharper than $\dot{\delta}$. Again, the client selection scheme is the key to reducing error.
\end{remark}

\subsection{Impact of Aggregation Error \label{subsec:aggregation_error}}

We have defined three terms to quantify the impact of policy aggregation and value function aggregation, i.e., $\tilde{\varepsilon}_{\theta}$, $\bar{\varepsilon}_{\theta}$ and $\bar{\varepsilon}_{w}$. To further investigate how these errors affect the convergence of FAPI, there are two possible approaches: 1) Performing a general analysis with a few standard assumptions in the convex optimization literature; and 2) Carrying out the analysis with a specific neural network parameterization to gain insight into network configuration that can affect the approximation error. In light of the recent breakthrough in overparameterized neural networks \citep{pmlr-v97-arora19a,NEURIPS2019_98baeb82,wang2019neural,NEURIPS2019_227e072d}, we employ the second strategy. To that end, we analyze the impact of aggregation error with the following two-layer ReLU-activated neural networks to parameterize the state value function and policy, respectively, as
\begin{alignat}{1}
u_{w}(s) = \frac{1}{\sqrt{m}} \sum_{i}^{m} b_{i} \cdot \sigma(w_{i}^{T} (s) ), \label{eq:nn1} \\
f_{\theta}(s,a) = \frac{1}{\sqrt{m}} \sum_{i}^{m} b_{i} \cdot \sigma(\theta_{i}^{T} (s,a) ), \label{eq:nn2}
\end{alignat}
where $m$ is the width of the network, $\theta = (\theta_{i}, \dots \theta_{m})^{T} \in \mathcal{R}^{m (d_{s}+d_{a})}$, $w = (w_{i}, \dots w_{m})^{T} \in \mathcal{R}^{m d_{s}}$ denotes the input weights, and $b_{i} \in \{-1, 1\}, \forall i \in [m]$ is the output weight. Without loss of generality, we assume that $a \in \mathbb{R}^{d_{a}}, s \in \mathbb{R}^{d_{s}}$, $\left\Vert (s,a) \right\Vert_{2} \le 1$, and denote $d = d_{s} + d_{a}$. Given arbitrary $\hat{R}_{\vartheta} < \infty$, we consider the following parameter initialization, where $\vartheta$ is a placeholder for $\theta$ and $w$:
\begin{alignat}{1}
\mathbb{E}_{\text{init}} \left[ \vartheta^{0}_{i,j} \right] = 0, \mathbb{E}_{\text{init}} \left[ \left( \vartheta^{0}_{i,j} \right)^{2} \right] = \frac{1}{d \cdot m}, \forall i \in [m], j \in [d], \nonumber \\
b_{i} \thicksim \text{Unif}(\{-1,1\}), \mathbb{E}_{\text{init}} \left[ 
\left\Vert \vartheta^{0}_{i} \right\Vert_{2}^{-2} \right] < \infty, \left\Vert \vartheta^{0}_{i} \right\Vert_{2} \le \hat{R}_{\vartheta}, \forall i \in [m]. \label{eq:initialization}
\end{alignat}
In other words, $0 < \left\Vert \vartheta^{0}_{i} \right\Vert_{2} \le \hat{R}_{\vartheta}, \forall i \in [m]$. The output weights $b_{i}, \forall i \in [m]$ are fixed. The input weights $\vartheta_{i}, \forall i \in [m]$ are projected into a ball centered at the initial parameter, i.e., $\theta \in \mathcal{B}^{0}_{\mathcal{R}_{\theta}} = \left\{ \theta: \left\Vert \theta - \theta^{0} \right\Vert_{2} \le \mathcal{R}_{\theta} \right\}, w \in \mathcal{B}^{0}_{\mathcal{R}_{w}} = \left\{ w: \left\Vert w - w^{0} \right\Vert_{2} \le \mathcal{R}_{w} \right\}$. Then, the state value function and Softmax policy are parameterized by
\begin{alignat}{1}
V^{t}(s) = u_{w^{t}}(s), \pi^{t}(a \vert s) = \frac{\exp(f_{\theta^{t}}(s,a))}{\sum_{a^{\prime} \in \mathcal{A}} \exp(f_{\theta^{t}}(s,a^{\prime}))}, \forall a \in \mathcal{A}, s \in \mathcal{S}. \label{eq:policy_parameterization}
\end{alignat}
The following lemma quantifies the aggregation error under the above setting.

\begin{lemma}
\label{lemma:aggregation_error}
By utilizing the initialization scheme in (\ref{eq:initialization}), the policy parameterization in (\ref{eq:policy_parameterization}), and the neural network parameterization in (\ref{eq:nn1}) and (\ref{eq:nn2}), we have
\begin{alignat}{1}
\mathbb{E}_{\text{init}} \left[ \bar{\varepsilon}_{w} \right] &= \mathcal{O} \left( R_{w}^{6/5} m^{-1/10} \hat{R}_{w}^{2/5} \right), \label{eq:lemma:aggregation_error_first} \\
\mathbb{E}_{\text{init}} \left[ \bar{\varepsilon}_{\theta} \right] &= \mathcal{O}\left( R_{\theta}^{1/2} \right), \label{eq:lemma:aggregation_error_second} \\
\mathbb{E}_{\text{init}} \left[ \tilde{\varepsilon}_{\theta} \right] &= \mathcal{O}\left( R_{\theta}^{1/2} \right). \label{eq:lemma:aggregation_error_third}
\end{alignat}
\end{lemma}
See Appendix \ref{proof:lemma:aggregation_error} for the detailed proof.
\begin{remark}
\label{remark:aggregation_error}
Lemma \ref{lemma:aggregation_error} indicates that the aggregation error is determined by the neural network parameterization ($m$) and optimization method ($R_{\theta}$ and $R_{w}$). Moreover, it can be observed from the proof that the error of value function aggregation is minimized when using linear function approximation, e.g., linear regression and tabular implementation, and the error of policy aggregation is minimized when using log-linear policy parameterization, e.g., Softmax policy with linear regression. In other words, there is a trade-off between training (non-linearity) and aggregation error (linearity). A non-linear function approximation can fit the training data well, but suffers from high aggregation error. In contrast, the linear function approximation may not be able to solve complex problems, but it does not introduce aggregation error.
\end{remark}

\subsection{Connection to Centralized Learning \label{subsec:connections}}

When the environments are homogeneous, the policy is log-linear w.r.t. parameters, and the value function is linear w.r.t. parameters, the learning process of FAPI will be equivalent to learning from N copies of the same environment (that is identical to the imaginary MDP) with federated policy evaluation. Under such circumstances, Theorem \ref{theorem:errorboundfinal} degenerates to that for centralized learning (\ref{eq:proposition:BertsekasProp2.4.3}) \citep{bertsekas2022abstract} with the same error bound, i.e., $\frac{\bar{\epsilon} - 2 \gamma \bar{\delta}}{(1 - \gamma)^2}$, as shown in the following proposition.

\begin{proposition}
\label{proposition:homoerrorboundfinal}
With homogeneous environments and federated policy evaluation, the error bound for partial/full client participation is
\begin{alignat}{1}
\limsup_{t\to\infty} \left\Vert \bar{V}^{\pi^{t}} - \bar{V}^{\pi^{\ast}} \right\Vert & \le \frac{\hat{\varepsilon}_{\theta} \sqrt{\left\vert \mathcal{A} \right\vert} R_{\max}}{\left( 1 - \gamma \right)^{3}} + \frac{2 \gamma \bar{\varepsilon}_{w}}{\left( 1 - \gamma \right)^{2}} + \frac{\bar{\epsilon} + 2 \gamma \bar{\delta}}{(1 - \gamma)^2}, \nonumber
\end{alignat}
where $\hat{\varepsilon}_{\theta}$ is equal to $\bar{\varepsilon}_{\theta}$ and $\tilde{\varepsilon}_{\theta}$ for full participation and partial participation, respectively.
\end{proposition}
See Appendix \ref{proof:proposition:homoerrorboundfinal} for the detailed proof. As discussed in Section \ref{subsec:aggregation_error}, $\hat{\varepsilon}_{\theta} = 0$ when the policy is log-linear w.r.t. parameters, and $\bar{\varepsilon}_{w} = 0$ when the value function is linear w.r.t. parameters.

\section{Federated Policy Optimization with Heterogeneity-aware Client Selection\label{sec:ebalgorithm}}

In this section, we propose a federated policy optimization algorithm based on the discussions in Section \ref{sec:ebFAPI}. The pseudocode for the proposed Federated Policy Optimization with Heterogeneity-aware Client Selection (FedPOHCS) is illustrated in Algorithm \ref{alg:FedPOHCS}.

\subsection{Client Selection Metric}

As characterized by Remark \ref{remark:errorboundpartial}, a proper client selection method should be able to capture the heterogeneity level $\kappa_{1}/\kappa_{2}$ of the selected clients. In general, the smaller the difference between the transition probability of selected clients and that of the imaginary MDP, the better bound we may obtain. As there are many methods to approximate and represent the transition probability, we consider it as an implementation consideration and leave it to the applications. The use of transition probability makes the proposed client selection scheme a model-based framework \citep{DBLP:journals/corr/abs-2006-16712}.

Next, we make several approximations to the theoretically justified client selection metric, i.e., the level of heterogeneity of the $n$-th client $\kappa_{n,I}$ defined in Definition \ref{definition:impacthete}. It is hard to compute $\kappa_{n,I}$ by finding the maximum value over all states since the number of samples is finite in practice and this metric may suffer from high variance. Therefore, we use the current global policy to compute the metric and weight each transition with the stationary (or steady-state) distribution $d_{\pi,P_{n}}(s)$ of the entry state $s$ \citep{NEURIPS2020_69bfa2aa}. Moreover, in the proofs of the lemmas and propositions, we relax the inequalities by replacing all value functions with their upper bound $\frac{R_{\max}}{1-\gamma}$. To further improve the approximation, we scale each transition with the value (advantage or Q-value) of each state-action pair $(s^{\prime}, a)$. Then, for each tuple of $(s, s^{\prime},a)$ in the n-th client, we have
\begin{alignat}{1}
\hat{\kappa}_{n,I}(s,s^{\prime},a) = \left\vert d_{\pi,P_{n}}(s) P_{n}^{\pi}(s^{\prime} | s) A_{\pi,P_{n}}(s^{\prime},a) - \sum_{j=1}^{N} q_{j} d_{\pi,P_{j}}(s) P_{j}^{\pi}(s^{\prime} | s) A_{\pi,P_{j}}(s^{\prime},a) \right\vert. \nonumber
\end{alignat}
However, it is too expensive to compute all $\vert \mathcal{S}\vert \times \vert \mathcal{S} \vert \times \vert \mathcal{A} \vert$ elements for each client. Since the stationary distribution is a fixed-point distribution, i.e., $\sum_{s} d_{\pi,P_{n}}(s) P_{n}^{\pi}(s^{\prime} | s) = d_{\pi,P_{n}}(s^{\prime})$, we can sum over the first dimension to simplify the metric and reduce the amount of computation. In other words, we want to compute
\begin{alignat}{1}
\hat{\kappa}_{n,I}(s^{\prime},a) & = \left\vert \sum_{s} d_{\pi,P_{n}}(s) P_{n}^{\pi}(s^{\prime} | s) A_{\pi,P_{n}}(s^{\prime},a) - \sum_{j=1}^{N} q_{j} \sum_{s} d_{\pi,P_{j}}(s) P_{j}^{\pi}(s^{\prime} | s) A_{\pi,P_{j}}(s^{\prime},a) \right\vert \nonumber \\
& = \left\vert d_{\pi,P_{n}}(s^{\prime}) A_{\pi,P_{n}}(s^{\prime},a) - \sum_{j=1}^{N} q_{j} d_{\pi,P_{j}}(s^{\prime}) A_{\pi,P_{j}}(s^{\prime},a) \right\vert. \nonumber
\end{alignat}

For compact notation, we define $\mathbf{D}_{\pi, P_{n}}$ and $\mathbf{D}_{\pi, \mu, P_{n}, \gamma}$  as two $\left|\mathcal{S}\right|\times\left|\mathcal{S}\right|$ diagonal matrices whose $i$-th diagonal entries are the stationary distribution and discounted visitation frequencies of state $s_{i}$, respectively. We denote $\mathbf{\Pi}_{\pi}$ as a $\left|\mathcal{S}\right|\times\left|\mathcal{A}\right|$ matrix whose $(i,j)$th entry is $\pi(a_{j}\vert s_{i})$, and $\mathbf{A}_{\pi,P_{n}}$ as a $\left|\mathcal{S}\right|\times\left|\mathcal{A}\right|$ matrix whose $(i,j)$th entry is the advantage of action $a_{j}$ on state $s_{i}$, i.e., a matrix representation of the advantage function \citep{10.5555/645531.656005}. Then, we can rewrite the approximated level of heterogeneity $\sum_{s^{\prime},a} \hat{\kappa}_{n,I}(s^{\prime},a)^{2}$ as $\hat{\kappa}_{n,I} = \left\Vert \sum_{k=1}^{N} q_{k} \mathbf{D}_{\pi,P_{k}} \mathbf{A}_{\pi,P_{k}} - \mathbf{D}_{\pi,P_{n}} \mathbf{A}_{\pi,P_{n}} \right\Vert_{F}$, where we use the Frobenius norm to make the metric more sensitive to entries with large difference.

However, the metric $\hat{\kappa}_{n,I}$ has an obvious drawback, i.e., the algorithm will keep selecting clients that are closer to the imaginary MDP even if those clients are sufficiently trained. As a solution, we propose to consider the learning potential of local policies. In fact, the learning step size of many policy optimization algorithms depends on the magnitude of value functions. For example, the advantage function (or Q-value) can affect the magnitude of policy gradient, and the Q-value can affect the updates in Temporal Difference (TD) methods, including Q-learning and SARSA. This observation motivates us to use the magnitude of the advantage function, i.e., $\left\Vert \mathbf{D}_{\pi,P_{n}} \mathbf{A}_{\pi,P_{n}} \right\Vert_{F}$, to measure how much the $n$-th client can learn starting from the current global policy. Finally, we define the selection metric of FedPOHCS as
\begin{alignat}{1}
\Delta_{n} = \left\Vert \mathbf{D}_{\pi,P_{n}} \mathbf{A}_{\pi,P_{n}} \right\Vert_{F} - \left\Vert \sum_{k=1}^{N} q_{k} \mathbf{D}_{\pi,P_{k}} \mathbf{A}_{\pi,P_{k}} - \mathbf{D}_{\pi,P_{n}} \mathbf{A}_{\pi,P_{n}} \right\Vert_{F}, \; \forall n = 1,\cdots,N. \label{eq:DeltaN}
\end{alignat}
As shown in the first phase (lines 3 - 9) of Algorithm \ref{alg:FedPOHCS}, the server samples a candidate set with $d$ clients and computes $\Delta_{n}$ for all clients in this set. Then, the server selects $K$ clients with the largest $\Delta_{n}$ from the candidate set and starts the second phase (lines 10 - 15) of Algorithm \ref{alg:FedPOHCS}.

As a client selection metric, $\Delta_{n}$ is better than the original heterogeneity measurement $\kappa_{n,I}$, because it helps skip clients that can not sufficiently contribute to the global objective (\ref{eq:4}). This helps allocate resources to clients whose information has not been fully learned instead of investing all resources into those who are just closer to the imaginary MDP throughout the learning process. Similar measurements of heterogeneity and learning potential can be found in FedKL \citep{10038492}, where states are weighted by the discounted visitation frequency instead of the stationary distribution.

\subsection{Implementation}

To compute $\Delta_{n}$, we adopt the tabular maximum likelihood model \citep{DBLP:journals/corr/abs-2006-16712,JMLR:v10:strehl09a,ornik2021learning,STREHL20081309} in the implementation of FedPOHCS. Given a set of trajectories, let $C_{n}(s,a)$ denote the number of times action $a$ was taken under state $s$, $C_{n}(s,a,s^\prime)$ represent the number of times the MDP transited from state $s$ to $s\prime$ after taking action $a$, and $r_{n}$ denote the corresponding sequence of reward received. Generally speaking, the more trajectories we have, the more accurate the modeling will be. For Mountain Cars and Hoppers, the number of trajectories is 200. For HongKongOSMs, the number of trajectories is 1000. We can estimate the transition probability $P_{n}$ of the $n$-th MDP and the reward function $\mathcal{R}$ by
\begin{alignat}{1}
\hat{P}_{n} = \frac{C_{n}(s,a,s^\prime)}{C_{n}(s,a)}, \quad \hat{\mathcal{R}}_{n} = \frac{1}{C_{n}(s,a)} \sum_{i=1}^{C_{n}(s,a)} r_{n}[i]. \nonumber
\end{alignat}
We estimate the state visitation frequency matrix $\hat{\mathbf{D}}_{\pi, \mu, P_{n}}$ in place of $\mathbf{D}_{\pi, P_{n}}$ as follows \citep{ziebart2008maximum}:
\begin{alignat}{1}
D_{n,s^\prime,0} & = \mu(s^{\prime}), \quad D_{n,s^\prime,t+1} = \sum_{s,a}D_{n,s,t} \pi(a \vert s)\hat{P}_{n}(s^{\prime} \vert s, a), \quad D_{n,s^{\prime}} = \sum_{t}D_{n,s^{\prime},t}, \nonumber
\end{alignat}
where the time horizon $t$ is a hyperparameter. Then, we diagonalize the vector $D_{n,s^\prime}$ to obtain the state visitation frequency matrix $\hat{\mathbf{D}}_{\pi, \mu, P_{n}}$. Last, we estimate the advantage function by Generalized Advantage Estimation (GAE) \citep{DBLP:journals/corr/SchulmanMLJA15}.

In contrast to client selection schemes that update the selection metrics together with the models at the end of each round, we employ an extra round for metric computation (lines 3 - 9) in Algorithm \ref{alg:FedPOHCS}. A similar selection scheme was utilized by a biased client selection method called Power-of-Choice \citep{pmlr-v151-jee-cho22a}.

The choice of the local learner is optional as our discussion is general and does not rely on any particular implementation of policy evaluation and policy improvement. In particular, we assume all clients adopt the Proximal Policy Optimization (PPO) algorithm proposed by \cite{DBLP:journals/corr/SchulmanWDRK17}, which is a PG method motivated by TRPO.

\begin{algorithm}
\caption{\label{alg:FedPOHCS}FedPOHCS}
\begin{algorithmic}[1]
\STATE{\textbf{Input:} the initial estimation of the transition probabilities $\hat{P}_{k}$ and reward functions $\hat{\mathcal{R}}_{k}$, d, T, K.}
	\FOR{t = 0,1,...,T}
             \STATE{Sample the candidate client set $\mathcal{C}$ of $d$ ($K \le d \le N$) clients without replacement.}
             \STATE{Synchronize the global policy $\pi^{t}$ to every selected client.}
            \FOR{$k \in \mathcal{C}$}
                \STATE{Compute the advantage function $\mathbf{A}_{\pi,P_{k}}$ and the state visitation frequency $\hat{\mathbf{D}}_{\pi^{t}, \mu, P_{k}}$.}
                \STATE{Upload $\mathbf{A}_{\pi,P_{k}}, \hat{\mathbf{D}}_{\pi^{t}, \mu, P_{k}}$ to the central server.}
            \ENDFOR
		\STATE{Compute $\Delta_{k}, \forall{k=1,...,d}$, Select $K$ clients based on $\Delta_{k}$ and replace $\mathcal{C}$ with these clients.}
		\FOR{$k \in \mathcal{C}$}
			\STATE{Local update of client policy $\pi_{k}^{t+1}$, transition probability $\hat{P}_{k}$, and reward function $\hat{\mathcal{R}}_{k}$.}
                \STATE{Upload $\pi_{k}^{t+1}$ to the central server.}
		\ENDFOR
		\STATE{The central server performs aggregation: $\pi^{t+1}(a \vert s) \leftarrow \sum_{k=1}^{K}q_{k}\pi_{k}^{t+1}(a \vert s), \forall s \in \mathcal{S}, a \in \mathcal{A}$.}
	\ENDFOR
\end{algorithmic}
\end{algorithm}

\subsection{Limitations of the Proposed Implementation}

Algorithm \ref{alg:FedPOHCS} utilizes a two-phase communication scheme which is not communication-efficient unless the convergence improvement obtained by client selection suppresses this cost. However, we note that such a communication cost can be removed by using outdated information to compute the selection metrics at the cost of accuracy as in \cite{pmlr-v151-jee-cho22a}. In particular, we can remove the first phase, and order selected clients to upload their local information (e.g., $\mathbf{A}_{\pi, P_{n}}$ and $\mathbf{D}_{\pi, \mu, P_{n}}$ matrices) when uploading their models at the end of each communication round. This communication-efficient variant saves a lot of communication overhead and has been shown to be effective by [1], i.e., with slightly worse performance. The performance of the one-phase scheme will be shown later in the experiment results.

Another limitation of Algorithm \ref{alg:FedPOHCS} is that it requires the clients to upload their state visitation frequencies $\hat{\mathbf{D}}_{\pi, \mu, P_{n}}$ and advantage functions $\mathbf{A}_{\pi, P_{n}}$ which may contain private information. This problem can be addressed by privacy protection methods, e.g., Homomorphic Encryption (HE) \citep{jiang2018secure}.

\section{Experiments\label{sec:Experiments}}

In this section, we introduce two federated environments for empirical evaluation and evaluate the proposed client selection method from different perspectives. All experimental results are reproducible and can be accessed from: https://github.com/ShiehShieh/FedPOHCS.

\subsection{Environments}

We evaluate the effectiveness of the proposed client selection algorithm on a federated version of the mountain car continuous control problem \citep{Moore90efficientmemory-based} and the Mujoco Hopper problem \citep{coulom2002reinforcement}. Furthermore, we utilize the Flow simulator \cite{pmlr-v87-vinitsky18a} and OpenStreetMap (OSM) \citep{OpenStreetMap, 9119753} dataset to create a series of traffic networks for autonomous vehicle training. We construct the federated environments as follows:

\textbf{Mountain Cars} consists of 60 equally weighted MountainCarContinuous-v0 environments developed by OpenAI Gym \citep{DBLP:journals/corr/BrockmanCPSSTZ16}. In each environment, a car aims to reach a hill. The episode terminates when the car reaches this hill or runs out of time. The environment consists of a 2-dimensional continuous state space and a 1-dimensional continuous action space. To introduce heterogeneity into the system, we assume that the engine of each car is different and the $n$-th car shifts the input action by $\theta_{n}$ on all states. To introduce a medium-level heterogeneity, the constant shift $\theta_{n}$ is uniformly sampled from $[-1.5, 1.5]$ and assigned to each environment at initialization. In fact, to make the experiments traceable, we set the constant shift to $\theta_{n} = -1.5 + \frac{n}{20}, n=1,\dots,60$. The intervals for low-level and high-level heterogeneity are $[-1.0, 1.0]$ and $[-2.0, 2.0]$, respectively.

\textbf{Hoppers} consists of 60 equally weighted Hopper-v3 environments developed by OpenAI Gym. The environment has an 11-dimensional continuous state space and a 3-dimensional continuous action space. We introduce the heterogeneity into this system by following \cite{pmlr-v151-jin22a}, i.e., the leg size is uniformly sampled from $[0.01, 0.07]$, $[0.01, 0.10]$, and $[0.01, 0.15]$ for low-level, medium-level, and high-level heterogeneity, respectively.

\textbf{HongKongOSMs} consists of 10 equally weighted traffic networks, each based on one of the OSM datasets as shown in Figure \ref{fig:osms}. Each traffic network contains one RL-controlled and 10 IDM-controlled (Intelligent Driver Model) vehicles. The 18-dimensional observation includes headway, speed, and positional information of visible neighborhoods of the RL-controlled vehicle, and only observations for vehicles running on the same and adjacent lanes are visible to the RL-controlled vehicle. The 2-dimensional action includes acceleration and lane-changing decisions. We adopted the Eclipse Simulation of Urban MObility (SUMO) simulator to conduct this experiment.

\begin{figure}[!t]
\centering
\subfloat[]{\includegraphics[width=0.19\columnwidth]{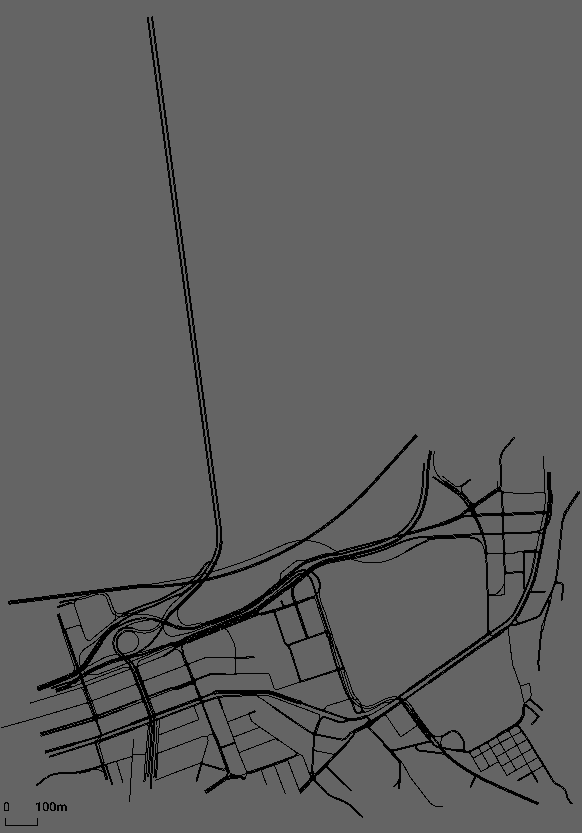}
\label{fig:causeway}}
\hfil
\subfloat[]{\includegraphics[width=0.19\columnwidth]{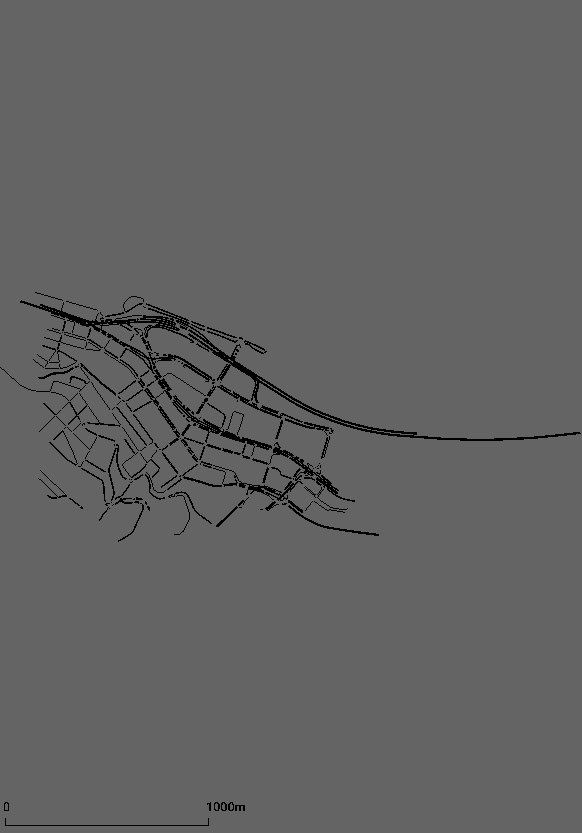}\label{fig:central}}
\hfil
\subfloat[]{\includegraphics[width=0.19\columnwidth]{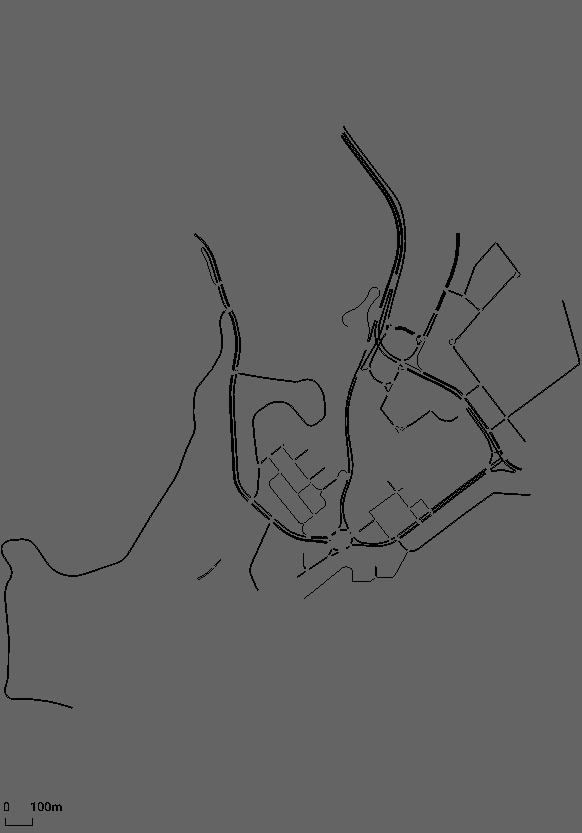}\label{fig:cw}}
\hfil
\subfloat[]{\includegraphics[width=0.19\columnwidth]{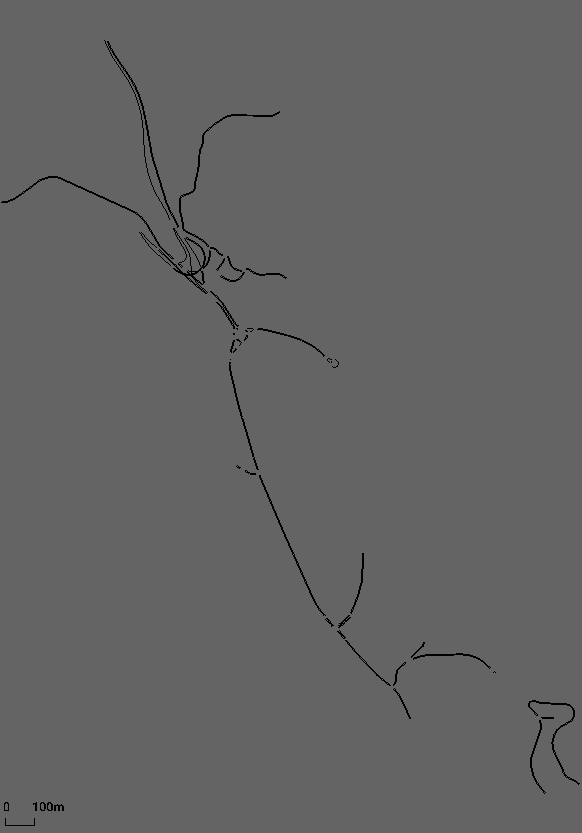}\label{fig:cwb}}
\hfil
\subfloat[]{\includegraphics[width=0.19\columnwidth]{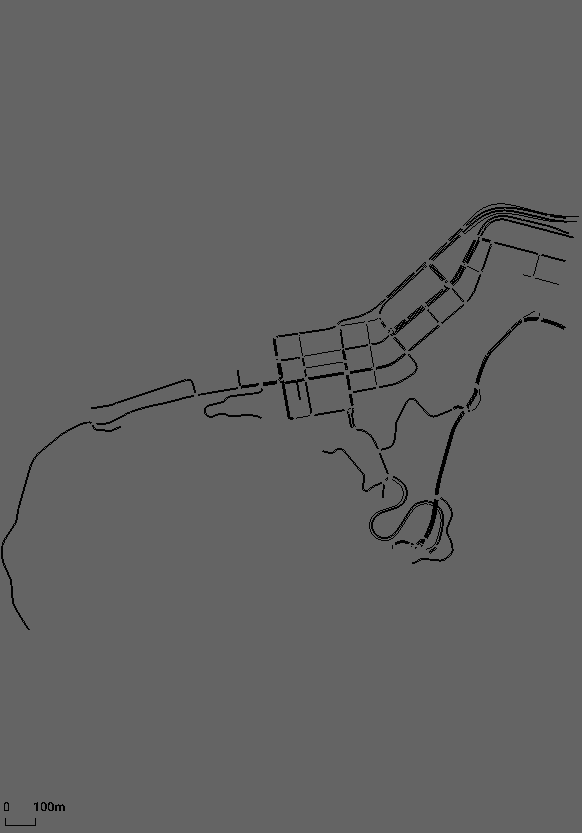}\label{fig:kennedy}}
\\
\subfloat[]{\includegraphics[width=0.19\columnwidth]{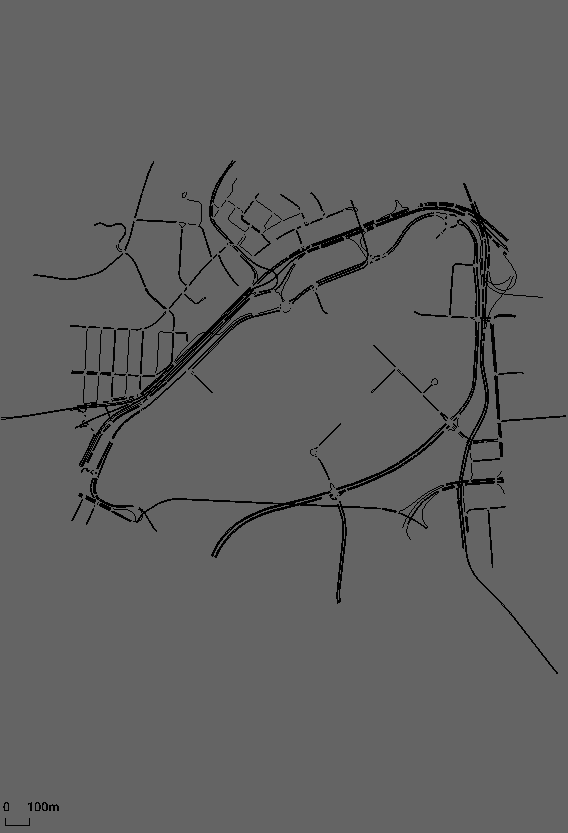}
\label{fig:kt}}
\hfil
\subfloat[]{\includegraphics[width=0.19\columnwidth]{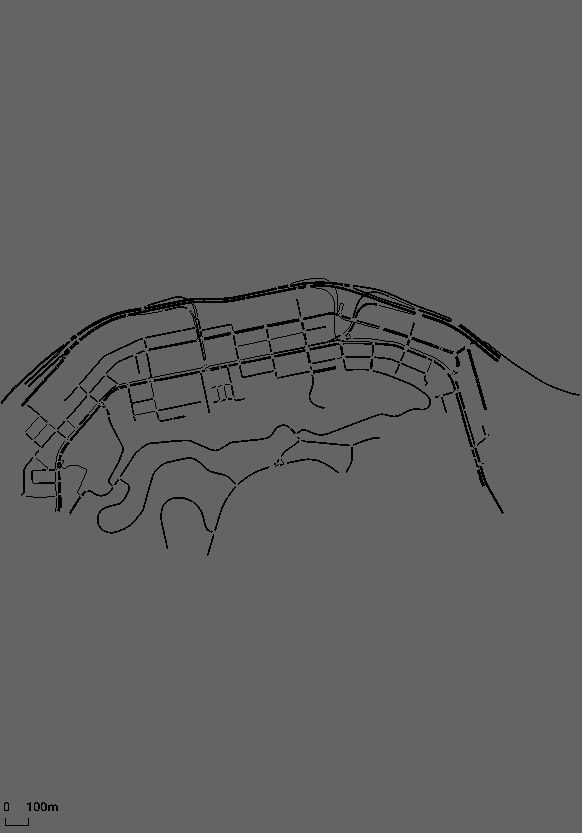}\label{fig:np}}
\hfil
\subfloat[]{\includegraphics[width=0.19\columnwidth]{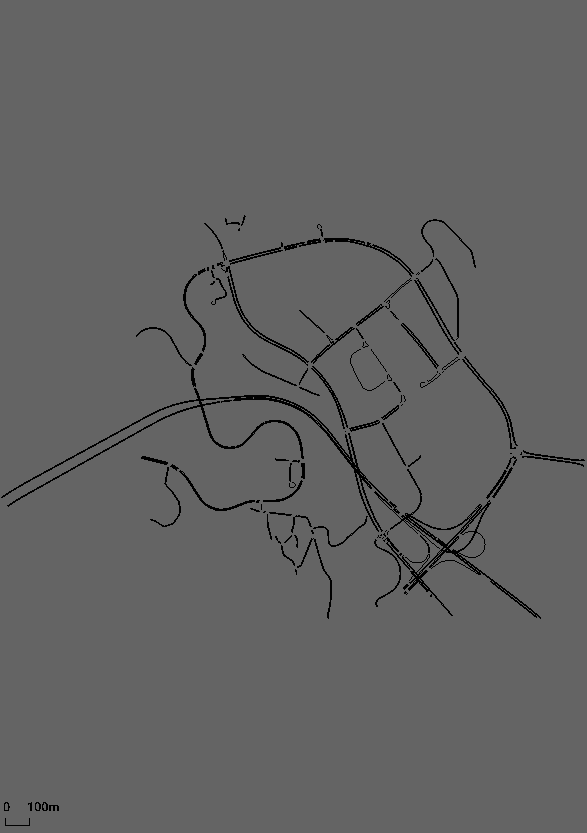}\label{fig:yt}}
\hfil
\subfloat[]{\includegraphics[width=0.19\columnwidth]{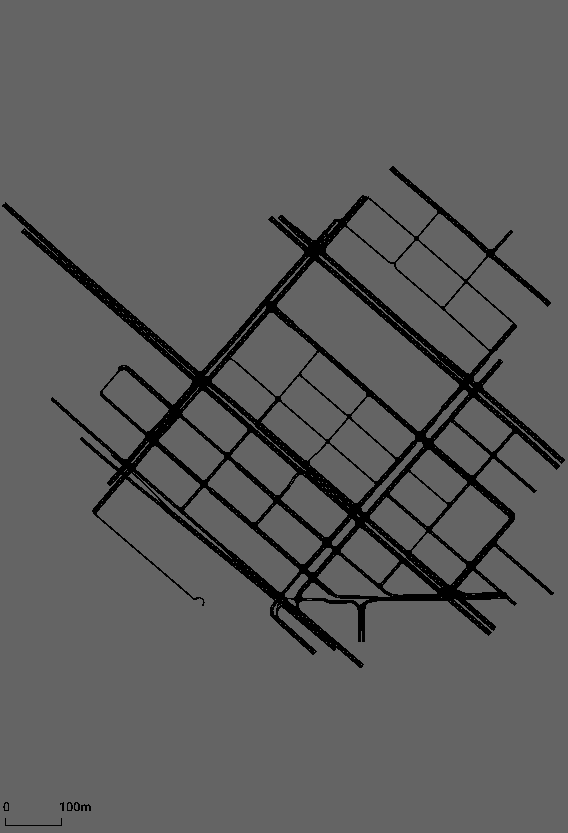}\label{fig:ssp}}
\hfil
\subfloat[]{\includegraphics[width=0.19\columnwidth]{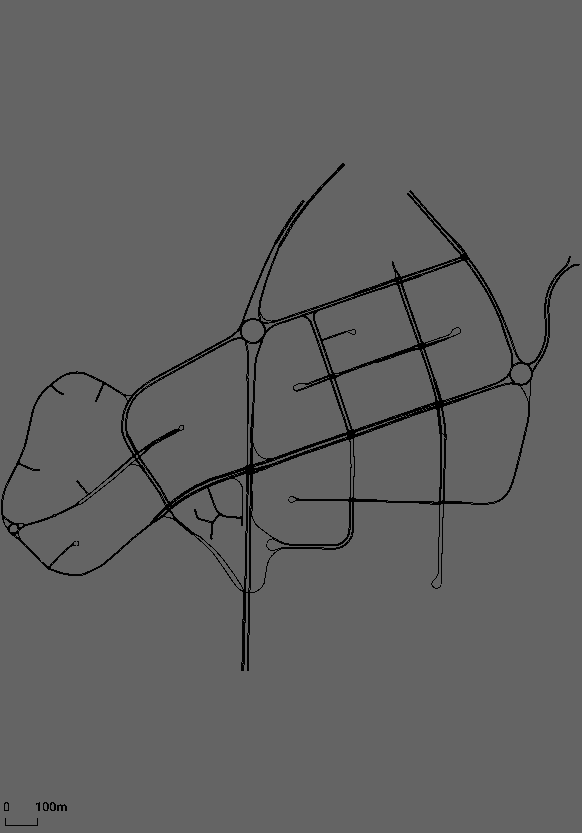}\label{fig:tko}}
\caption{OSM datasets of ten areas in Hong Kong. (a) Causeway; (b) Central; (c) Chai Wan; (d) Clear Water Bay; (e) Kennedy Town; (f) Kai Tak; (g) North Point; (h) Po Lam; (i) Sham Shui Po; (j) Tseung Kwan O.}
\label{fig:osms}
\end{figure}

\subsection{Experiment Settings}

We use neural networks to represent policies as in \cite{DBLP:journals/corr/SchulmanLMJA15,pmlr-v87-vinitsky18a}. Specifically, we use Multilayer Perceptrons (MLPs) with $\tanh$ non-linearity and hidden layers (64, 64). We use the SGD optimizer with a momentum of 0.9 and learning-rate decay of 0.98, 0.9, and 0.98 per round for Mountain Cars, Hoppers, and HongKongOSMs, respectively. Hyperparameters are carefully tuned so that they are near-optimal for each algorithm. See Appendix \ref{sec:additionexperimentdetails} for more experimental details.

We compare FedPOHCS with biased and unbiased client selection methods, including FedAvg (random selection), Power-of-Choice, and GradientNorm \citep{9292468,marnissi2021client,chen2022optimal}. FedAvg randomly selects K clients. Power-of-Choice utilizes the aforementioned two-phase scheme and selects candidate clients with the highest loss (or lowest advantages/values in case of RL problem). GradientNorm selects candidate clients with the largest gradient norm. Besides the client selection scheme, all algorithms follow the same procedure described at the beginning of Section \ref{subsec:FRL}. In particular, clients perform local training with the algorithm proposed by \cite{DBLP:journals/corr/SchulmanWDRK17}, with adaptive KL penalty term. At the end of every round, the server broadcasts the global policy to all clients, orders them to interact with their MDPs for several episodes (10 for Mountain Cars, 100 for Hoppers, and 20 for HongKongOSMs) and report the mean returns to evaluate the performance. Each experiment is averaged across three independent runs with different random seeds and parameter initializations, both of which are shared among all algorithms for a fair comparison. Confidence intervals are also reported.

\begin{figure}[ht!]
\centering
\includegraphics[width=0.6\columnwidth]{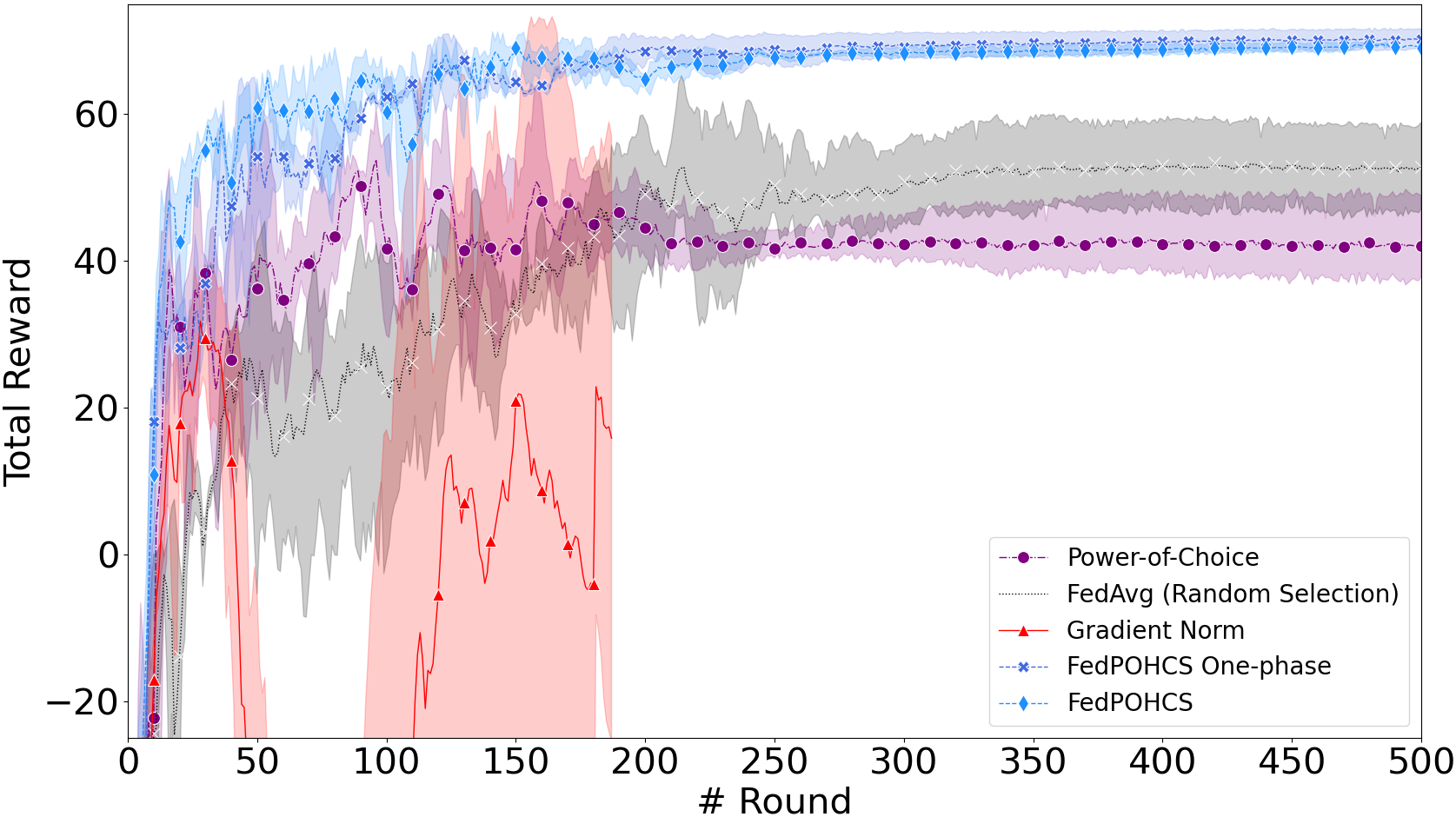}
\caption{Comparison of FedAvg, Power-of-Choice, FedPOHCS, and GradientNorm on Mountain Cars with a medium level of heterogeneity. For FedAvg, the learning rate is 0.005 and the KL target is 0.003. For Power-of-Choice, GradientNorm, and FedPOHCS, the learning rate is 1e-3, and the KL target is 0.003.}
\label{fig:mcc}
\end{figure}

\begin{figure}[ht!]
\centering
\includegraphics[width=0.6\linewidth]{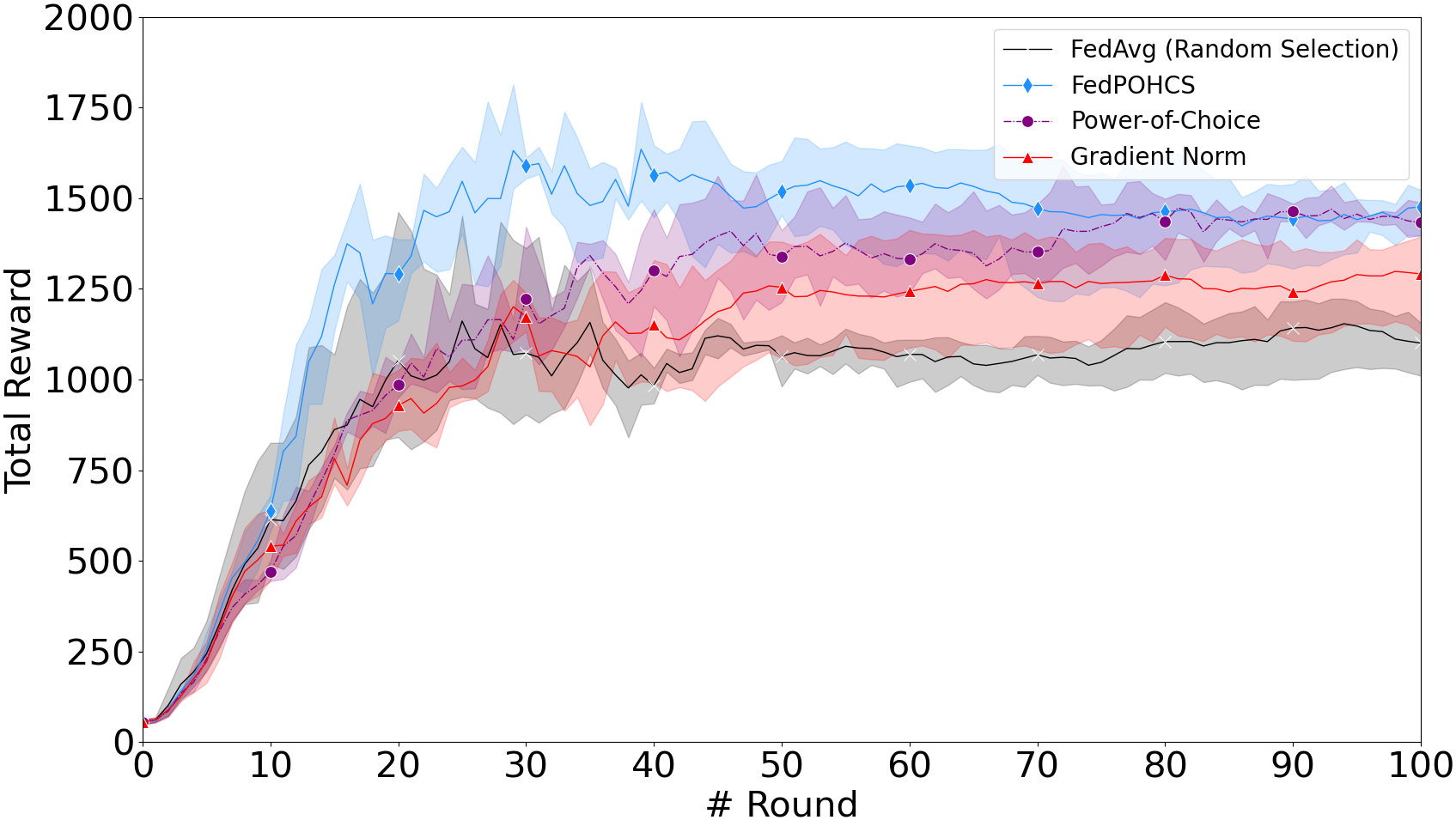}
\caption{Comparison of FedAvg, Power-of-Choice, FedPOHCS, and GradientNorm on Hoppers with a medium level of heterogeneity. For all algorithms, the learning rate is 0.03, the learning rate decay is 0.9, and the KL target is 0.003.}
\label{fig:hoppers}
\end{figure}

\begin{figure}[ht!]
\centering
\includegraphics[width=0.6\linewidth]{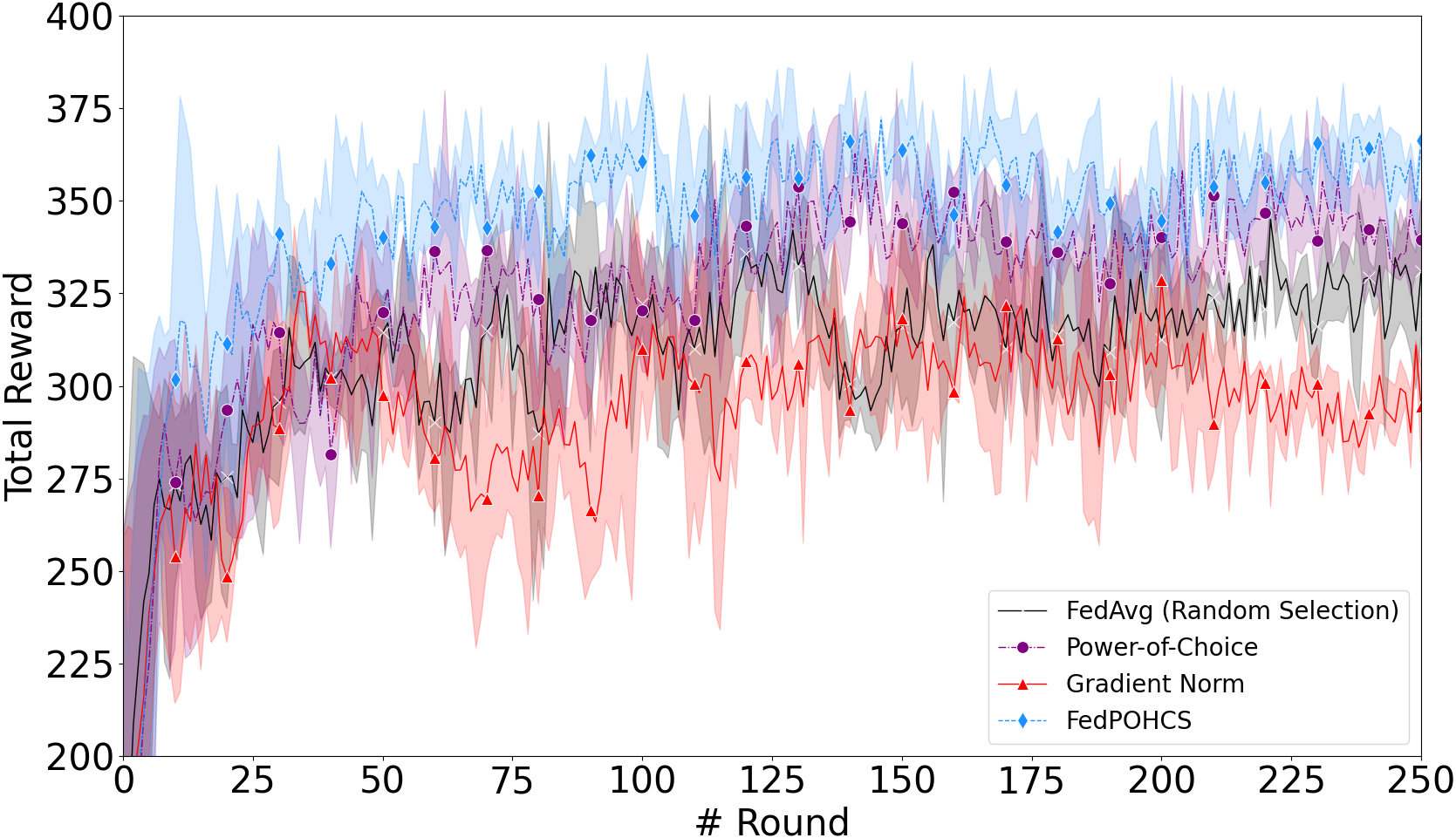}
\caption{Comparison of FedAvg, Power-of-Choice, FedPOHCS, and GradientNorm on HongKongOSMs. For all algorithms, the learning rate is 0.0001, the learning rate decay is 0.98, and the KL target is 0.0001.}
\label{fig:osm}
\end{figure}

\begin{figure}[ht!]
\centering
\includegraphics[width=0.6\columnwidth]{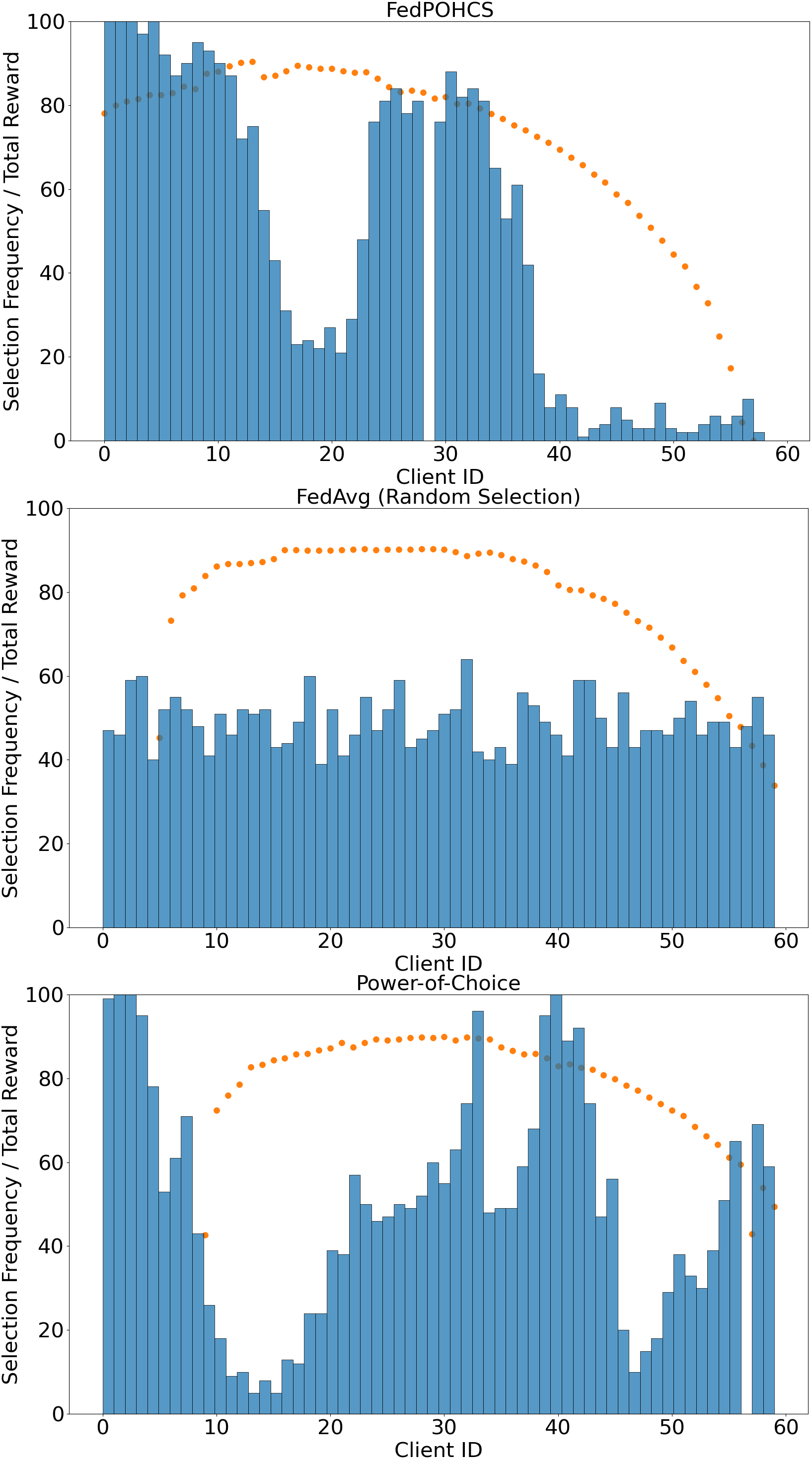}
\caption{The histogram represents the frequency of each client being selected. The scatter points denote the return obtained by the final policy from each client. (top) FedPOHCS, (middle) FedAvg, and (bottom) Powder-of-Choice.}
\label{fig:client_partitipant_ratio}
\end{figure}

\begin{figure}[ht!]
\centering
\subfloat[]{\includegraphics[width=0.6\columnwidth]{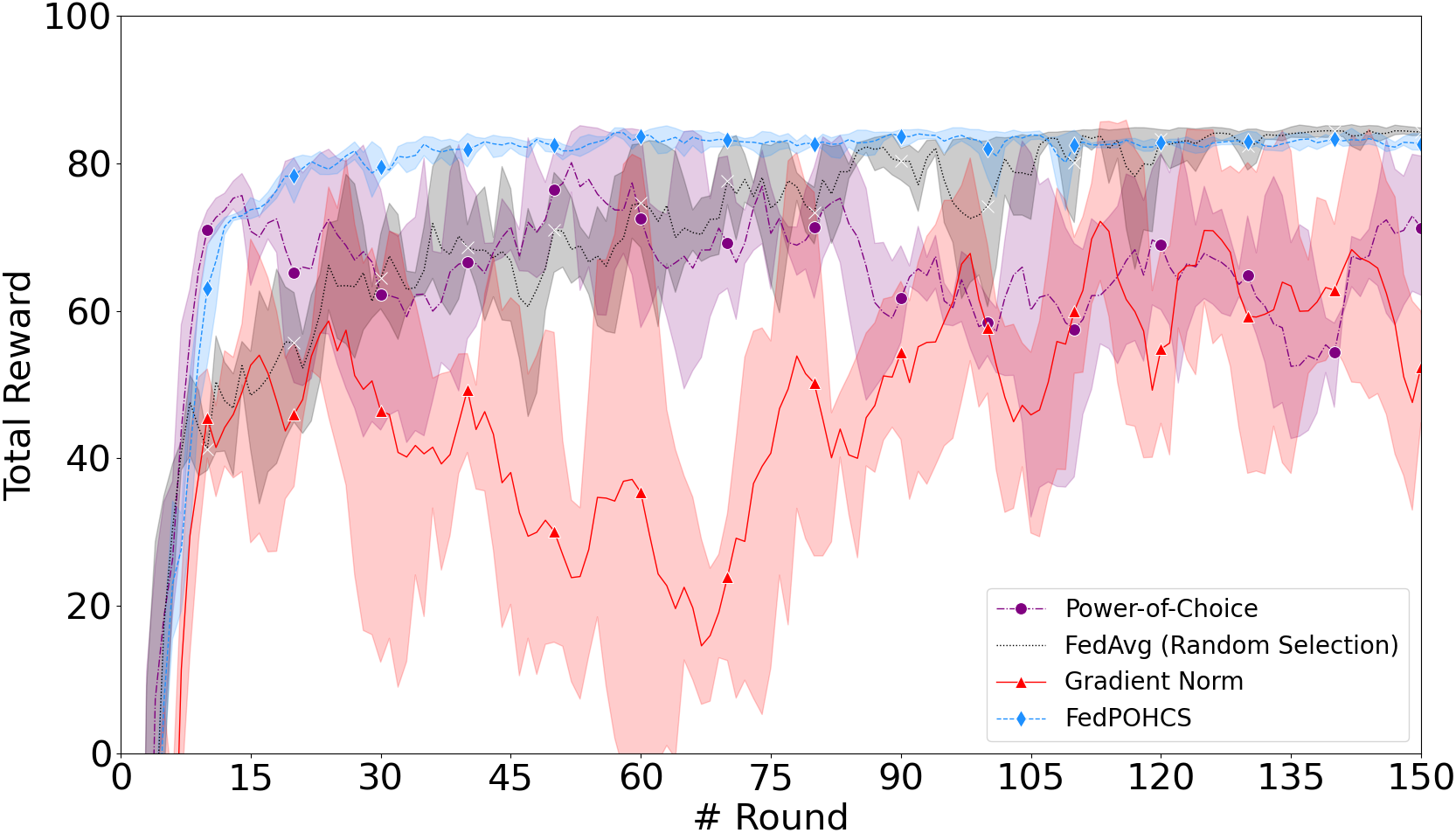}
\label{fig:mcc_low_level_hete}}
\hfil
\subfloat[]{\includegraphics[width=0.6\columnwidth]{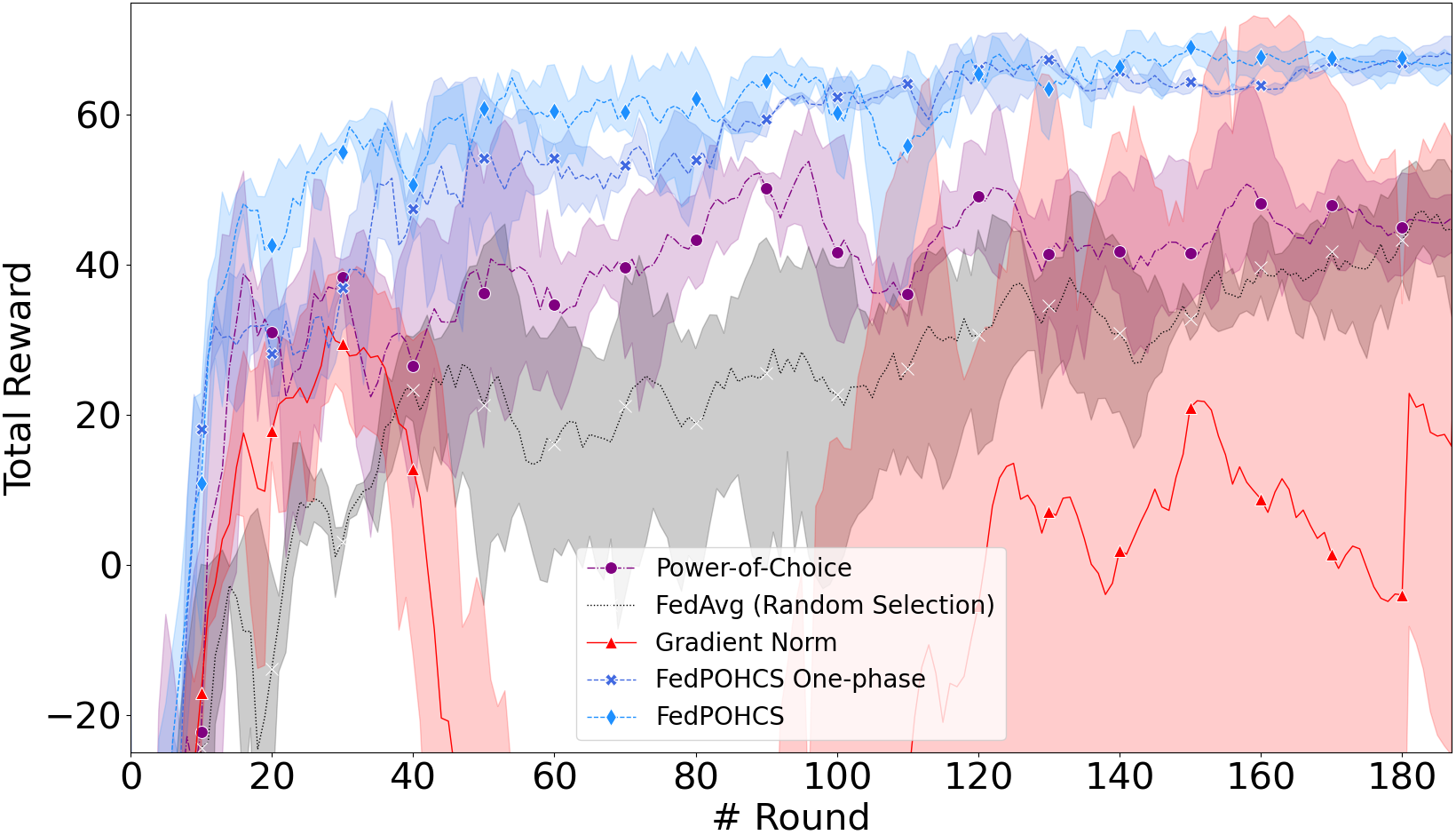}
\label{fig:mcc_medium_level_hete}}
\hfil
\subfloat[]{\includegraphics[width=0.6\columnwidth]{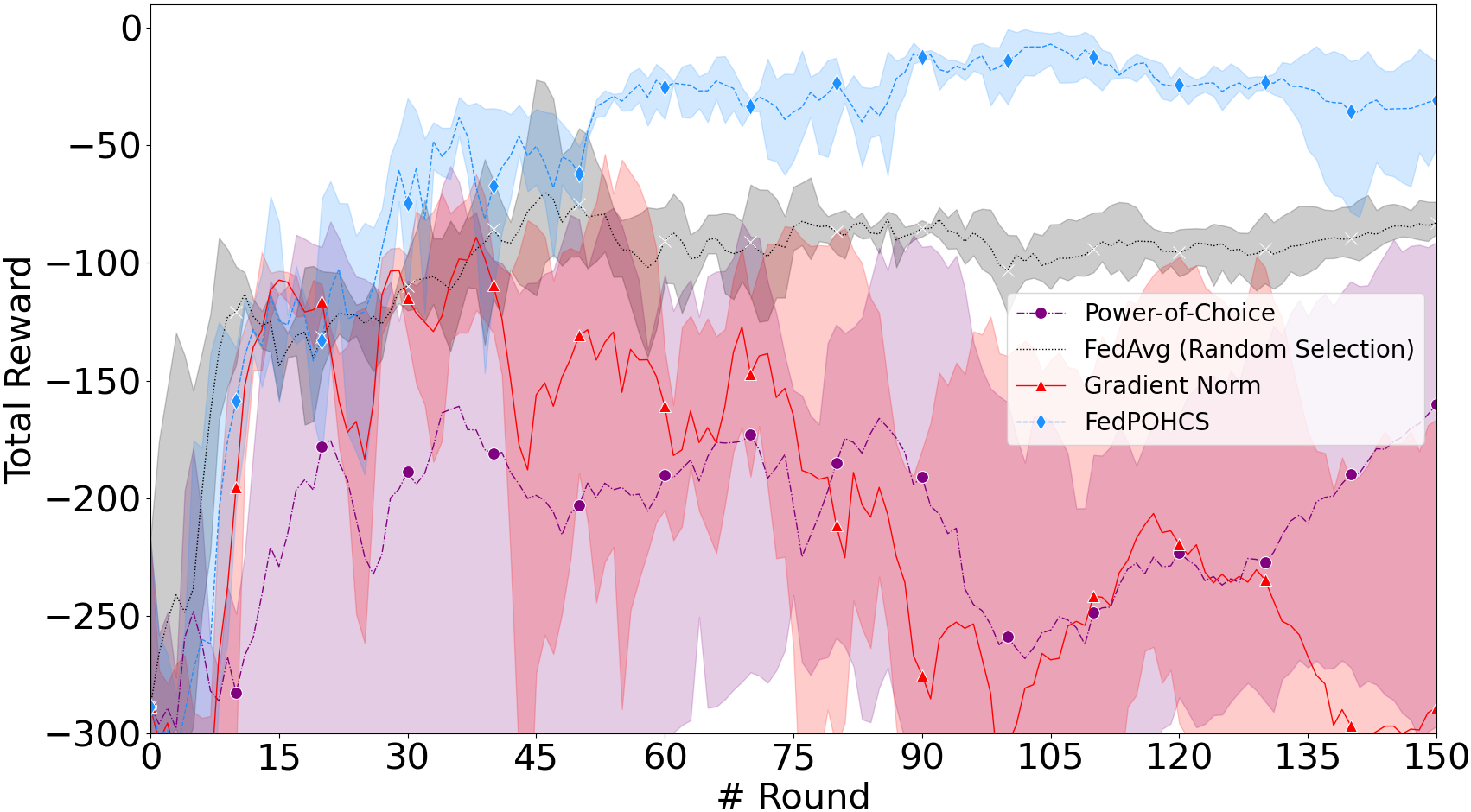}
\label{fig:mcc_high_level_hete}}
\caption{Comparison of FedAvg, Power-of-Choice, FedPOHCS and GradientNorm on Mountain Cars with low/medium/high level of heterogeneity.}
\label{fig:mcc_level_hete}
\end{figure}

\begin{figure}[ht!]
\centering
\subfloat[]{\includegraphics[width=0.6\columnwidth]{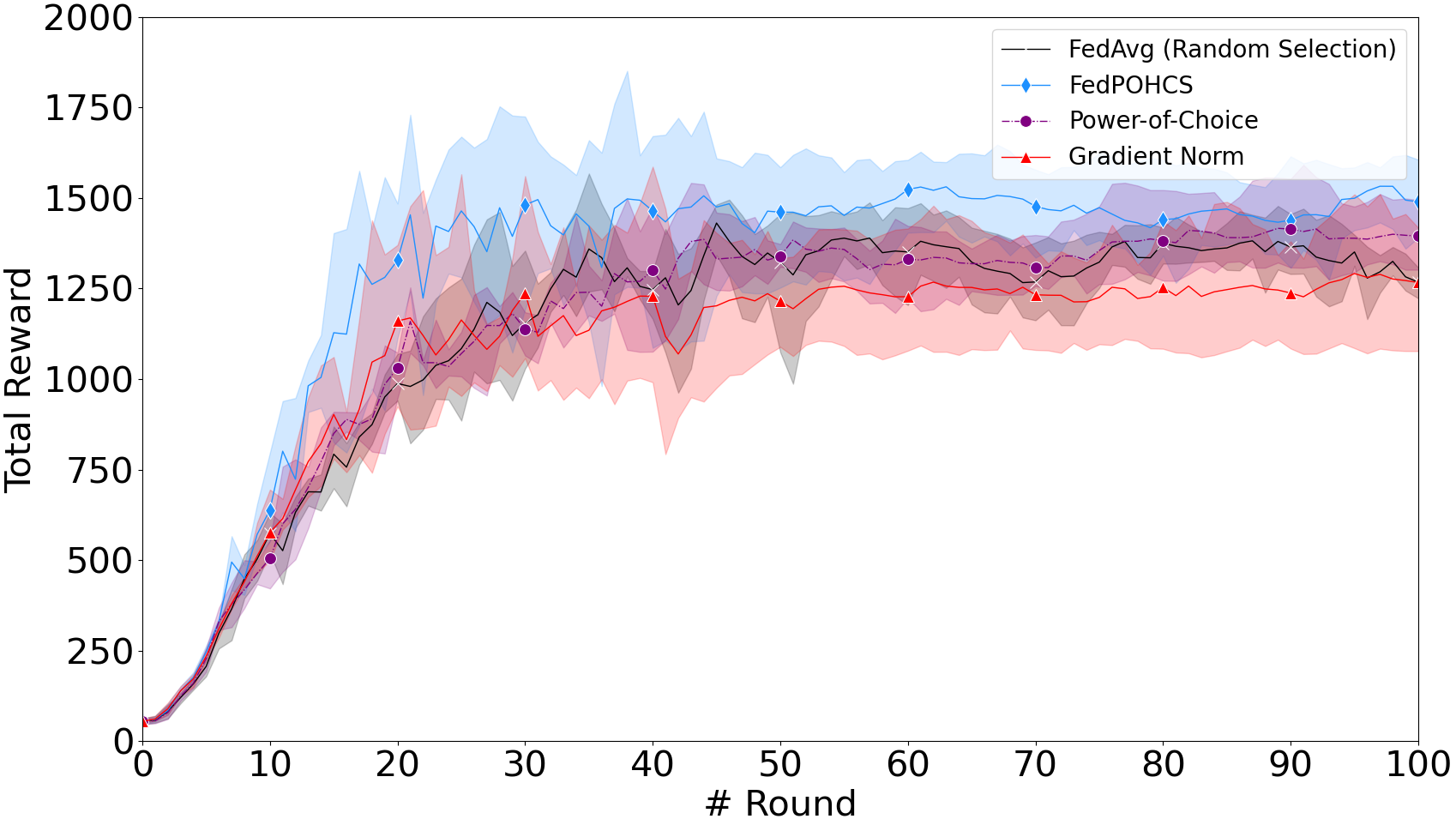}
\label{fig:hopper_low_level_hete}}
\hfil
\subfloat[]{\includegraphics[width=0.6\columnwidth]{graphics/hopper_medium_level_hete}
\label{fig:hopper_medium_level_hete}}
\hfil
\subfloat[]{\includegraphics[width=0.6\columnwidth]{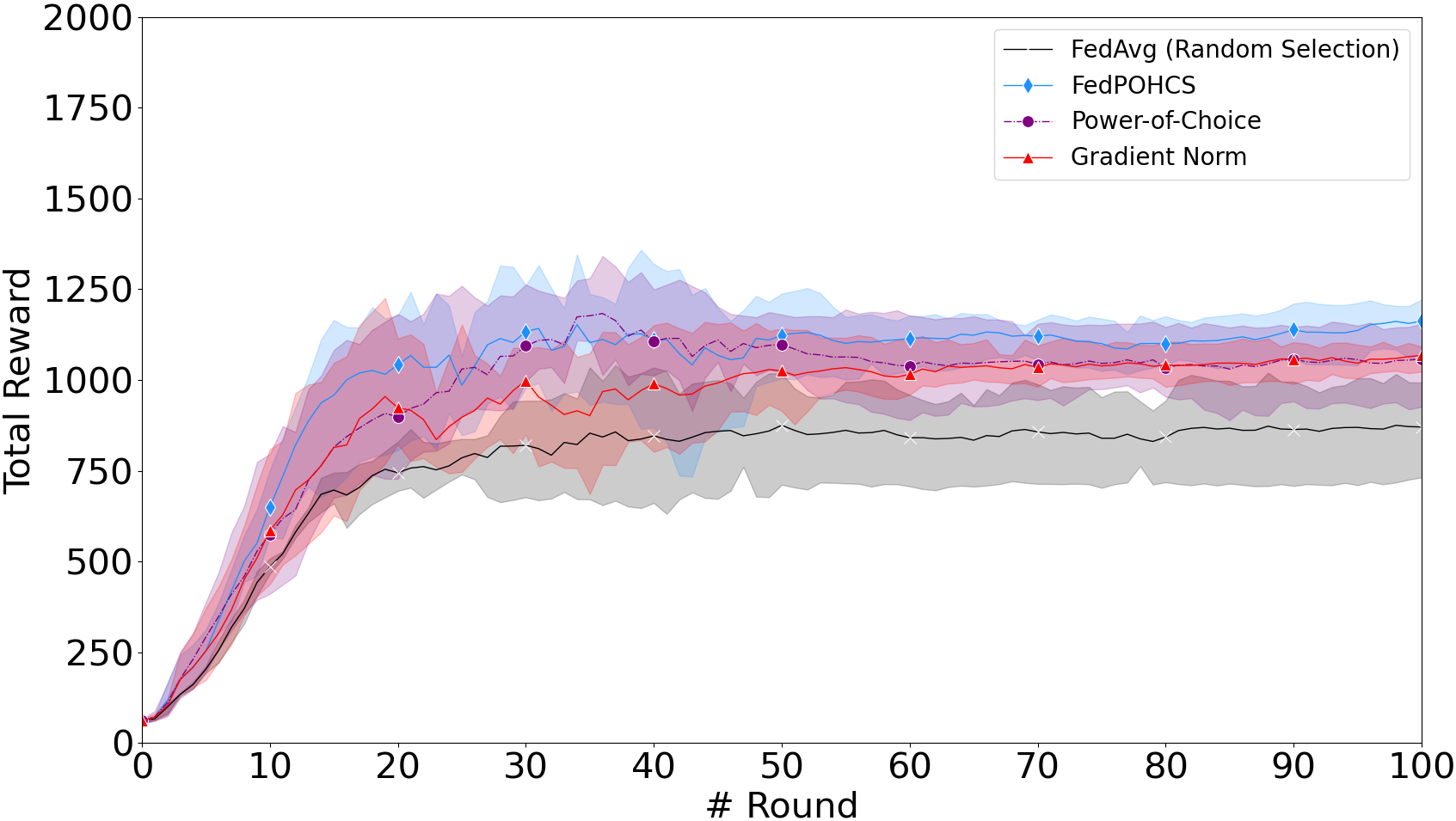}
\label{fig:hopper_high_level_hete}}
\caption{Comparison of FedAvg, Power-of-Choice, FedPOHCS and GradientNorm on Hoppers with low/medium/high level of heterogeneity.}
\label{fig:hopper_level_hete}
\end{figure}

\subsection{Performance and Stability}

Although the original mountain car problem is simple and most modern RL algorithms can easily obtain a score over 90.0, the federated setting imposes great difficulties in solving it. Figure \ref{fig:mcc} shows the performance comparison between FedPOHCS and several baselines on Mountain Cars with medium-level heterogeneity. It can be observed that FedPOHCS has a faster convergence speed and a more stable learning process. To compute the selection metric, we discretize the state and action by rounding them off to the nearest 0.1, resulting in about 8000 states and 100 actions that are frequently visited. In each round, the first phase of FedPOHCS takes 1-10 seconds, and the local training takes 20-30 seconds. Although the running time of the first phase may vary depending on the implementation, FedPOHCS outperforms all baselines in terms of the number of rounds and wall-clock time in our setting. In Figure \ref{fig:mcc}, we have also included the communication-efficient variant, i.e., the one-phase scheme, for comparison purposes. It can be observed that the one-phase scheme may slow down and destabilize learning in the early stage of training, but the final performance is comparable to the two-phase scheme.

We can draw a similar conclusion on Hoppers with medium-level heterogeneity. As shown in Figure \ref{fig:hoppers}, FedPOHCS can obtain an accumulated reward of 1450 within 20 rounds of training, while it takes about 80 rounds for Power-of-Choice to reach 1450, demonstrating the advantage of the proposed FedPOHCS algorithm in speeding up convergence.

HongKongOSMs is much more difficult to solve as shown by the high-variance curves in Figure \ref{fig:osm}. Since different OSM datasets have different numbers of lanes, target velocity, and maximum acceleration/deceleration, their state spaces and transition probabilities may be highly distinct from each other. Compared with other approaches, FedPOHCS can obtain higher rewards and converge to the highest point.

\subsection{Effectiveness of Metrics}

In Figure \ref{fig:client_partitipant_ratio}, we show how different algorithms select clients (the histogram) and the reward obtained by the final policy from each client (the scatter points) on Mountain Cars with medium-level heterogeneity. Note that with the constant shift $\theta_{n}=-1.5+\frac{n}{20}$, clients with small IDs are very different from those with large IDs. It can be observed that, compared with random selection and Power-of-Choice, FedPOHCS refuses to allocate resources to clients with IDs in $[40, 60]$, while Power-of-Choice spends a significant amount of resources on them. The final policy generated by FedPOHCS performs well on almost all clients and gets fairly high scores on clients with IDs in $[1, 10]$ without hurting other clients. This indicates that clients with IDs in $[1, 10]$ and clients with IDs in $[40, 60]$ are competing with each other. This is because they have constant shifts $\theta_{n}$ of different directions. Moreover, policies that perform well on clients with IDs in $[1, 10]$ may get fairly good scores on clients with IDs in $[40, 60]$, while the opposite is not true. We conjecture that the relation between the constant shifts $\theta_{n}$ and the transition probabilities is not linear and the imaginary MDP is closer to clients with small IDs, and hence learning on clients with IDs in $[40, 60]$ can be harmful to the overall performance.

\subsection{Different levels of heterogeneity}

In Figure \ref{fig:mcc_level_hete}, we show the performance comparison with different levels of heterogeneity, i.e., low (a), medium (b), and high (c), on Mountain Cars. It can be observed that, as the level of heterogeneity increases, the performance of all algorithms decreases, but the gap between FedPOHCS and the others increases. In other words, FedPOHCS demonstrates a bigger advantage when the level of heterogeneity is high, though its performance is also affected by severe heterogeneity. We have also conducted similar experiments on Hoppers as shown in Figure \ref{fig:hopper_level_hete}. While we can draw the same conclusion on FedAvg, i.e., its performance decreases as the level of heterogeneity increases, all three biased client selection methods are less affected by the level of heterogeneity.

\section{Conclusion\label{sec:Conclusion}}

In this work, we derived an error bound for federated policy optimization that explicitly unveils the impact of environment heterogeneity. The associated analysis covered various scenarios in FRL and offered insights into the effects of different federated settings. In particular, it was shown that clients whose environment dynamics are close to the population distribution are preferable for training. Based on these results, a client selection algorithm was proposed for FRL with heterogeneous clients. Experiment results demonstrated that the proposed client selection scheme outperforms other baselines on two federated RL problems. The results of this work represent a small step in understanding FRL and may motivate further research efforts in client selection for FRL.

\acks{This work was fully supported by a grant from the NSFC/RGC Joint Research Scheme sponsored by the Research Grants Council of the Hong Kong Special Administrative Region, China and National Natural Science Foundation of China (Project No. N\_HKUST656/22).}


\newpage

\setcounter{section}{0}
\renewcommand\thesection{\Alph{section}}
\renewcommand*{\theHsection}{chX.\the\value{section}}

\appendix
\section*{APPENDICES}

This appendix contains the proof of Lemma \ref{lemma:wfpepieb} in Appendix \ref{proof:lemma:wfpepieb}, the proof of Lemma \ref{lemma:wofpepieb} in Appendix \ref{proof:lemma:wofpepieb}, the proof of Proposition \ref{proposition:errorboundpartial} in Appendix \ref{proof:proposition:errorboundpartial}, the proof of Theorem \ref{theorem:errorboundfinal} in Appendix \ref{proof:theorem:errorboundfinal}, the proof of Proposition \ref{proposition:wfpeerrorboundpartial} in Appendix \ref{proof:proposition:wfpeerrorboundpartial}, the proof of Proposition \ref{proposition:homoerrorboundfinal} in Appendix \ref{proof:proposition:homoerrorboundfinal}, and additional experiment setting in Appendix \ref{sec:additionexperimentdetails}.

\section{Proof of Lemma \ref{lemma:wfpepieb}\label{proof:lemma:wfpepieb}}

The following Lemmas will be frequently used throughout the appendix.
\begin{lemma}
\label{lemma:normtmtn}
For any value function $V \in \mathbb{R}^{\left\vert \mathcal{S} \right\vert}$, policy $\pi$ and client $n,m$, we have 
\begin{alignat}{1}
\left\Vert T_{m} V - T_{n} V \right\Vert \le \frac{\gamma R_{\max} \kappa_{m,n}}{1 - \gamma}, \label{eq:lemma:normtmtn1} \\
\left\Vert T_{I} V - T_{n} V \right\Vert \le \frac{\gamma R_{\max} \kappa_{n,I}}{1 - \gamma}, \label{eq:lemma:normtmtn2} \\
\left\Vert T^{\pi}_{m} V - T^{\pi}_{n} V \right\Vert \le \frac{\gamma R_{\max} \kappa_{m,n}}{1 - \gamma}, \label{eq:lemma:normtmtn3} \\
\left\Vert T^{\pi}_{I} V - T^{\pi}_{n} V \right\Vert \le \frac{\gamma R_{\max} \kappa_{n,I}}{1 - \gamma}. \label{eq:lemma:normtmtn4}
\end{alignat}
\end{lemma}
\begin{proof}
For any $s \in \mathcal{S}$ and $m,n = 1,\dots,N$, let $a_{\_} = \arg\max_{a} \mathcal{R}(s,a) + \gamma \sum_{s^{\prime}} \bar{P}(s^{\prime} \vert s,a) V(s^{\prime})$, $a_{n} = \arg\max_{a} \mathcal{R}(s,a) + \gamma \sum_{s^{\prime}} P_{n}(s^{\prime} \vert s,a) V(s^{\prime})$, and $a_{m} = \arg\max_{a} \mathcal{R}(s,a) + \gamma \sum_{s^{\prime}} P_{m}(s^{\prime} \vert s,a) V(s^{\prime})$ denote the greedy actions taken by $T_{I}$, $T_{n}$ and $T_{m}$ on $V(s)$, respectively. Then, for any state $s$, we have
\begin{alignat}{1}
\left\vert T_{m} V(s) - T_{n} V(s) \right\vert & = \left\vert \mathcal{R}(s,a_{m}) - \mathcal{R}(s,a_{n}) + \gamma \sum_{s^{\prime}} \left( P_{m}(s^{\prime} \vert s,a_{m}) - P_{n}(s^{\prime} \vert s,a_{n}) \right) V(s^{\prime}) \right\vert. \nonumber
\end{alignat}
Without loss of generality, we assume that $T_{m} V(s) \ge T_{n} V(s)$. As a result, we can obtain the following inequality by replacing $a_{n}$ with $a_{m}$
\begin{alignat}{1}
\left\vert T_{m} V(s) - T_{n} V(s) \right\vert & \le \left\vert \gamma \sum_{s^{\prime}} \left( P_{m}(s^{\prime} \vert s,a_{m}) - P_{n}(s^{\prime} \vert s,a_{m}) \right) V(s^{\prime}) \right\vert \nonumber \\
& \le \gamma \sum_{s^{\prime}} \left\vert \left( P_{m}(s^{\prime} \vert s,a_{m}) - P_{n}(s^{\prime} \vert s,a_{m}) \right) \right\vert \left\vert V(s^{\prime}) \right\vert \nonumber \\
& \le \frac{\gamma R_{\max} \kappa_{m,n}}{1 - \gamma}, \nonumber
\end{alignat}
where the last inequality follows from the fact that all value functions are bounded by $\frac{R_{\max}}{1 - \gamma}$. This completes the proof of (\ref{eq:lemma:normtmtn1}) and the proof of (\ref{eq:lemma:normtmtn2}) is similar. Next, we prove (\ref{eq:lemma:normtmtn3}). By following a similar procedure, we can obtain 
\begin{alignat}{1}
\left\vert T^{\pi}_{m} V(s) - T^{\pi}_{n} V(s) \right\vert & = \left\vert \gamma \sum_{a,s^\prime} \pi(a \vert s) \left( P_{m}(s^\prime \vert s,a) - P_{n}(s^\prime \vert s,a) \right) V(s^\prime) \right\vert \nonumber \\
& \le \gamma  \sum_{s^\prime}  \left\vert \left( P_{m}^{\pi}(s^\prime \vert s) - P^{\pi}_{n}(s^\prime \vert s) \right) V(s^\prime) \right\vert \nonumber \\
& \le \gamma  \sum_{s^\prime}  \left\vert P_{m}^{\pi}(s^\prime \vert s) - P^{\pi}_{n}(s^\prime \vert s) \right\vert \left\vert V(s^\prime) \right\vert \nonumber \\
& \le \frac{\gamma R_{\max} \kappa_{m,n}}{1 - \gamma}. \nonumber
\end{alignat}
This completes the proof of (\ref{eq:lemma:normtmtn3}), and the proof of (\ref{eq:lemma:normtmtn4}) is similar.
\end{proof}

\begin{lemma}
\label{lemma:normtpitbarpi}
Let $\bar{\pi}^{t}(a \vert s) = \sum_{n=1}^{N} q_{n} \pi^{t}_{n} (a \vert s)$ denote the expected output of all local policies and $\tilde{\pi}^{t}(a \vert s) = \sum_{m\in\mathcal{C}} q^{\prime}_{m} \pi^{t}_{m}(a \vert s), \forall s \in \mathcal{S}, a \in \mathcal{A}$ represent the expected output of a set $\mathcal{C}$ of local policies. Define $\bar{\varepsilon}_{\theta} = \max_{t} \left\Vert \pi^{t}(\cdot \vert s) - \bar{\pi}^{t}_{n}(\cdot \vert s) \right\Vert_{2}$ and $\tilde{\varepsilon}_{\theta} = \max_{t} \left\Vert \pi^{t}(\cdot \vert s) - \tilde{\pi}^{t}_{m}(\cdot \vert s) \right\Vert_{2}$ for full participation and partial participation, respectively. For any value function $V \in \mathbb{R}^{\left\vert \mathcal{S} \right\vert}$ and policy $\pi^{t}$ at round $t$, we have
\begin{alignat}{1}
\left\Vert T_{I}^{\pi^{t}} V - T^{\bar{\pi}^{t}}_{I} V \right\Vert &\le \frac{\bar{\varepsilon}_{\theta} \sqrt{\left\vert \mathcal{A} \right\vert} R_{\max}}{1 - \gamma}, \label{eq:lemma:normtpitbarpi1} \\
\left\Vert T_{n}^{\pi^{t}} V - T^{\bar{\pi}^{t}}_{n} V \right\Vert &\le \frac{\bar{\varepsilon}_{\theta} \sqrt{\left\vert \mathcal{A} \right\vert} R_{\max}}{1 - \gamma}, \label{eq:lemma:normtpitbarpi2} \\
\left\Vert T_{I}^{\pi^{t}} V - T^{\tilde{\pi}^{t}}_{I} V \right\Vert &\le \frac{\tilde{\varepsilon}_{\theta} \sqrt{\left\vert \mathcal{A} \right\vert} R_{\max}}{1 - \gamma}, \label{eq:lemma:normtpitbarpi3} \\
\left\Vert T_{n}^{\pi^{t}} V - T^{\tilde{\pi}^{t}}_{n} V \right\Vert &\le \frac{\tilde{\varepsilon}_{\theta} \sqrt{\left\vert \mathcal{A} \right\vert} R_{\max}}{1 - \gamma}. \label{eq:lemma:normtpitbarpi4}
\end{alignat}
\end{lemma}
\begin{proof}
For any state $s$, we have
\begin{alignat}{1}
\left\vert T_{I}^{\pi^{t}} V^{t}(s) - T^{\bar{\pi}^{t}}_{I} V^{t}(s) \right\vert & = \left\vert \sum_{a} \sum_{n=1}^{N} q_{n} \left( \pi^{t}(a \vert s) - \pi^{t}_{n}(a \vert s) \right) \left( \mathcal{R}(s,a) + \gamma \sum_{s^\prime} \bar{P}(s^\prime \vert s,a) V^{t}(s^\prime)  \right) \right\vert \nonumber \\
& \le \left\Vert \pi^{t}(\cdot \vert s) - \sum_{n=1}^{N} q_{n} \pi^{t}_{n}(\cdot \vert s) \right\Vert_{2} \left\Vert \mathcal{R}(s,\cdot) + \gamma \sum_{s^\prime} \bar{P}(s^\prime \vert s, \cdot) V^{t}(s^\prime) \right\Vert_{2} \nonumber \\
& \le \frac{\sqrt{\left\vert \mathcal{A} \right\vert} R_{\max}}{1 - \gamma} \left\Vert \pi^{t}(\cdot \vert s) - \sum_{n=1}^{N} q_{n} \pi^{t}_{n}(\cdot \vert s) \right\Vert_{2} \nonumber \\
& \le \frac{\bar{\varepsilon}_{\theta} \sqrt{\left\vert \mathcal{A} \right\vert} R_{\max}}{1 - \gamma}, \nonumber
\end{alignat}
which completes the proof of (\ref{eq:lemma:normtpitbarpi1}) and the proofs of (\ref{eq:lemma:normtpitbarpi2}), (\ref{eq:lemma:normtpitbarpi3}), and (\ref{eq:lemma:normtpitbarpi4}) are similar.
\end{proof}

Next, we prove Lemma \ref{lemma:wfpepieb}.
\begin{proof}
Let $\bar{\pi}^{t}(a \vert s) = \sum_{n=1}^{N} q_{n} \pi^{t}_{n} (a \vert s)$ denote the expected output of all local policies. For any state $s$, we have
\begin{alignat}{1}
\left\vert T_{I}^{\pi^{t+1}} V^{t}(s) - T_{I} V^{t}(s) \right\vert &\le \left\vert T_{I}^{\pi^{t+1}} V^{t}(s) - T^{\bar{\pi}^{t+1}}_{I} V^{t}(s) \right\vert + \left\vert T^{\bar{\pi}^{t+1}}_{I} V^{t}(s) - T_{I} V^{t}(s) \right\vert. \label{eq:proof:lemma:wfpepieb0}
\end{alignat}
By Lemma \ref{lemma:normtpitbarpi}, the first term on the right-hand side (RHS) of (\ref{eq:proof:lemma:wfpepieb0}) is upper bounded by
\begin{alignat}{1}
\left\vert T_{I}^{\pi^{t+1}} V^{t}(s) - T^{\bar{\pi}^{t+1}}_{I} V^{t}(s) \right\vert &\le \frac{\bar{\varepsilon}_{\theta} \sqrt{\left\vert \mathcal{A} \right\vert} R_{\max}}{1 - \gamma}. \label{eq:proof:lemma:wfpepieb6}
\end{alignat}

Now we bound the second term on the RHS of (\ref{eq:proof:lemma:wfpepieb0}). For any state $s$, we have
\begin{alignat}{1}
& \left\vert T^{\bar{\pi}^{t+1}}_{I} V^{t}(s) - T_{I} V^{t}(s) \right\vert \nonumber \\
& = \left\vert \sum_{a} \bar{\pi}^{t+1}(a \vert s) \left( \mathcal{R}(s,a) + \gamma \sum_{s^\prime} \bar{P}(s^\prime \vert s,a) V^{t}(s^\prime)  \right) - T_{I} V^{t} \right\vert \nonumber \\
& \stackrel{(a)}{=} \left\vert \sum_{n=1}^{N} q_{n} \left[ \sum_{a} \pi^{t+1}_{n}(a \vert s) \left( \mathcal{R}(s,a) + \gamma \sum_{s^\prime} \bar{P}(s^\prime \vert s,a) V^{t}(s^\prime)  \right) - T_{I} V^{t} \right] \right\vert \nonumber \\
& \stackrel{(b)}{=} \left\vert \sum_{n=1}^{N} q_{n} \left[ \gamma \sum_{a,s^\prime} \pi^{t+1}_{n}(a \vert s) \left( \bar{P}(s^\prime \vert s,a) - P_{n}(s^\prime \vert s,a) \right) V^{t}(s^\prime) + T^{\pi^{t+1}_{n}}_{n} V^{t}(s) - T_{I} V^{t} \right] \right\vert \nonumber \\
& \stackrel{(c)}{\le} \left\vert \sum_{n=1}^{N} q_{n} \gamma \sum_{a,s^\prime} \pi^{t+1}_{n}(a \vert s) \left( \bar{P}(s^\prime \vert s,a) - P_{n}(s^\prime \vert s,a) \right) V^{t}(s^\prime) \right\vert \nonumber \\
& \quad + \left\vert \sum_{n=1}^{N} q_{n} \left( T^{\pi^{t+1}_{n}}_{n} V^{t}(s) - T_{n} V^{t}(s) \right) \right\vert + \left\vert \sum_{n=1}^{N} q_{n} \left( T_{n} V^{t}(s) - T_{I} V^{t}(s) \right) \right\vert. \label{eq:proof:lemma:wfpepieb1}
\end{alignat}
Step (a) follows from $\bar{\pi}^{t+1}(a \vert s) = \sum_{n=1}^{N} q_{n} \pi^{t+1}_{n}(a \vert s),\forall s \in \mathcal{S}, a \in \mathcal{A}$. In step (b), we added and then subtracted the term $\sum_{a,s^\prime} \pi^{t+1}_{n}(a \vert s) P_{n}(s^\prime \vert s,a) V^{t}(s^\prime)$. The added term is combined with $\sum_{a,s^\prime} \pi^{t+1}_{n}(a \vert s) \mathcal{R}(s,a)$ to form $T^{\pi^{t+1}_{n}}_{n} V^{t}(s)$. Step (c) follows from the triangle inequality.

By Lemma \ref{lemma:normtmtn}, the first term on the RHS of (\ref{eq:proof:lemma:wfpepieb1}) is upper bounded by
\begin{alignat}{1}
& \left\vert \sum_{n=1}^{N} q_{n} \gamma \sum_{a,s^\prime} \pi^{t+1}_{n}(a \vert s) \left( \bar{P}(s^\prime \vert s,a) - P_{n}(s^\prime \vert s,a) \right) V^{t}(s^\prime) \right\vert \nonumber \\
& = \left\vert \sum_{n=1}^{N} q_{n} \left( T^{\pi^{t+1}_{n}}_{I} V^{t}(s) - T^{\pi^{t+1}_{n}}_{n} V^{t}(s) \right) \right\vert \nonumber \\
& \le \frac{\gamma R_{\max} \kappa_{1}}{1 - \gamma}. \label{eq:proof:lemma:wfpepieb2}
\end{alignat}
By (\ref{eq:deltan}), the second term on the RHS of (\ref{eq:proof:lemma:wfpepieb1}) is upper bounded by $\bar{\epsilon}$. By Lemma \ref{lemma:normtmtn}, the third term on the RHS of (\ref{eq:proof:lemma:wfpepieb1}) is upper bounded by
\begin{alignat}{1}
\left\vert \sum_{n=1}^{N} q_{n} \left( T_{n}V^{t}(s) - T_{I} V^{t}(s) \right) \right\vert & \le \frac{\gamma R_{\max} \kappa_{1}}{1 - \gamma}. \label{eq:proof:lemma:wfpepieb4}
\end{alignat}
By substituting the above-mentioned three upper bounds into (\ref{eq:proof:lemma:wfpepieb1}), we can obtain
\begin{alignat}{1}
\left\vert T^{\pi^{t+1}}_{I} V^{t}(s) - T_{I} V^{t}(s) \right\vert
& \le \frac{2 \gamma R_{\max} \kappa_{1}}{1 - \gamma} + \bar{\epsilon}. \label{eq:proof:lemma:wfpepieb5}
\end{alignat}
By substituting (\ref{eq:proof:lemma:wfpepieb5}) and (\ref{eq:proof:lemma:wfpepieb6}) into (\ref{eq:proof:lemma:wfpepieb0}), we can obtain
\begin{alignat}{1}
\left\vert T_{I}^{\pi^{t+1}} V^{t}(s) - T_{I} V^{t}(s) \right\vert \le \frac{\bar{\varepsilon}_{\theta} \sqrt{\left\vert \mathcal{A} \right\vert} R_{\max}}{1 - \gamma} + \frac{2 \gamma R_{\max} \kappa_{1}}{1 - \gamma} + \bar{\epsilon}.
\end{alignat}
\end{proof}

\section{Proof of Lemma \ref{lemma:wofpepieb}\label{proof:lemma:wofpepieb}}

To prove Lemma \ref{lemma:wofpepieb}, we first introduce Lemma \ref{lemma:StrehlLemma1}. 
\begin{lemma}
\label{lemma:StrehlLemma1}
For any state $s$, policy $\pi$ and clients $m,n$, we have 
\begin{alignat}{1}
\left\vert V^{\pi}_{m}(s) - V^{\pi}_{n}(s) \right\vert \le \frac{\gamma R_{\max} \kappa_{m,n}}{(1 - \gamma)^{2}}. \nonumber
\end{alignat}
\end{lemma}
\begin{proof}
For any states $s$, the distance between $V^{\pi}_{m}(s)$ and $V^{\pi}_{n}(s)$ can be bounded as 
\begin{alignat}{1}
& \left\vert V^{\pi}_{m}(s) - V^{\pi}_{n}(s) \right\vert \nonumber \\
& \stackrel{(a)}{=} \left\vert \sum_{a} \pi(a \vert s) \left( \mathcal{R}(s,a) + \gamma \sum_{s^{\prime} \in \mathcal{S}} P_{m}(s^{\prime} \vert s,a) V^{\pi}_{m}(s^{\prime}) - \mathcal{R}(s,a) - \gamma \sum_{s^{\prime} \in \mathcal{S}} P_{n}(s^{\prime} \vert s,a) V^{\pi}_{n}(s^{\prime}) \right) \right\vert \nonumber \\
& \stackrel{(b)}{=} \left\vert \gamma \sum_{s^{\prime} \in \mathcal{S}} \left( P^{\pi}_{m}(s^{\prime} \vert s) V^{\pi}_{m}(s^{\prime}) - P^{\pi}_{m}(s^{\prime} \vert s) V^{\pi}_{n}(s^{\prime}) + P^{\pi}_{m}(s^{\prime} \vert s) V^{\pi}_{n}(s^{\prime}) - P^{\pi}_{n}(s^{\prime} \vert s) V^{\pi}_{n}(s^{\prime}) \right) \right\vert \nonumber \\
& \stackrel{(c)}{\le} \gamma \sum_{s^{\prime} \in \mathcal{S}} P^{\pi}_{m}(s^{\prime} \vert s) \left\vert V^{\pi}_{m}(s^{\prime}) - V^{\pi}_{n}(s^{\prime}) \right\vert + \gamma \sum_{s^{\prime} \in \mathcal{S}} \left\vert P^{\pi}_{m}(s^{\prime} \vert s) - P^{\pi}_{n}(s^{\prime} \vert s) \right\vert \left\vert V^{\pi}_{n}(s^{\prime}) \right\vert \nonumber \\
& \stackrel{(d)}{\le} \gamma \max_{s^{\prime}} \left\vert V^{\pi}_{m}(s^{\prime}) - V^{\pi}_{n}(s^{\prime}) \right\vert + \frac{\gamma R_{\max} \sum_{s^{\prime} \in \mathcal{S}} \left\vert P^{\pi}_{m}(s^{\prime} \vert s) - P^{\pi}_{n}(s^{\prime} \vert s) \right\vert}{1 - \gamma}.
\label{51}
\end{alignat}
Step (a) follows from Bellman's equation. In step (b), we added and then subtracted the term $P^{\pi}_{m}(s^{\prime} \vert s) V^{\pi}_{n}(s^{\prime})$. Step (c) follows from the triangle inequality. Step (d) follows from the fact that all value functions are bounded by $\frac{R_{\max}}{1 - \gamma}$. By taking the maximum of both sides of (\ref{51}) over state $s$ and after some mathematical manipulations, we can finally obtain
\begin{alignat}{1}
\left\vert V^{\pi}_{m}(s) - V^{\pi}_{n}(s) \right\vert \le \frac{\gamma R_{\max} \kappa_{m,n}}{(1 - \gamma)^{2}}. \nonumber
\end{alignat}
\end{proof}
Note that we proved Lemma \ref{lemma:StrehlLemma1} for the state-value function, and a similar result for the action-value function was given by Lemma 1 in \citep{STREHL20081309}.

Next, we prove Lemma \ref{lemma:wofpepieb}. 

\begin{proof}
By the triangle inequality, we have
\begin{alignat}{1}
\left\Vert T^{\pi^{t+1}}_{I} V^{t} - T_{I} V^{t} \right\Vert &\le \left\Vert T^{\pi^{t+1}}_{I} V^{t} - T^{\bar{\pi}^{t+1}}_{I} V^{t} \right\Vert + \left\Vert T^{\bar{\pi}^{t+1}}_{I} V^{t} - T^{\bar{\pi}^{t+1}}_{I} \bar{V}^{t} \right\Vert \nonumber \\
& \quad + \left\Vert T^{\bar{\pi}^{t+1}}_{I} \bar{V}^{t} - T_{I} \bar{V}^{t} \right\Vert + \left\Vert T_{I} \bar{V}^{t} - T_{I} V^{t} \right\Vert, \nonumber
\end{alignat}
from which we can obtain
\begin{alignat}{1}
\left\Vert T^{\pi^{t+1}}_{I} V^{t} - T_{I} V^{t} \right\Vert &\le \frac{\bar{\varepsilon}_{\theta} \sqrt{\left\vert \mathcal{A} \right\vert} R_{\max}}{1 - \gamma} + 2 \gamma \bar{\varepsilon}_{w} + \left\Vert T^{\bar{\pi}^{t+1}}_{I} \bar{V}^{t} - T_{I} \bar{V}^{t} \right\Vert, \label{eq:wofpepieb0}
\end{alignat}
by Lemma \ref{lemma:normtpitbarpi}, the contraction property of the Bellman operators, and the definition of $\bar{\varepsilon}_{w}$.

To finish the proof, it suffices to bound the third term on the RHS of (\ref{eq:wofpepieb0}). By the definition of the Bellman operators, we have
\begin{alignat}{1}
\left\Vert T^{\bar{\pi}^{t+1}}_{I} \bar{V}^{t} - T_{I} \bar{V}^{t} \right\Vert & = \left\Vert \sum_{n=1}^{N} q_{n} \left( T^{\pi^{t+1}_{n}}_{I} \bar{V}^{t} - T_{n}\bar{V}^{t} \right) \right\Vert \nonumber \\
& \le \sum_{n=1}^{N} q_{n} \left\Vert T^{\pi^{t+1}_{n}}_{I} \bar{V}^{t} - T_{n}\bar{V}^{t} \right\Vert. \label{eq:wofpepieb1}
\end{alignat}
Next, we further bound the RHS of (\ref{eq:wofpepieb1}). By the triangle inequality, we have
\begin{eqnarray}
 \left\Vert T^{\pi^{t+1}_{n}}_{I} \bar{V}^{t} - T_{n}\bar{V}^{t} \right\Vert 
 \le \left\vert T^{\pi^{t+1}_{n}}_{I} \bar{V}^{t} - T^{\pi^{t+1}_{n}}_{I} V^{t}_{n} \right\Vert + \left\Vert T^{\pi^{t+1}_{n}}_{I} V^{t}_{n} - T_{n}{V}^{t}_{n} \right\Vert + \left\Vert T_{n}{V}^{t}_{n} - T_{n}\bar{V}^{t} \right\Vert, \nonumber
\end{eqnarray}
from which we can obtain
\begin{eqnarray}
 \left\Vert T^{\pi^{t+1}_{n}}_{I} \bar{V}^{t} - T_{n}\bar{V}^{t} \right\Vert \le 2 \gamma \left\Vert \bar{V}^{t} - V^{t}_{n} \right\Vert + \left\Vert T^{\pi^{t+1}_{n}}_{I} V^{t}_{n} - T_{n}{V}^{t}_{n} \right\Vert \label{eq:wofpepieb2}
\end{eqnarray}
by the contraction property of the Bellman operators.

Next, we bound the first term on the RHS of (\ref{eq:wofpepieb2}). For any state $s$, we have 
\begin{alignat}{1}
\left\vert \bar{V}^{t}(s) - V_{n}^{t}(s) \right\vert & = \left\vert \sum_{j=1}^{N} q_{j} \left( V^{t}_{n}(s) - V^{t}_{j}(s) \right) \right\vert \nonumber \\
& \le \sum_{j=1}^{N} q_{j} \left\vert V^{\pi^{t}}_{n}(s) - V^{\pi^{t}}_{j}(s) \right\vert + \left\vert V^{\pi^{t}}_{n}(s) - V^{t}_{n}(s) \right\vert \nonumber  \\
& \quad + \sum_{j=1}^{N} q_{j} \left\vert V^{\pi^{t}}_{j}(s) - V^{t}_{j}(s) \right\vert \nonumber
\end{alignat}
due to the triangle inequality. By Lemma \ref{lemma:StrehlLemma1} and (\ref{eq:deltan}), 
we can further obtain
\begin{eqnarray}
\left\vert \bar{V}^{t}(s) - V_{n}^{t}(s) \right\vert \le \sum_{j=1}^{N} q_{j} \frac{\gamma R_{\max} \kappa_{n,j}}{(1 - \gamma)^{2}} + \bar{\delta} + \delta_{n}. \nonumber
\end{eqnarray}
Thus, we have the following bound
\begin{alignat}{1}
\left\Vert \bar{V}^{t} - V_{n}^{t} \right\Vert \le \sum_{j=1}^{N} q_{j} \frac{\gamma R_{\max} \kappa_{n,j}}{(1 - \gamma)^{2}} + \bar{\delta} + \delta_{n}. \label{eq:wofpepieb3}
\end{alignat}

Now we bound the second term on the RHS of (\ref{eq:wofpepieb2}). For any $s \in \mathcal{S}$, we have
\begin{alignat}{1}
\left\vert T^{\pi^{t+1}_{n}}_{I} V^{t}_{n}(s) - T_{n}{V}^{t}_{n}(s) \right\vert & \le \left\vert T^{\pi^{t+1}_{n}}_{I} V^{t}_{n}(s) - T^{\pi^{t+1}_{n}}_{n}{V}^{t}_{n}(s) \right\vert + \left\vert T^{\pi^{t+1}_{n}}_{n}{V}^{t}_{n}(s) - T_{n}{V}^{t}_{n}(s) \right\vert \nonumber \\
& \le \frac{\gamma R_{\max} \kappa_{n,I}}{1 - \gamma} + \epsilon_{n}, \nonumber
\end{alignat}
where the last inequality follows from Lemma \ref{lemma:normtmtn} and (\ref{eq:deltan}). Thus, we can obtain
\begin{alignat}{1}
\left\Vert T^{\pi^{t+1}_{n}}_{I} V^{t}_{n} - T_{n}{V}^{t}_{n} \right\Vert \le \frac{\gamma R_{\max} \kappa_{n,I}}{1 - \gamma} + \epsilon_{n}. \label{eq:wofpepieb4}
\end{alignat}
By substituting (\ref{eq:wofpepieb3}) and (\ref{eq:wofpepieb4}) into (\ref{eq:wofpepieb2}), and then substituting (\ref{eq:wofpepieb2}) into (\ref{eq:wofpepieb1}),  we have 
\begin{alignat}{1}
\left\Vert T^{\pi^{t+1}}_{I} \bar{V}^{t} - T_{I} \bar{V}^{t} \right\Vert & \le \frac{2 \gamma^{2} R_{\max} \kappa_{2}}{(1 - \gamma)^{2}} + \frac{\gamma R_{\max} \kappa_{1}}{1 - \gamma} + 4 \gamma \bar{\delta} + \bar{\epsilon}. \label{eq:wofpepieb5}
\end{alignat}
By substituting (\ref{eq:wofpepieb5}) into (\ref{eq:wofpepieb0}), we can obtain
\begin{alignat}{1}
\left\Vert T^{\pi^{t+1}}_{I} V^{t} - T_{I} V^{t} \right\Vert &\le \frac{\bar{\varepsilon}_{\theta} \sqrt{\left\vert \mathcal{A} \right\vert} R_{\max}}{1 - \gamma} + 2 \gamma \bar{\varepsilon}_{w} + \frac{2 \gamma^{2} R_{\max} \kappa_{2}}{(1 - \gamma)^{2}} + \frac{\gamma R_{\max} \kappa_{1}}{1 - \gamma} + 4 \gamma \bar{\delta} + \bar{\epsilon}. \nonumber
\end{alignat}
\end{proof}

\section{Proof of Proposition \ref{proposition:errorboundpartial}\label{proof:proposition:errorboundpartial}}

\begin{proof}
Let $\tilde{\pi}^{t+1}(a \vert s) = \sum_{m\in\mathcal{C}} q^{\prime}_{m} \pi^{t+1}_{m}(a \vert s), \forall s \in \mathcal{S}, a \in \mathcal{A}$ denote the expected output of a set $\mathcal{C}$ of local policies. By the triangle inequality, we have
\begin{alignat}{1}
\left\Vert T^{\pi^{t+1}}_{I} V^{t} - T_{I} V^{t} \right\Vert &\le \left\Vert T^{\pi^{t+1}}_{I} V^{t} - T^{\bar{\pi}^{t+1}}_{I} V^{t} \right\Vert + \left\Vert T^{\tilde{\pi}^{t+1}}_{I} V^{t} - T^{\tilde{\pi}^{t+1}}_{I} \bar{V}^{t} \right\Vert \nonumber \\
& \quad + \left\Vert T^{\tilde{\pi}^{t+1}}_{I} \bar{V}^{t} - T_{I} \bar{V}^{t} \right\Vert + \left\Vert T_{I} \bar{V}^{t} - T_{I} V^{t} \right\Vert, \nonumber
\end{alignat}
from which we can obtain
\begin{alignat}{1}
\left\Vert T^{\pi^{t+1}}_{I} V^{t} - T_{I} V^{t} \right\Vert &\le \frac{\tilde{\varepsilon}_{\theta} \sqrt{\left\vert \mathcal{A} \right\vert} R_{\max}}{1 - \gamma} + 2 \gamma \bar{\varepsilon}_{w} + \left\Vert T^{\tilde{\pi}^{t+1}}_{I} \bar{V}^{t} - T_{I} \bar{V}^{t} \right\Vert, \label{eq:proof:proposition:errorboundpartial0}
\end{alignat}
by Lemma \ref{lemma:normtpitbarpi}, the contraction property of the Bellman operators, and the definition of $\bar{\varepsilon}_{w}$.

To finish the proof, it suffices to bound the third term on the RHS of (\ref{eq:proof:proposition:errorboundpartial0}). By the definition of the Bellman operators, we have
\begin{alignat}{1}
\left\Vert T^{\tilde{\pi}^{t+1}}_{I} \bar{V}^{t} - T_{I} \bar{V}^{t} \right\Vert & = \left\Vert \sum_{m\in\mathcal{C}} q^{\prime}_{m} \sum_{n=1}^{N} q_{n} \left( T^{\pi^{t+1}_{m}}_{I} \bar{V}^{t} -  T_{n} \bar{V}^{t} \right) \right\Vert \nonumber \\
& \le \sum_{m\in\mathcal{C}} q^{\prime}_{m} \sum_{n=1}^{N} q_{n} \left\Vert T^{\pi^{t+1}_{m}}_{I} \bar{V}^{t} -  T_{n} \bar{V}^{t} \right\Vert. \label{eq:proof:proposition:errorboundpartial1}
\end{alignat}
The RHS of (\ref{eq:proof:proposition:errorboundpartial1}) can be further bounded by
\begin{alignat}{1}
\left\Vert T^{\pi^{t+1}_{m}}_{I} \bar{V}^{t} - T_{n} \bar{V}^{t} \right\Vert & \le \left\Vert T^{\pi^{t+1}_{m}}_{I} \bar{V}^{t} -T^{\pi^{t+1}_{m}}_{I} V^{t}_{m} \right\Vert + \left\Vert T^{\pi^{t+1}_{m}}_{I} V^{t}_{m} - T_{m} V^{t}_{m} \right\Vert \nonumber \\
& \quad + \left\Vert T_{m} V^{t}_{m} - T_{n} V^{t}_{m} \right\Vert + \left\Vert T_{n} V^{t}_{m} - T_{n} \bar{V}^{t} \right\Vert \nonumber \\
& \stackrel{(a)}{\le} 2 \gamma \left( \left\Vert \bar{V}^{t} - V^{\pi^{t}}_{m} \right\Vert + \left\Vert V^{\pi^{t}}_{m} - V^{t}_{m} \right\Vert \right) + \left\Vert T^{\pi^{t+1}_{m}}_{I} V^{t}_{m} - T^{\pi^{t+1}_{m}}_{m} V^{t}_{m} \right\Vert \nonumber  \\
& \quad + \left\Vert T^{\pi^{t+1}_{m}}_{m} V^{t}_{m} - T_{m} V^{t}_{m} \right\Vert + \left\Vert T_{m} V^{t}_{m} - T_{n} V^{t}_{m} \right\Vert \nonumber \\
& \stackrel{(b)}{\le} \sum_{j=1}^{N} q_{j} \frac{2 \gamma^{2} R_{\max} \kappa_{m,j}}{(1 - \gamma)^{2}} + 2 \gamma \bar{\delta} + 2 \gamma \delta_{m} + \epsilon_{m} \nonumber \\
& \quad + \frac{\gamma R_{\max} \kappa_{m,n}}{1 - \gamma} + \frac{\gamma R_{\max} \kappa_{m,I}}{1 - \gamma},
\label{eq:proof:proposition:errorboundpartial2}
\end{alignat}
where step (a) follows from the contraction property of the Bellman operators and the triangle inequality. Step (b) follows from Lemma \ref{lemma:normtmtn} and (\ref{eq:deltan}).
Thus, by substituting (\ref{eq:proof:proposition:errorboundpartial2}) into the RHS of (\ref{eq:proof:proposition:errorboundpartial1}), 
we can obtain
\begin{alignat}{1}
\left\Vert T^{\pi^{t+1}}_{I} \bar{V}^{\pi^{t}} - T_{I} \bar{V}^{\pi^{t}} \right\Vert & \le \sum_{m\in\mathcal{C}} \sum_{n=1}^{N} q^{\prime}_{m} q_{n} \left\Vert T^{\pi^{t+1}_{m}}_{I} \bar{V}^{\pi^{t}} -  T_{n} \bar{V}^{\pi^{t}} \right\Vert \nonumber \\
& \le \sum_{m\in\mathcal{C}} \sum_{n=1}^{N} q^{\prime}_{m} q_{n} \frac{\left( \gamma + \gamma^{2} \right) R_{\max} \kappa_{m,n}}{(1 - \gamma)^{2}} + \sum_{m\in\mathcal{C}} q^{\prime}_{m} \frac{\gamma R_{\max} \kappa_{m,I}}{1 - \gamma} \nonumber \\
& \quad + 2 \gamma \sum_{m \in \mathcal{C}} q^{\prime}_{m} \delta_{m} + 2 \gamma \bar{\delta} + \sum_{m \in \mathcal{C}} q^{\prime}_{m} \epsilon_{m}. \label{eq:proof:proposition:errorboundpartial4}
\end{alignat}
By substituting (\ref{eq:proof:proposition:errorboundpartial4}) into (\ref{eq:proof:proposition:errorboundpartial0}), we can conclude
\begin{alignat}{1}
\left\Vert T^{\pi^{t+1}}_{I} V^{t} - T_{I} V^{t} \right\Vert &\le \frac{\tilde{\varepsilon}_{\theta} \sqrt{\left\vert \mathcal{A} \right\vert} R_{\max}}{1 - \gamma} + 2 \gamma \bar{\varepsilon}_{w} + \sum_{m\in\mathcal{C}} \sum_{n=1}^{N} q^{\prime}_{m} q_{n} \frac{\left( \gamma + \gamma^{2} \right) R_{\max} \kappa_{m,n}}{(1 - \gamma)^{2}} \nonumber \\
& \quad + \sum_{m\in\mathcal{C}} q^{\prime}_{m} \frac{\gamma R_{\max} \kappa_{m,I}}{1 - \gamma} + 2 \gamma \sum_{m \in \mathcal{C}} q^{\prime}_{m} \delta_{m} + 2 \gamma \bar{\delta} + \sum_{m \in \mathcal{C}} q^{\prime}_{m} \epsilon_{m}. \nonumber
\end{alignat}
\end{proof}

\section{Proof of Theorem \ref{theorem:errorboundfinal}\label{proof:theorem:errorboundfinal}}


\begin{proof}
By the triangle inequality, for any state $s$, we have
\begin{alignat}{1}
\left\vert \bar{V}^{\pi^{t}}(s) - \bar{V}^{\pi^{\ast}}(s) \right\vert & \le \left\vert \bar{V}^{\pi^{t}}(s) - V^{\pi^{t}}_{I}(s) \right\vert + \left\vert V^{\pi^{t}}_{I}(s) - V^{\pi^{\ast}_{I}}_{I}(s) \right\vert + \left\vert V^{\pi^{\ast}_{I}}_{I}(s) - \bar{V}^{\pi^{\ast}}(s) \right\vert, \label{eq:proof:theorem:errorboundfinal1}
\end{alignat}
where the first and second terms on the RHS of (\ref{eq:proof:theorem:errorboundfinal1}) can be upper bounded by Lemma \ref{lemma:normbarvdiffvi} and Proposition \ref{proposition:BertsekasProp2.4.3}, respectively.

To bound the third term on the RHS of (\ref{eq:proof:theorem:errorboundfinal1}), we notice that, for certain states, the distance between the value function $V^{\pi^{\ast}_{I}}_{I} = V_{I}^{\ast}$ for the optimal policy $\pi^{\ast}_{I}$ in the imaginary MDP $\mathcal{M}_{I}$ and the average value function $\bar{V}^{\pi^{\ast}}$ of the optimal policy $\pi^{\ast}$ for (\ref{eq:4}) is upper bounded. Specifically, for state $s$ such that $\bar{V}^{\pi^{\ast}}(s) \ge \bar{V}^{\pi^{\ast}_{I}}(s)$, we have
\begin{alignat}{1}
\left\vert V^{\pi^{\ast}_{I}}_{I}(s) - \bar{V}^{\pi^{\ast}}(s) \right\vert & \le \left\vert V^{\pi^{\ast}}_{I}(s) - \bar{V}^{\pi^{\ast}}(s) \right\vert \le \frac{\gamma R_{\max} \kappa_{1}}{(1 - \gamma)^{2}}. \label{eq:proof:theorem:errorboundfinal2}
\end{alignat}
By substituting (\ref{eq:proof:theorem:errorboundfinal2}), Lemma \ref{lemma:normbarvdiffvi} and Proposition \ref{proposition:BertsekasProp2.4.3} into (\ref{eq:proof:theorem:errorboundfinal1}), for a given state $s$ with $\bar{V}^{\pi^{\ast}}(s) \ge \bar{V}^{\pi^{\ast}_{I}}(s)$, we have
\begin{alignat}{1}
\limsup_{t\to\infty} \left\vert \bar{V}^{\pi^{t}}(s) - \bar{V}^{\pi^{\ast}}(s) \right\vert & \le \frac{\gamma R_{\max} \kappa_{1}}{(1 - \gamma)^{2}} + \frac{\tilde{\epsilon} + 2 \gamma \dot{\delta}}{(1 - \gamma)^{2}} + \frac{\gamma R_{\max} \kappa_{1}}{(1 - \gamma)^{2}}, \label{eq:roof:theorem:errorboundfinal1}
\end{alignat}
where $\tilde{\epsilon}$ may be one of $\dot{\epsilon}$ (Lemma \ref{lemma:wofpepieb}), $\epsilon^{\prime}$ (Lemma \ref{lemma:wfpepieb}), $\hat{\epsilon}$ (Proposition \ref{proposition:errorboundpartial}) or $\acute{\epsilon}$ (Proposition \ref{proposition:wfpeerrorboundpartial}).

When the aforementioned condition does not hold, we can alternatively bound the distance between $\bar{V}^{\pi^{t}}(s)$ and $\bar{V}^{\pi^{\ast}_{I}}(s)$ as
\begin{alignat}{1}
\left\vert \bar{V}^{\pi^{t}}(s) - \bar{V}^{\pi^{\ast}_{I}}(s) \right\vert & \le \left\vert \bar{V}^{\pi^{t}}(s) - V^{\pi^{t}}_{I}(s) \right\vert + \left\vert V^{\pi^{t}}_{I}(s) - V^{\pi^{\ast}_{I}}_{I}(s) \right\vert + \left\vert V^{\pi^{\ast}_{I}}_{I}(s) - \bar{V}^{\pi^{\ast}_{I}}(s) \right\vert, \label{eq:proof:theorem:errorboundfinal3}
\end{alignat}
and actually $\pi^{\ast}_{I}$ performs better on these states.

By substituting Lemma \ref{lemma:normbarvdiffvi} and Proposition \ref{proposition:BertsekasProp2.4.3} into (\ref{eq:proof:theorem:errorboundfinal3}), for a given state $s$ with $\bar{V}^{\pi^{\ast}}(s) \le \bar{V}^{\pi^{\ast}_{I}}(s)$, we have
\begin{alignat}{1}
\limsup_{t\to\infty} \left\vert \bar{V}^{\pi^{t}}(s) - \bar{V}^{\pi^{\ast}_{I}}(s) \right\vert & \le \frac{\gamma R_{\max} \kappa_{1}}{(1 - \gamma)^{2}} + \frac{\tilde{\epsilon} + 2 \gamma \dot{\delta}}{(1 - \gamma)^{2}} + \frac{\gamma R_{\max} \kappa_{1}}{(1 - \gamma)^{2}}. \label{eq:roof:theorem:errorboundfinal2}
\end{alignat}
Let $\bar{V}^{\max}_{s} = \max \left\{ \bar{V}^{\pi^{\ast}}(s), \bar{V}^{\pi^{\ast}_{I}}(s) \right\},\forall s \in \mathcal{S}$. We can then combine (\ref{eq:roof:theorem:errorboundfinal1}) and (\ref{eq:roof:theorem:errorboundfinal2}) into
\begin{alignat}{1}
\limsup_{t\to\infty} \left\vert \bar{V}^{\pi^{t}}(s) - \bar{V}^{\max}_{s} \right\vert & \le \frac{\tilde{\epsilon} + 2 \gamma \dot{\delta}}{(1 - \gamma)^{2}} + 2 \frac{\gamma R_{\max} \kappa_{1}}{(1 - \gamma)^{2}}, \nonumber
\end{alignat}
which proves (\ref{eq:theorem:errorboundfinal1}) in Theorem \ref{theorem:errorboundfinal}. The error bound for Algorithm \ref{alg:FAPI2} with partial client participation can be obtained by replacing $\tilde{\epsilon}$ with $\hat{\epsilon}$ as
\begin{alignat}{1}
\limsup_{t\to\infty} \left\vert \bar{V}^{\pi^{t}}(s) - \bar{V}^{\max}_{s} \right\vert & \le \frac{\hat{\epsilon} + 2 \gamma \dot{\delta}}{(1 - \gamma)^{2}} + 2 \frac{\gamma R_{\max} \kappa_{1}}{(1 - \gamma)^{2}} \nonumber \\
& = \frac{2 \gamma (\gamma^{2} - \gamma + 1)}{(1 - \gamma)^{4}} R_{\max} \kappa_{1} + \frac{\gamma}{(1 - \gamma)^{3}} R_{\max} \sum_{m\in\mathcal{C}} q^{\prime}_{m} \kappa_{m,I} \nonumber \\
& \quad + \frac{\gamma + \gamma^{2}}{(1 - \gamma)^{4}} R_{\max} \sum_{m\in\mathcal{C}} \sum_{n=1}^{N} q^{\prime}_{m} q_{n} \kappa_{m,n} \nonumber \\
& \quad + \frac{\tilde{\varepsilon}_{\theta} \sqrt{\left\vert \mathcal{A} \right\vert} R_{\max}}{\left( 1 - \gamma \right)^{3}} + \frac{2 \gamma \bar{\varepsilon}_{w}}{\left( 1 - \gamma \right)^{2}} \nonumber \\
& \quad + \tilde{\mathcal{O}}\left( \bar{\delta} + \sum_{m \in \mathcal{C}} q^{\prime}_{m} \delta_{m} + \sum_{m \in \mathcal{C}} q^{\prime}_{m} \epsilon_{m} \right), \nonumber
\end{alignat}
where $\tilde{\mathcal{O}}$ omits some constants related to $\gamma$. This proves (\ref{eq:theorem:errorboundfinal2}) in Theorem \ref{theorem:errorboundfinal}.
\end{proof}

\subsection{Interpretation of the Error Bound}

It is difficult to analyze FAPI because we can not directly apply the results of API to FAPI. Fortunately, the use of the imaginary MDP $M_{I}$ aligns FAPI with API and enables the application of Proposition \ref{proposition:BertsekasProp2.4.3} in Theorem \ref{theorem:errorboundfinal}.

However, the imaginary MDP $\mathcal{M}_{I}$ also brings two problems: (1) While $\pi^{\ast}$ is the optimal policy for the objective function of FRL (\ref{eq:4}), Proposition \ref{proposition:BertsekasProp2.4.3} can only show how far the generated policy $\pi^{t}$ is from the optimal policy $\pi_{I}^{\ast}$ in the imaginary MDP (refer to Remark \ref{remark:errorboundfinal1}); and (2) The performance ($V^{\pi}_{I}$) of any policy $\pi$ in the imaginary MDP $\mathcal{M}_{I}$ does not reflect its performance ($\bar{V}^{\pi}$) on FRL (\ref{eq:4}). These problems make the proof intractable as we have to bound the distance between $\bar{V}^{\pi^{\ast}}$ and $V_{I}^{\pi^{\ast}_{I}}$, which is not always feasible. As shown in (\ref{eq:proof:theorem:errorboundfinal2}), their distance on a given state $s$ is bounded only when $\bar{V}^{\pi^{\ast}}(s) \ge \bar{V}^{\pi^{\ast}_{I}}(s)$. The difficulty of bounding their distance on all states stems from the fact that the optimal policy $\pi^{\ast}$ for (\ref{eq:4}) does not necessarily outperform other policies on every state (i.e., it may not be uniformly the best).

For the states where $\pi^{\ast}_{I}$ outperforms $\pi^{\ast}$, we can alternatively bound the error with respect to $\bar{V}^{\pi^{\ast}_{I}}$. Although $\bar{V}^{\pi^{\ast}_{I}}$ is not directly related to the objective function of FRL (\ref{eq:4}), it can be regarded as an approximation to $\bar{V}^{\pi^{\ast}}$, especially when the level of heterogeneity is low where $\bar{V}^{\pi^{\ast}_{I}}$ is close to $\bar{V}^{\pi^{\ast}}$. As a result, the error bound in (\ref{eq:theorem:errorboundfinal1}) is a combination of two bounds with respect to $\bar{V}^{\pi^{\ast}_{I}}$ and $\bar{V}^{\pi^{\ast}}$, respectively. An illustrative example is given in Figure \ref{fig:errorbound}.

\begin{figure}[ht!] 
\centering
\includegraphics[width=0.66\columnwidth]{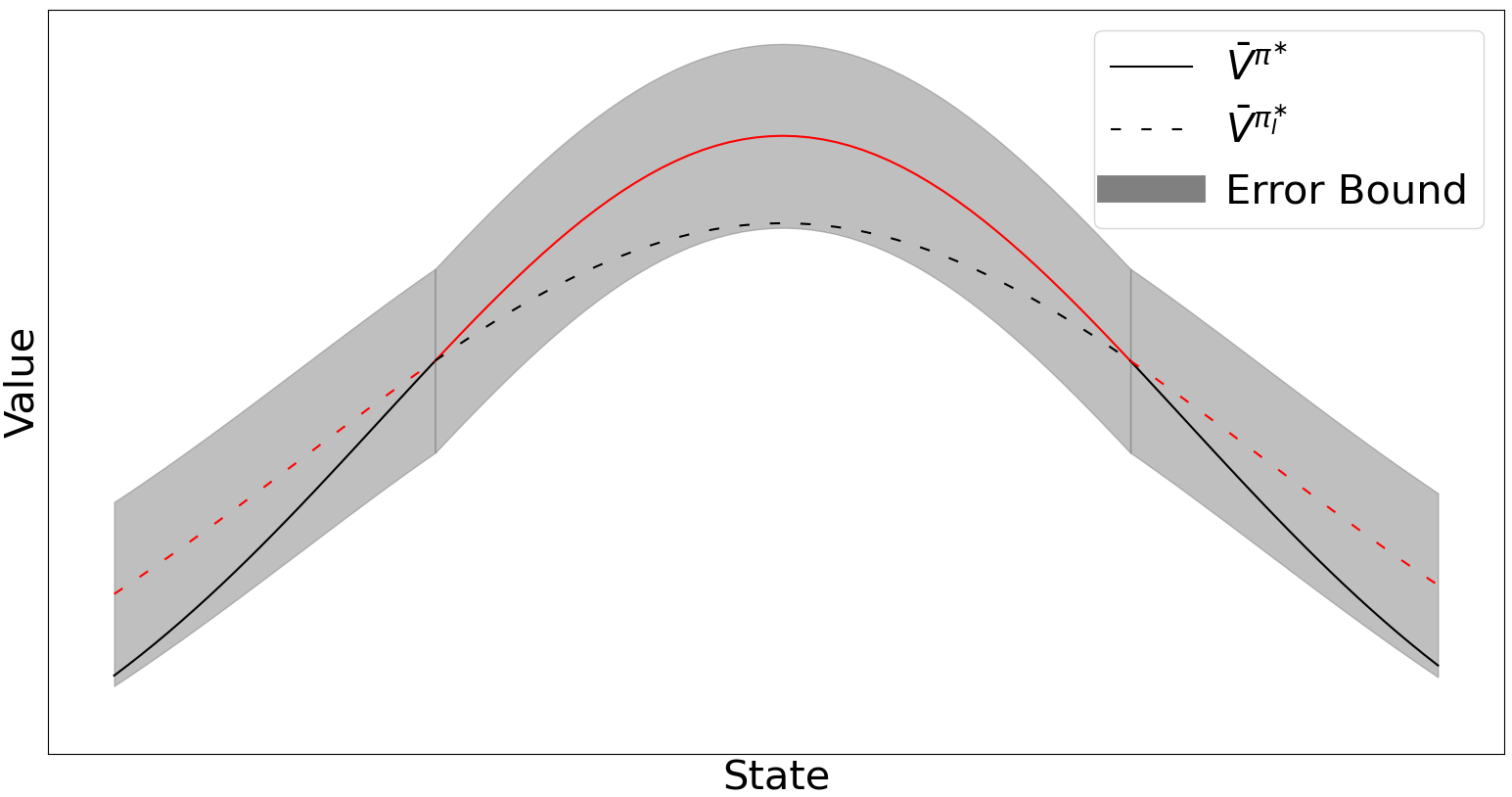}
\caption{The area in grey indicates the error bound in (\ref{eq:theorem:errorboundfinal1}). The maximum of the two curves on each state is highlighted in red, and the bound is drawn with respect to this highlighted curve.}
\label{fig:errorbound}
\end{figure}

\section{Proof of Proposition \ref{proposition:wfpeerrorboundpartial} \label{proof:proposition:wfpeerrorboundpartial}}

\begin{proof}
Let $\tilde{\pi}^{t+1}(a \vert s) = \sum_{m\in\mathcal{C}} q^{\prime}_{m} \pi^{t+1}_{m}(a \vert s), \forall s \in \mathcal{S}, a \in \mathcal{A}$ denote the expected output of a set $\mathcal{C}$ of local policies. For any state $s$, we have
\begin{alignat}{1}
\left\vert T_{I}^{\pi^{t+1}} V^{t}(s) - T_{I} V^{t}(s) \right\vert &\le \left\vert T_{I}^{\pi^{t+1}} V^{t}(s) - T^{\bar{\pi}^{t+1}}_{I} V^{t}(s) \right\vert + \left\vert T^{\bar{\pi}^{t+1}}_{I} V^{t}(s) - T_{I} V^{t}(s) \right\vert. \label{eq:eq:proposition:wfpeerrorboundpartial0}
\end{alignat}
By Lemma \ref{lemma:normtpitbarpi}, the first term on the right-hand side (RHS) of (\ref{eq:eq:proposition:wfpeerrorboundpartial0}) is upper bounded by
\begin{alignat}{1}
\left\vert T_{I}^{\pi^{t+1}} V^{t}(s) - T^{\bar{\pi}^{t+1}}_{I} V^{t}(s) \right\vert &\le \frac{\bar{\varepsilon}_{\theta} \sqrt{\left\vert \mathcal{A} \right\vert} R_{\max}}{1 - \gamma}.
\end{alignat}

To finish the proof, it suffices to bound the second term on the RHS of (\ref{eq:eq:proposition:wfpeerrorboundpartial0}). By the definition of the Bellman operators, we have
\begin{alignat}{1}
\left\Vert T^{\tilde{\pi}^{t+1}}_{I} V^{t} - T_{I} V^{t} \right\Vert & = \left\Vert \sum_{m\in\mathcal{C}} q^{\prime}_{m} \sum_{n=1}^{N} q_{n} \left( T^{\pi^{t+1}_{m}}_{I} V^{t} -  T_{n} V^{t} \right) \right\Vert \nonumber \\
& \le \sum_{m\in\mathcal{C}} q^{\prime}_{m} \sum_{n=1}^{N} q_{n} \left\Vert T^{\pi^{t+1}_{m}}_{I} V^{t} -  T_{n} V^{t} \right\Vert. \label{eq:eq:proposition:wfpeerrorboundpartial2}
\end{alignat}
The RHS of (\ref{eq:eq:proposition:wfpeerrorboundpartial2}) can be further bounded by
\begin{alignat}{1}
\left\Vert T^{\pi^{t+1}_{m}}_{I} V^{t} - T_{n} V^{t} \right\Vert & \le \left\Vert T^{\pi^{t+1}_{m}}_{I} V^{t} - T^{\pi^{t+1}_{m}}_{m} V^{t} \right\Vert + \left\Vert T^{\pi^{t+1}_{m}}_{m} V^{t} - T_{m} V^{t} \right\Vert + \left\Vert T_{m} V^{t} - T_{n} V^{t} \right\Vert \nonumber \\
& \le \frac{\gamma R_{\max} \kappa_{m,I}}{1 - \gamma} + \epsilon_{m} + \frac{\gamma R_{\max} \kappa_{m,n}}{1 - \gamma}, \label{eq:eq:proposition:wfpeerrorboundpartial3}
\end{alignat}
where the last inequality follows from Lemma \ref{lemma:normtmtn} and (\ref{eq:deltan}). By substituting (\ref{eq:eq:proposition:wfpeerrorboundpartial3}) into (\ref{eq:eq:proposition:wfpeerrorboundpartial2}), we can obtain
\begin{alignat}{1}
\left\Vert T^{\pi^{t+1}}_{I} V^{\pi^{t}} - T_{I} V^{\pi^{t}} \right\Vert & \le \sum_{m\in\mathcal{C}} \sum_{n=1}^{N} q^{\prime}_{m} q_{n} \left\Vert T^{\pi^{t+1}_{m}}_{I} V^{\pi^{t}} -  T_{n} V^{\pi^{t}} \right\Vert \nonumber \\
& \le \sum_{m\in\mathcal{C}} \sum_{n=1}^{N} q^{\prime}_{m} q_{n} \frac{ \gamma R_{\max} \kappa_{m,n}}{1 - \gamma} + \sum_{m\in\mathcal{C}} q^{\prime}_{m} \frac{\gamma R_{\max} \kappa_{m,I}}{1 - \gamma} + \sum_{m \in \mathcal{C}} q^{\prime}_{m} \epsilon_{m}. \label{eq:eq:proposition:wfpeerrorboundpartial4}
\end{alignat}

By substituting (\ref{eq:eq:proposition:wfpeerrorboundpartial4}) into (\ref{eq:eq:proposition:wfpeerrorboundpartial0}), we can conclude
\begin{alignat}{1}
\left\Vert T^{\pi^{t+1}}_{I} V^{t} - T_{I} V^{t} \right\Vert &\le \frac{\tilde{\varepsilon}_{\theta} \sqrt{\left\vert \mathcal{A} \right\vert} R_{\max}}{1 - \gamma} + \sum_{m\in\mathcal{C}} \sum_{n=1}^{N} q^{\prime}_{m} q_{n} \frac{ \gamma R_{\max} \kappa_{m,n}}{1 - \gamma} \nonumber \\
& \quad + \sum_{m\in\mathcal{C}} q^{\prime}_{m} \frac{\gamma R_{\max} \kappa_{m,I}}{1 - \gamma} + \sum_{m \in \mathcal{C}} q^{\prime}_{m} \epsilon_{m}. \nonumber
\end{alignat}
\end{proof}

\section{Proof of Proposition \ref{proposition:homoerrorboundfinal} \label{proof:proposition:homoerrorboundfinal}}

\begin{proof}
By Theorem \ref{theorem:errorboundfinal}, we have
\begin{alignat}{1}
\limsup_{t\to\infty} \left\vert \bar{V}^{\pi^{t}}(s) - \bar{V}^{\max}_{s} \right\vert & \le \frac{\tilde{\epsilon} + 2 \gamma \dot{\delta}}{(1 - \gamma)^{2}} + 2 \frac{\gamma R_{\max} \kappa_{1}}{(1 - \gamma)^{2}}, \nonumber
\end{alignat}
where $\tilde{\epsilon}$ may be one of $\acute{\epsilon}$ (Proposition \ref{proposition:wfpeerrorboundpartial}) or $\epsilon^{\prime}$ (Lemma \ref{lemma:wfpepieb}). Since the environments are homogeneous, we have
\begin{alignat}{1}
\bar{V}^{\max}_{s} & = \bar{V}^{\pi^{\ast}}(s) = \bar{V}^{\pi^{\ast}_{I}}(s), \nonumber \\
\kappa_{1} & = \kappa_{2} = 0, \nonumber \\
\dot{\delta} & = \bar{\delta}, \nonumber \\
\epsilon^{\prime} & = \acute{\epsilon} = \bar{\epsilon} + \frac{\hat{\varepsilon}_{\theta} \sqrt{\left\vert \mathcal{A} \right\vert} R_{\max}}{1 - \gamma}, \nonumber
\end{alignat}
where $\hat{\varepsilon}_{\theta}$ is equal to $\bar{\varepsilon}_{\theta}$ and $\tilde{\varepsilon}_{\theta}$ for full participation and partial participation, respectively. Thus, we can conclude
\begin{alignat}{1}
\limsup_{t\to\infty} \left\Vert \bar{V}^{\pi^{t}} - \bar{V}^{\pi^{\ast}} \right\Vert \le \frac{\hat{\varepsilon}_{\theta} \sqrt{\left\vert \mathcal{A} \right\vert} R_{\max}}{\left( 1 - \gamma \right)^{3}} + \frac{2 \gamma \bar{\varepsilon}_{w}}{\left( 1 - \gamma \right)^{2}} + \frac{\bar{\epsilon} + 2 \gamma \bar{\delta}}{(1 - \gamma)^2}. \nonumber
\end{alignat}
\end{proof}

\section{Proof of Lemma \ref{lemma:aggregation_error} \label{proof:lemma:aggregation_error}}

Our proof relies on the definition of local linearization for the two-layer neural network at its random initialization, which was first introduced by \citep{NEURIPS2019_98baeb82,wang2019neural,NEURIPS2019_227e072d}:
\begin{alignat}{1}
u_{w^{t}}^{0}(s) = \frac{1}{\sqrt{m}} \sum_{i}^{m} b_{i} \cdot \mathbf{1} \left\{ \left( w^{0}_{i} \right)^{T}(s) > 0 \right\} \left( w^{t}_{i} \right)^{T}(s), \nonumber \\
f_{\theta^{t}}^{0}(s,a) = \frac{1}{\sqrt{m}} \sum_{i}^{m} b_{i} \cdot \mathbf{1} \left\{ \left( \theta^{0}_{i} \right)^{T}(s,a) > 0 \right\} \left( \theta^{t}_{i} \right)^{T}(s,a). \nonumber
\end{alignat}
The following lemma characterizes the error induced by the above local linearization.
\begin{lemma}
\label{lemma:linearization_error}
For $w \in \mathcal{B}_{R_{w}}^{0}, \theta \in \mathcal{B}_{R_{\theta}}^{0}, s \in \mathcal{S}$, and $a \in \mathcal{A}$, we have
\begin{alignat}{1}
\mathbb{E}_{\text{init}} \left[ \left\vert f_{\theta}(s,a) - f_{\theta}^{0}(s,a) \right\vert \right] &= \mathcal{O} \left( R_{\theta}^{6/5} m^{-1/10} \hat{R}_{\theta}^{2/5} \right), \label{eq:lemma:linearization_error_first} \\
\mathbb{E}_{\text{init}} \left[ \left\vert u_{w}(s) - u_{w}^{0}(s) \right\vert \right] &= \mathcal{O} \left( R_{w}^{6/5} m^{-1/10} \hat{R}_{w}^{2/5} \right). \label{eq:lemma:linearization_error_second}
\end{alignat}
\end{lemma}
\begin{proof}
Given any pair of model parameters $\theta \in \mathcal{B}_{R_{\theta}}^{0}$ and $\theta^{\prime} \in \mathcal{B}_{R_{\theta}}^{0}$,
\begin{alignat}{1}
\mathbf{1} \left\{ \theta_{i}^{T}(s,a) > 0 \right\} \neq \mathbf{1} \left\{ \left( \theta^{\prime}_{i} \right) ^{T}(s,a) > 0 \right\} \nonumber
\end{alignat}
implies
\begin{alignat}{1}
\left\vert \left( \theta^{\prime}_{i} \right)^{T}(s,a) \right\vert \le \left\vert \theta_{i}^{T}(s,a) - \left( \theta^{\prime}_{i} \right)^{T}(s,a) \right\vert \le \left\Vert \theta_{i} - \theta^{\prime}_{i} \right\Vert_{2}.  \nonumber
\end{alignat}
Consequently, we have
\begin{alignat}{1}
\left\vert f_{\theta}(s,a) - f_{\theta}^{0}(s,a) \right\vert &= \frac{1}{\sqrt{m}} \left\vert \sum_{i=1}^{m} b_{i} \cdot \left( \mathbf{1} \left\{ \theta^{T}_{i}(s,a) > 0 \right\} - \mathbf{1} \left\{ \left( \theta^{0}_{i} \right)^{T}(s,a) > 0 \right\} \right) \cdot \theta_{i}^{T}(s,a) \right\vert \nonumber \\
&\le \frac{1}{\sqrt{m}} \sum_{i=1}^{m} \left\vert \mathbf{1} \left\{ \theta^{T}_{i}(s,a) > 0 \right\} - \mathbf{1} \left\{ \left( \theta^{0}_{i} \right)^{T}(s,a) > 0 \right\} \right\vert \cdot \left\vert \theta_{i}^{T}(s,a) \right\vert \nonumber \\
&\le \frac{1}{\sqrt{m}} \sum_{i=1}^{m} \mathbf{1} \left\{ \left( \theta^{0}_{i} \right)^{T}(s,a) \le \left\Vert \theta_{i} - \theta^{0}_{i} \right\Vert_{2} \right\} \cdot \left\Vert \theta_{i} - \theta^{0}_{i} \right\Vert_{2}. \label{eq:lemma:linearization_error0}
\end{alignat}
Next, we analyze $\mathbb{E}_{\text{init}} \left[ \left\vert f_{\theta}(s,a) - f_{\theta}^{0}(s,a) \right\vert \right]$ by examining two cases.

\begin{itemize}
    \item Case 1: $G_{1} = \frac{1}{\sqrt{m}} \sum_{i \in C_{1}} \left( \mathbf{1} \left\{ \left( \theta^{0}_{i} \right)^{T}(s,a) \le \left\Vert \theta_{i} - \theta^{0}_{i} \right\Vert_{2} \right\} \right) \cdot \left\Vert \theta_{i} - \theta^{0}_{i} \right\Vert_{2}$, where $C_{1} = \left\{ i \in [m]: \left\Vert \theta_{i} - \theta^{0}_{i} \right\Vert_{2} \le \Delta \right\}$ and $\Delta > 0$.

Without loss of generality, we assume that parameters $\theta(0)$ are uniformly initialized from a circle with a radius of $R_{0}$. Then, the number of neurons lying in $C_{1}$ is approximately $\frac{m \Delta}{R_{0}}$. Thus, we have
\begin{alignat}{1}
G_{1} &\le \frac{1}{\sqrt{m}} \sum_{i \in C_{1}} \left\Vert \theta_{i} - \theta^{0}_{i} \right\Vert_{2} = \mathcal{O} \left( m^{1/2} \Delta^{2} R_{0}^{-1} \right). \label{eq:lemma:linearization_error4}
\end{alignat}
Note that the size of $C_{1}$ decreases as $\Delta$ decreases. Given a fixed $m$, a sufficiently small $\Delta$ exists that makes (\ref{eq:lemma:linearization_error4}) negligible.

    \item Case 2: $G_{2} = \frac{1}{\sqrt{m}} \sum_{i \in C_{2}} \left( \mathbf{1} \left\{ \left( \theta^{0}_{i} \right)^{T}(s,a) \le \left\Vert \theta_{i} - \theta^{0}_{i} \right\Vert_{2} \right\} \right) \cdot \left\Vert \theta_{i} - \theta^{0}_{i} \right\Vert_{2}$, where $C_{2} = \left\{ i \in [m]: \left\Vert \theta_{i} - \theta^{0}_{i} \right\Vert_{2} > \Delta \right\}$ and $\Delta > 0$.

We have $\left\Vert \theta_{i} - \theta^{0}_{i} \right\Vert_{2} / \left\Vert \theta^{0}_{i} \right\Vert_{2} > \Delta / \hat{R}_{\theta}$. Consequently, there exits a constant $c \ge \hat{R}_{\theta} / \Delta$ such that for any layer $i \in C_{2}, a \in \mathcal{A}, s \in \mathcal{S}$, it holds that
\begin{alignat}{1}
\mathbf{1} \left\{ \left( \theta^{0}_{i} \right)^{T}(s,a) \le \left\Vert \theta_{i} - \theta^{0}_{i} \right\Vert_{2} \right\} \le 1 \le c \left\Vert \theta_{i} - \theta^{0}_{i} \right\Vert_{2} / \left\Vert \theta^{0}_{i} \right\Vert_{2}. \label{eq:lemma:linearization_error5}
\end{alignat}
Next, by the Cauchy-Schwarz inequality and $\left\Vert \theta - \theta^{0} \right\Vert_{2} \le R_{\theta}$, we have
\begin{alignat}{1}
G_{2} &\le \frac{ R_{\theta}}{\sqrt{m}} \sqrt{\sum_{i \in C_{2}} \mathbf{1} \left\{ \left( \theta^{0}_{i} \right)^{T}(s,a) \le \left\Vert \theta_{i} - \theta^{0}_{i} \right\Vert_{2} \right\}}. \label{eq:lemma:linearization_error7}
\end{alignat}
By (\ref{eq:lemma:linearization_error5}) and taking expectation on both sides of (\ref{eq:lemma:linearization_error7}), we can obtain
\begin{alignat}{1}
\mathbb{E}_{\text{init}} \left[ G_{2} \right] &\le \frac{R_{\theta}}{\sqrt{m}} \mathbb{E}_{\text{init}} \left[ \sqrt{\sum_{i \in C_{2}} \mathbf{1} \left\{ \left( \theta^{0}_{i} \right)^{T}(s,a) \le \left\Vert \theta_{i} - \theta^{0}_{i} \right\Vert_{2} \right\}} \right] \nonumber \\
&\le \frac{R_{\theta}}{\sqrt{m}} \sqrt{\mathbb{E}_{\text{init}} \left[ \sum_{i \in C_{2}} \mathbf{1} \left\{ \left( \theta^{0}_{i} \right)^{T}(s,a) \le \left\Vert \theta_{i} - \theta^{0}_{i} \right\Vert_{2} \right\} \right]} \nonumber \\
&\le \frac{R_{\theta}}{\sqrt{m}} \sqrt{c \mathbb{E}_{\text{init}} \left[ \sum_{i \in C_{2}} \left\Vert \theta_{i} - \theta^{0}_{i} \right\Vert_{2} / \Vert \theta^{0}_{i} \Vert_{2} \right]}. \nonumber
\end{alignat}
By the Cauchy-Schwarz inequality, we have
\begin{alignat}{1}
\mathbb{E}_{\text{init}} \left[ \sum_{i \in C_{2}} \left\Vert \theta_{i} - \theta^{0}_{i} \right\Vert_{2} / \Vert \theta^{0}_{i} \Vert_{2} \right] &\le \mathbb{E}_{\text{init}} \left[ \sum_{i \in C_{2}} \left\Vert \theta_{i} - \theta^{0}_{i} \right\Vert_{2}^{2} \right]^{1/2} \cdot \mathbb{E}_{\text{init}} \left[ \sum_{i \in C_{2}} \Vert \theta^{0}_{i} \Vert_{2}^{-2} \right]^{1/2} \nonumber \\
&\le R_{\theta} \cdot \mathbb{E}_{\text{init}} \left[ \sum_{i \in C_{2}} \Vert \theta^{0}_{i} \Vert_{2}^{-2} \right]^{1/2}, \label{eq:lemma:linearization_error3}
\end{alignat}
where the second inequality follows from $\sum_{i=1}^{m} \left\Vert \theta_{i} - \theta^{0}_{i} \right\Vert_{2}^{2} = \left\Vert \theta - \theta^{0} \right\Vert_{2}^{2} \le R_{\theta}^{2}$. Since $\mathbb{E}_{\text{init}} \left[ \left\Vert \theta_{i} \right\Vert_{2}^{-2} \right] \le \infty, \forall i \in [m]$ by the initialization scheme (\ref{eq:initialization}), we have that the RHS of (\ref{eq:lemma:linearization_error3}) is $\mathcal{O} (R_{\theta}m^{1/2})$. Thus, we can obtain
\begin{alignat}{1}
\mathbb{E}_{\text{init}} \left[ G_{2} \right] &= \mathcal{O} \left( R_{\theta}^{3/2}m^{-1/4} \hat{R}_{\theta}^{1/2} \Delta^{-1/2} \right). \label{eq:lemma:linearization_error6}
\end{alignat}

\end{itemize}
By (\ref{eq:lemma:linearization_error4}) and (\ref{eq:lemma:linearization_error6}), we can obtain
\begin{alignat}{1}
\mathbb{E}_{\text{init}} \left[ \left\vert f_{\theta}(s,a) - f_{\theta}^{0}(s,a) \right\vert \right] &\le \mathbb{E}_{\text{init}} \left[ G_{1} + G_{2} \right] = \mathcal{O} \left( R_{\theta}^{3/2}m^{-1/4} \hat{R}_{\theta}^{1/2} \Delta^{-1/2} + m^{1/2} \Delta^{2} R_{0}^{-1} \right). \nonumber
\end{alignat}
We further assume $
m^{1/2} \Delta^{2} R_{0}^{-1} \le \varrho$, i.e., $\Delta \le \varrho^{1/2} m^{-1/4} R_{0}^{1/2}$, which implies that
\begin{alignat}{1}
\mathbb{E}_{\text{init}} \left[ \left\vert f_{\theta}(s,a) - f_{\theta}^{0}(s,a) \right\vert \right] &= \mathcal{O} \left( R_{\theta}^{3/2}m^{-1/4} \hat{R}_{\theta}^{1/2} \Delta^{-1/2} + \varrho \right) \nonumber \\
&= \mathcal{O} \left( R_{\theta}^{3/2}m^{-1/8} \hat{R}_{\theta}^{1/2} \varrho^{-1/4} R_{0}^{-1/4} + \varrho \right). \nonumber
\end{alignat}
Moreover, we assume that $R_{\theta}^{3/2} m^{-1/8} \hat{R}_{\theta}^{1/2} \varrho^{-1/4} R_{0}^{-1/4} \ge \varrho$, i.e., $\varrho \le R_{\theta}^{6/5} m^{-1/10} \hat{R}_{\theta}^{2/5} R_{0}^{-1/5}$, which gives
\begin{alignat}{1}
\mathbb{E}_{\text{init}} \left[ \left\vert f_{\theta}(s,a) - f_{\theta}^{0}(s,a) \right\vert \right] &= \mathcal{O} \left( R_{\theta}^{6/5} m^{-1/10} \hat{R}_{\theta}^{2/5} \right). \nonumber
\end{alignat}
This completes the proof of (\ref{eq:lemma:linearization_error_first}), and the proof of (\ref{eq:lemma:linearization_error_second}) is similar.
\end{proof}

The following lemma provides the upper bound of the difference between network outputs.
\begin{lemma}
\label{lemma:changing_action_error}
For state $s \in \mathcal{S}$, any pair of actions $a$ and $a^{\prime}$, and model parameters $\vartheta,\vartheta^{\prime} \in \mathcal{B}_{R_{\vartheta}}^{0}$, which is $\theta$ for the policy and $w$ for the value function, we have
\begin{alignat}{1}
\mathbb{E}_{\text{init}} \left[ \left\vert u_{\vartheta}(s,a) - u_{\vartheta^{\prime}}(s,a^{\prime}) \right\vert \right] &= \mathcal{O} \left( R_{\vartheta} \right), \label{eq:lemma:changing_action_error_first} \\
\mathbb{E}_{\text{init}} \left[ \left\vert u_{\vartheta}(s,a) - u_{\vartheta^{\prime}}(s,a) \right\vert \right] &= \mathcal{O} \left( R_{\vartheta} \right). \label{eq:lemma:changing_action_error_second}
\end{alignat}
\end{lemma}
\begin{proof}
By Jensen's inequality, we have
\begin{alignat}{1}
& \mathbb{E}_{\text{init}} \left\vert \left[ u_{\vartheta}(s,a) - u_{\vartheta^{\prime}}(s,a^{\prime}) \right\vert \right]^{2} \nonumber \\
&\le \frac{1}{m} \mathbb{E}_{\text{init}} \left[ \left\vert \sum_{i}^{m} b_{i} \cdot \mathbf{1} \left\{ \vartheta_{i}^{T}(s,a) > 0 \right\} \vartheta_{i}^{T} (s,a) - \sum_{i}^{m} b_{i} \cdot \mathbf{1} \left\{ \left( \vartheta^{\prime}_{i} \right)^{T}(s,a^{\prime}) > 0 \right\} \left( \vartheta^{\prime}_{i} \right)^{T} (s,a^{\prime}) \right\vert^{2} \right]. \nonumber
\end{alignat}
By the fact that $(a + b)^{2} \le 2a^{2} + 2b^{2}$ and $ab - cd = a(b-d) + d(a-c)$, we have
\begin{alignat}{1}
& \mathbb{E}_{\text{init}} \left\vert \left[ u_{\vartheta}(s,a) - u_{\vartheta^{\prime}}(s,a^{\prime}) \right\vert \right]^{2} \nonumber \\
\le& \frac{1}{m} \mathbb{E}_{\text{init}} \left[ \left\vert \sum_{i}^{m} b_{i} \cdot \mathbf{1} \left\{ \vartheta_{i}^{T}(s,a) > 0 \right\} \left( \vartheta_{i}^{T} (s,a) - \left( \vartheta^{\prime}_{i} \right)^{T} (s,a^{\prime}) \right) \right. \right. \nonumber \\
& \left. \left. + \sum_{i}^{m} b_{i} \cdot \left( \mathbf{1} \left\{ \left( \vartheta^{\prime}_{i} \right)^{T}(s,a^{\prime}) > 0 \right\} - \mathbf{1} \left\{ \left( \vartheta^{\prime}_{i} \right)^{T}(s,a^{\prime}) > 0 \right\} \right) \left( \vartheta^{\prime}_{i} \right)^{T} (s,a^{\prime}) \right\vert^{2} \right] \nonumber \\
\le& \frac{1}{m} \mathbb{E}_{\text{init}} \left[ 2 \left\vert \sum_{i}^{m} b_{i} \cdot \mathbf{1} \left\{ \vartheta_{i}^{T}(s,a) > 0 \right\} \left( \vartheta_{i}^{T} (s,a) - \left( \vartheta^{\prime}_{i} \right)^{T} (s,a^{\prime}) \right) \right\vert^{2} \right. \nonumber \\
& \left. + 2 \left\vert \sum_{i}^{m} b_{i} \cdot \left( \mathbf{1} \left\{ \left( \vartheta^{\prime}_{i} \right)^{T}(s,a^{\prime}) > 0 \right\} - \mathbf{1} \left\{ \left( \vartheta^{\prime}_{i} \right)^{T}(s,a^{\prime}) > 0 \right\} \right) \left( \vartheta^{\prime}_{i} \right)^{T} (s,a^{\prime}) \right\vert^{2} \right]. \nonumber
\end{alignat}
Furthermore, we have
\begin{alignat}{1}
\left\Vert \vartheta \right\Vert_{2}^{2} &\le \left( \left\Vert \vartheta - \vartheta^{0} \right\Vert_{2} + \left\Vert \vartheta^{0} \right\Vert_{2} \right)^{2} \le 2 R_{\vartheta}^{2} + 2 \left\Vert \vartheta^{0} \right\Vert_{2}^{2}. \label{eq:bound_vartheta}
\end{alignat}
By (\ref{eq:bound_vartheta}) and the Cauchy-Schwarz inequality, we have
\begin{alignat}{1}
\mathbb{E}_{\text{init}} \left\vert \left[ u_{\vartheta}(s,a) - u_{\vartheta^{\prime}}(s,a^{\prime}) \right\vert \right]^{2} &\le 2 \mathbb{E}_{\text{init}} \left[ \left( \sum_{i}^{m} \left( \vartheta_{i}^{T} (s,a) - \left( \vartheta^{\prime}_{i} \right)^{T} (s,a^{\prime}) \right)^{2} \right) + \left\Vert \vartheta^{\prime} \right\Vert_{2}^{2} \right] \nonumber \\
&\le 2 \mathbb{E}_{\text{init}} \left[ \left( \sum_{i}^{m} 2 \left( \vartheta_{i}^{T} (s,a) \right)^{2} + 2 \left( \left( \vartheta^{\prime}_{i} \right)^{T} (s,a^{\prime}) \right)^{2} \right) + \left\Vert \vartheta^{\prime} \right\Vert_{2}^{2} \right] \nonumber \\
&\le 6 \mathbb{E}_{\text{init}} \left[ \left\Vert \vartheta^{\prime} \right\Vert_{2}^{2} \right] + 4 \mathbb{E}_{\text{init}} \left[ \left\Vert \vartheta \right\Vert_{2}^{2} \right] \nonumber \\
&\le 20 R_{\vartheta}^{2} + 20. \label{eq:lemma:changing_action_error_0}
\end{alignat}
We can then complete the proof of (\ref{eq:lemma:changing_action_error_first}) by taking the square root of both sides of (\ref{eq:lemma:changing_action_error_0}). The proof for (\ref{eq:lemma:changing_action_error_second}) is similar.
\end{proof}

Next, we prove Lemma \ref{lemma:aggregation_error}.
\begin{proof}
Let $(t^{\prime},s^{\prime}) = \arg \max_{t>0,s \in \mathcal{S}} \left\vert V^{t}(s) - \bar{V}^{t}(s) \right\vert$, which are dependent on the initialization $w^{0}$. By the triangle inequality and Lemma \ref{lemma:linearization_error}, we have
\begin{alignat}{1}
\mathbb{E}_{\text{init}} \left[ \bar{\varepsilon}_{w} \right] &= \mathbb{E}_{\text{init}} \left[ \max_{t} \left\Vert V^{t} - \bar{V}^{t} \right\Vert \right] \nonumber \\
&\le \mathbb{E}_{\text{init}} \left[ \left\vert V^{t^{\prime}}(s^{\prime}) - V^{t^{\prime},0}(s^{\prime}) \right\vert + \left\vert V^{t^{\prime},0}(s^{\prime}) - \bar{V}^{t^{\prime}}(s^{\prime}) \right\vert \right] \nonumber \\
&= \mathbb{E}_{\text{init}} \left[ \left\vert V^{t^{\prime}}(s^{\prime}) - V^{t^{\prime},0}(s^{\prime}) \right\vert + \left\vert \sum_{n=1}^{N} q_{n} V_{n}^{t^{\prime},0}(s^{\prime}) - \sum_{n=1}^{N} q_{n} V_{n}^{t^{\prime}}(s^{\prime}) \right\vert \right] \nonumber \\
&= \mathcal{O} \left( R_{w}^{6/5} m^{-1/10} \hat{R}_{w}^{2/5} \right), \nonumber
\end{alignat}
which completes the proof of (\ref{eq:lemma:aggregation_error_first}).

Since the order of terms in the square $\left( \pi^{t}(a \vert s) - \sum_{n=1}^{N} q_{n} \pi^{t}_{n}(a \vert s) \right)^{2}$ does not affect the value, we define two sets $C_{1} = \left\{t > 0, a \in \mathcal{A}, s \in \mathcal{S}: \pi^{t}(a \vert s) > \sum_{n=1}^{N} q_{n} \pi^{t}_{n}(a \vert s) \right\}$, and $C_{2} = C_{1}^{\mathsf{C}}$. Accordingly, for all $(t, a, s) \in C_{1}$, we have $\pi^{t}(a \vert s) - \sum_{n=1}^{N} q_{n}\pi^{t}_{n}(a \vert s) \le \pi^{t}(a \vert s) \log \frac{\pi^{t}(a \vert s)}{\sum_{n=1}^{N} q_{n}\pi^{t}_{n}(a \vert s)}$ by the fact that $1 - \frac{1}{x} \le \ln x, \forall x > 0$. By the Arithmetic Mean-Geometric Mean (AM-GM) inequality, we have
\begin{alignat}{1}
\left( \pi^{t}(a \vert s) - \sum_{n=1}^{N} q_{n} \pi^{t}_{n}(a \vert s) \right)^{2} &\le \left( \pi^{t}(a \vert s) - \sum_{n=1}^{N} q_{n} \pi^{t}_{n}(a \vert s) \right) \pi^{t}(a \vert s) \left( \log \frac{\pi^{t}(a \vert s)}{\sum_{n=1}^{N} q_{n} \pi^{t}_{n}(a \vert s)} \right) \nonumber \\
&\le \left( \pi^{t}(a \vert s) - \sum_{n=1}^{N} q_{n} \pi^{t}_{n}(a \vert s) \right) \pi^{t}(a \vert s) \left( \sum_{n=1}^{N} q_{n} \log \frac{\pi^{t}(a \vert s)}{\pi^{t}_{n}(a \vert s)} \right) \nonumber \\
&\le \pi^{t}(a \vert s) \sum_{n=1}^{N} q_{n} \left( f_{\theta^{t}}(s,a) - f_{\theta^{t}_{n}}(s,a) + \log \frac{\sum_{a^{\prime} \in \mathcal{A}} \exp \left( f_{\theta^{t}_{n}}(s,a^{\prime}) \right)}{\sum_{a^{\prime} \in \mathcal{A}} \exp \left( f_{\theta^{t}}(s,a^{\prime}) \right)} \right). \nonumber
\end{alignat}
For all $(t, a, s) \in C_{2}$, we have $\sum_{n=1}^{N} q_{n} \pi^{t}_{n}(a \vert s) - \pi^{t}(a \vert s) = \sum_{n=1}^{N} q_{n} \left( \pi^{t}_{n}(a \vert s) - \pi^{t}(a \vert s) \right) \le \sum_{n=1}^{N} q_{n} \pi_{n}^{t}(a \vert s) \log \frac{\pi^{t}_{n}(a \vert s)}{\pi^{t}(a \vert s)}$ by the fact that $1 - \frac{1}{x} \le \ln x, \forall x > 0$. Therefore, we have
\begin{alignat}{1}
\left( \pi^{t}(a \vert s) - \sum_{n=1}^{N} q_{n} \pi^{t}_{n}(a \vert s) \right)^{2} &\le \sum_{n=1}^{N} q_{n} \pi_{n}^{t}(a \vert s) \log \frac{\pi^{t}_{n}(a \vert s)}{\pi^{t}(a \vert s)}. \nonumber
\end{alignat}
Let $\left( t^{\prime}, s^{\prime} \right) = \arg\max_{t>0, s \in \mathcal{S}} \left\Vert \pi^{t}(\cdot \vert s) - \sum_{n=1}^{N} q_{n} \pi^{t}_{n}(\cdot \vert s) \right\Vert_{2}$, which is conditional on $\theta^{0}$, we have
\begin{alignat}{1}
\bar{\varepsilon}_{\theta} &= \left\Vert \pi^{t^{\prime}}(\cdot \vert s^{\prime}) - \sum_{n=1}^{N} q_{n} \pi^{t^{\prime}}_{n}(\cdot \vert s^{\prime}) \right\Vert_{2} = \sqrt{\sum_{a \in \mathcal{A}} \left( \pi^{t^{\prime}}(a \vert s^{\prime}) - \sum_{n=1}^{N} q_{n} \pi^{t^{\prime}}_{n}(a \vert s^{\prime}) \right)^{2}}. \nonumber
\end{alignat}
Let $\mathcal{A}_{1} = \left\{ a^{\prime} \in \mathcal{A}: t^{\prime}, a^{\prime}, s^{\prime} \in C_{1} \right\}$ and $\mathcal{A}_{2} = \left\{ a^{\prime} \in \mathcal{A}: t^{\prime}, a^{\prime}, s^{\prime} \in C_{2} \right\}$, we have
\begin{alignat}{1}
\bar{\varepsilon}_{\theta} &\le \left( \sum_{a \in \mathcal{A}_{1}} \sum_{n=1}^{N} q_{n} \pi^{t^{\prime}}(a \vert s^{\prime}) \left( f_{\theta^{t^{\prime}}}(s^{\prime},a) - f_{\theta^{t^{\prime}}_{n}}(s^{\prime},a) + \log \frac{\sum_{a^{\prime} \in \mathcal{A}} \exp \left( f_{\theta^{t}_{n}}(s^{\prime},a^{\prime}) \right)}{\sum_{a^{\prime} \in \mathcal{A}} \exp \left( f_{\theta^{t}}(s^{\prime},a^{\prime}) \right)} \right) \right. \nonumber \\
& \quad \left. + \sum_{a \in \mathcal{A}_{2}} \sum_{n=1}^{N} q_{n} \pi_{n}^{t^{\prime}}(a \vert s^{\prime}) \left( f_{\theta^{t^{\prime}}_{n}}(s^{\prime},a) - f_{\theta^{t^{\prime}}}(s^{\prime},a) + \log \frac{\sum_{a^{\prime} \in \mathcal{A}} \exp \left( f_{\theta^{t}}(s^{\prime},a^{\prime}) \right)}{\sum_{a^{\prime} \in \mathcal{A}} \exp \left( f_{\theta^{t}_{n}}(s^{\prime},a^{\prime}) \right)} \right) \right)^{1/2} \nonumber \\
&\le \left( \sum_{a \in \mathcal{A}} \sum_{n=1}^{N} q_{n} \max\left\{ \pi_{n}^{t^{\prime}}(a \vert s^{\prime}), \pi^{t^{\prime}}(a \vert s^{\prime}) \right\} \left( \left\vert f_{\theta^{t^{\prime}}}(s^{\prime},a) - f_{\theta^{t^{\prime}}_{n}}(s^{\prime},a) \right\vert \right. \right. \nonumber \\
& \quad \left. \left. + \left\vert \log \frac{\sum_{a^{\prime} \in \mathcal{A}} \exp \left( f_{\theta^{t}}(s^{\prime},a^{\prime}) \right)}{\sum_{a^{\prime} \in \mathcal{A}} \exp \left( f_{\theta^{t}_{n}}(s^{\prime},a^{\prime}) \right)} \right\vert \right) \right)^{1/2}. \nonumber
\end{alignat}
By taking the maximum over $a \in \mathcal{A}$, we have
\begin{alignat}{1}
\bar{\varepsilon}_{\theta} &\le \left( \max_{a \in \mathcal{A}, n \in [N]} \left( \left\vert f_{\theta^{t^{\prime}}}(s^{\prime},a) - f_{\theta^{t^{\prime}}_{n}}(s^{\prime},a) \right\vert + \left\vert \log \frac{\sum_{a^{\prime} \in \mathcal{A}} \exp \left( f_{\theta^{t}}(s^{\prime},a^{\prime}) \right)}{\sum_{a^{\prime} \in \mathcal{A}} \exp \left( f_{\theta^{t}_{n}}(s^{\prime},a^{\prime}) \right)} \right\vert \right) \cdot \right. \nonumber \\
& \quad \left. \sum_{a \in \mathcal{A}} \sum_{n=1}^{N} q_{n} \max\left\{ \pi_{n}^{t^{\prime}}(a \vert s^{\prime}), \pi^{t^{\prime}}(a \vert s^{\prime}) \right\} \right)^{1/2} \nonumber \\
&\le \left( 2 \max_{a \in \mathcal{A}, n \in [N]} \left( \left\vert f_{\theta^{t^{\prime}}}(s^{\prime},a) - f_{\theta^{t^{\prime}}_{n}}(s^{\prime},a) \right\vert + \left\vert \log \frac{\sum_{a^{\prime} \in \mathcal{A}} \exp \left( f_{\theta^{t}}(s^{\prime},a^{\prime}) \right)}{\sum_{a^{\prime} \in \mathcal{A}} \exp \left( f_{\theta^{t}_{n}}(s^{\prime},a^{\prime}) \right)} \right\vert \right) \right)^{1/2}. \nonumber
\end{alignat}
By Jensen's inequality, we have
\begin{alignat}{1}
\mathbb{E}_{\text{init}} \left[ \bar{\varepsilon}_{\theta} \right] &\le \mathbb{E}_{\text{init}} \left[ 2 \max_{a,n} \left\vert f_{\theta^{t^{\prime}}}(s^{\prime},a) - f_{\theta^{t^{\prime}}_{n}}(s^{\prime},a) \right\vert + 2 \left\vert \log \frac{\sum_{a^{\prime} \in \mathcal{A}} \exp \left( f_{\theta^{t}}(s^{\prime},a^{\prime}) \right)}{\sum_{a^{\prime} \in \mathcal{A}} \exp \left( f_{\theta^{t}_{n}}(s^{\prime},a^{\prime}) \right)} \right\vert \right]^{1/2}. \label{eq:lemma:aggregation_error1}
\end{alignat}
It remains to bound the two absolute terms on the RHS of (\ref{eq:lemma:aggregation_error1}). The first absolute term can be bounded by Lemma \ref{lemma:changing_action_error} as
\begin{alignat}{1}
\mathbb{E}_{\text{init}} \left[ \max_{a \in \mathcal{A}, n \in [N]} \left\vert f_{\theta^{t^{\prime}}}(s^{\prime},a) - f_{\theta^{t^{\prime}}_{n}}(s^{\prime},a) \right\vert \right] &= \mathcal{O} \left( R_{\theta} \right). \label{eq:lemma:aggregation_error3}
\end{alignat}
Next, we bound the second absolute term on the RHS of (\ref{eq:lemma:aggregation_error1}). By the log-sum inequality, the log-sum-exp trick, and Lemma \ref{lemma:changing_action_error}, we can obtain
\begin{alignat}{1}
\log \sum_{a^{\prime} \in \mathcal{A}} \exp \left( f_{\theta^{t}}(s,a^{\prime}) \right) - \log \sum_{a^{\prime} \in \mathcal{A}} \exp \left( f_{\theta^{t}_{n}}(s,a^{\prime}) \right) &\le \max_{a \in \mathcal{A}} f_{\theta^{t}}(s,a) - \frac{1}{\left\vert \mathcal{A} \right\vert} \sum_{a^{\prime}} \left( f_{\theta_{n}^{t}}(s,a^{\prime}) \right), \nonumber
\end{alignat}
and
\begin{alignat}{1}
\log \sum_{a^{\prime} \in \mathcal{A}} \exp \left( f_{\theta^{t}_{n}}(s,a^{\prime}) \right) - \log \sum_{a^{\prime} \in \mathcal{A}} \exp \left( f_{\theta^{t}}(s,a^{\prime}) \right) &\le \max_{a \in \mathcal{A}} f_{\theta_{n}^{t}}(s,a) - \frac{1}{\left\vert \mathcal{A} \right\vert} \sum_{a^{\prime}} \left( f_{\theta^{t}}(s,a^{\prime}) \right), \nonumber
\end{alignat}
which indicates that
\begin{alignat}{1}
\mathbb{E}_{\text{init}} \left[ \max_{a \in \mathcal{A}, n \in [N]} \left\vert \log \frac{\sum_{a^{\prime} \in \mathcal{A}} \exp \left( f_{\theta^{t}}(s,a^{\prime}) \right)}{\sum_{a^{\prime} \in \mathcal{A}} \exp \left( f_{\theta^{t}_{n}}(s,a^{\prime}) \right)} \right\vert \right] &= \mathcal{O} \left( R_{\theta} \right). \label{eq:lemma:aggregation_error4}
\end{alignat}
By substituting (\ref{eq:lemma:aggregation_error3}) and (\ref{eq:lemma:aggregation_error4}) into (\ref{eq:lemma:aggregation_error1}), we can obtain
\begin{alignat}{1}
\mathbb{E}_{\text{init}} \left[ \bar{\varepsilon}_{\theta} \right] &= \mathcal{O}\left( R_{\theta}^{1/2} \right), \nonumber
\end{alignat}
which completes the proof of (\ref{eq:lemma:aggregation_error_second}). The proof of (\ref{eq:lemma:aggregation_error_third}) is similar.
\end{proof}

\section{Additional Experiment Setting \label{sec:additionexperimentdetails}}

\textbf{Machines:} We simulate the federated learning experiments (1 server and N devices) on a commodity machine with 16 Intel(R) Xeon(R) Gold 6348 CPU @ 2.60GHZz. It took about 6 mins to finish one round of training, i.e., 50 hours to obtain 500 data points for Figure \ref{fig:mcc}.

The hyperparameters for the algorithms on MountainCars, Hoppers, HongKongOSMs, and the general FRL setting are given in Table \ref{table:1}, Table \ref{table:2}, Table \ref{table:3}, and Table \ref{table:4}, respectively.  

\begin{table}[ht!]
\centering
\begin{tabular}{c | c c c c} 
 \hline
 Hyperparameter & FedPOHCS & FedAvg & Power-of-Choice & GradientNorm \\ [0.5ex] 
 \hline
 Learning Rate & 0.001 & 0.005 & 0.001 & 0.001 \\ 
 Learning Rate Decay & 0.98 & 0.98 & 0.98 & 0.98 \\
 Batch Size & 128 & 128 & 128 & 128 \\
 Timestep per Iteration & 2048 & 2048 & 2048 & 2048 \\
 Number of Epochs (E) & 1 & 1 & 1 & 1 \\
 Discount Factor ($\gamma$) & 0.99 & 0.99 & 0.99 & 0.99 \\
 Discount Factor for GAE & 0.95 & 0.95 & 0.95 & 0.95 \\
 KL Target & 0.003 & 0.003 & 0.003 & 0.003 \\ [1ex]
 \hline
\end{tabular}
\caption{Hyperparameters for each algorithm on MountainCars.}
\label{table:1}
\end{table}

\begin{table}[ht!]
\centering
\begin{tabular}{c | c c c c} 
 \hline
 Hyperparameter & FedPOHCS & FedAvg & Power-of-Choice & GradientNorm \\ [0.5ex] 
 \hline
 Learning Rate & 0.03 & 0.03 & 0.03 & 0.03 \\ 
 Learning Rate Decay & 0.9 & 0.9 & 0.9 & 0.9 \\
 Batch Size & 128 & 128 & 128 & 128 \\
 Timestep per Iteration & 2048 & 2048 & 2048 & 2048 \\
 Number of Epochs (E) & 1 & 1 & 1 & 1 \\
 Discount Factor ($\gamma$) & 0.99 & 0.99 & 0.99 & 0.99 \\
 Discount Factor for GAE & 0.95 & 0.95 & 0.95 & 0.95 \\
 KL Target & 0.003 & 0.003 & 0.003 & 0.003 \\ [1ex]
 \hline
\end{tabular}
\caption{Hyperparameters for each algorithm on Hoppers.}
\label{table:2}
\end{table}

\begin{table}[ht!]
\centering
\begin{tabular}{c | c c c c} 
 \hline
 Hyperparameter & FedPOHCS & FedAvg & Power-of-Choice & GradientNorm \\ [0.5ex] 
 \hline
 Learning Rate & 0.0001 & 0.0001 & 0.0001 & 0.0001 \\ 
 Learning Rate Decay & 0.98 & 0.98 & 0.98 & 0.98 \\
 Batch Size & 128 & 128 & 128 & 128 \\
 Timestep per Iteration & 2048 & 2048 & 2048 & 2048 \\
 Number of Epochs (E) & 10 & 10 & 10 & 10 \\
 Discount Factor ($\gamma$) & 0.99 & 0.99 & 0.99 & 0.99 \\
 Discount Factor for GAE & 0.95 & 0.95 & 0.95 & 0.95 \\
 KL Target & 0.0001 & 0.0001 & 0.0001 & 0.0001 \\ [1ex]
 \hline
\end{tabular}
\caption{Hyperparameters for each algorithm on HongKongOSMs.}
\label{table:3}
\end{table}

\begin{table}[ht!]
\centering
\begin{tabular}{c | c c c c} 
 \hline
 Environment & \#Client (N) & \#Candidate (d) & \#Participant (K) & \#Local Iteration (I) \\ [0.5ex] 
 \hline
 MountainCars & 60 & 18 & 6 & 5 \\
 Hoppers & 60 & 18 & 6 & 20 \\
 HongKongOSMs & 10 & 9 & 2 & 10 \\ [1ex]
 \hline
\end{tabular}
\caption{General FRL Setting. Refer to Section \ref{subsec:FRL} and Algorithm \ref{alg:FedPOHCS} for their definitions.}
\label{table:4}
\end{table}

\newpage

\vskip 0.2in
\bibliography{references}

\end{document}